%% file: main.tex
\documentclass{llncs}

\usepackage[colorlinks=false,urlcolor=blue]{hyperref}       
\usepackage{url}            
\usepackage{booktabs}       

\usepackage[dvipsnames]{xcolor}   
\usepackage{amsfonts}       
\usepackage{amsmath}
\usepackage{amssymb}
\usepackage{pifont}
\usepackage{graphicx}
\usepackage{wrapfig}
\usepackage{lipsum}
\usepackage{threeparttable}
\usepackage{nicefrac}       
\usepackage{microtype}      
\usepackage[ruled,noend,linesnumbered]{algorithm2e}

\SetCommentSty{mycommfont}
\usepackage{verbatim}
\usepackage{multirow}
\usepackage{diagbox}
\usepackage{siunitx}
\usepackage{bbding}
\usepackage{float}
\usepackage{etoolbox}
\usepackage{subcaption}
\usepackage{makecell}

\usepackage{booktabs}
\usepackage{bigstrut}
\usepackage{tabu}
\usepackage{tikz}

\usepackage{scalefnt}

\usepackage[firstpage]{draftwatermark}  

\SetWatermarkAngle{0}

\makeatletter
\patchcmd{\@algocf@start}
{-1.5em}
{0pt}
{}{}
\makeatother
\def\HiLi{\leavevmode\rlap{\hbox to \hsize{\color{gray!30}\leaders\hrule height .6\baselineskip depth .5ex\hfill}}}
\usepackage{array}
\newcolumntype{L}[1]{>{\raggedright\let\newline\\\arraybackslash\hspace{0pt}}m{#1}}
\newcolumntype{C}[1]{>{\centering\let\newline\\\arraybackslash\hspace{0pt}}m{#1}}
\newcolumntype{R}[1]{>{\raggedleft\let\newline\\\arraybackslash\hspace{0pt}}m{#1}}

\SetKwInput{KwInput}{Input}                
\SetKwInput{KwOutput}{Output}              

\usepackage{algcompatible}

\definecolor{light-gray}{gray}{0.8}

\newcommand{\BBReach}{\textsf{BBReach}}
\newcommand{\nobi}[1]{{\color{red} #1}}

\newif\ifdraft
\drafttrue
\newtoggle{conf-ver}
\settoggle{conf-ver}{false} 
\begin{document}
\SetWatermarkText{\raisebox{12.5cm}{%
  \hspace{0.02cm}
\href{}{\includegraphics[scale=0.09]{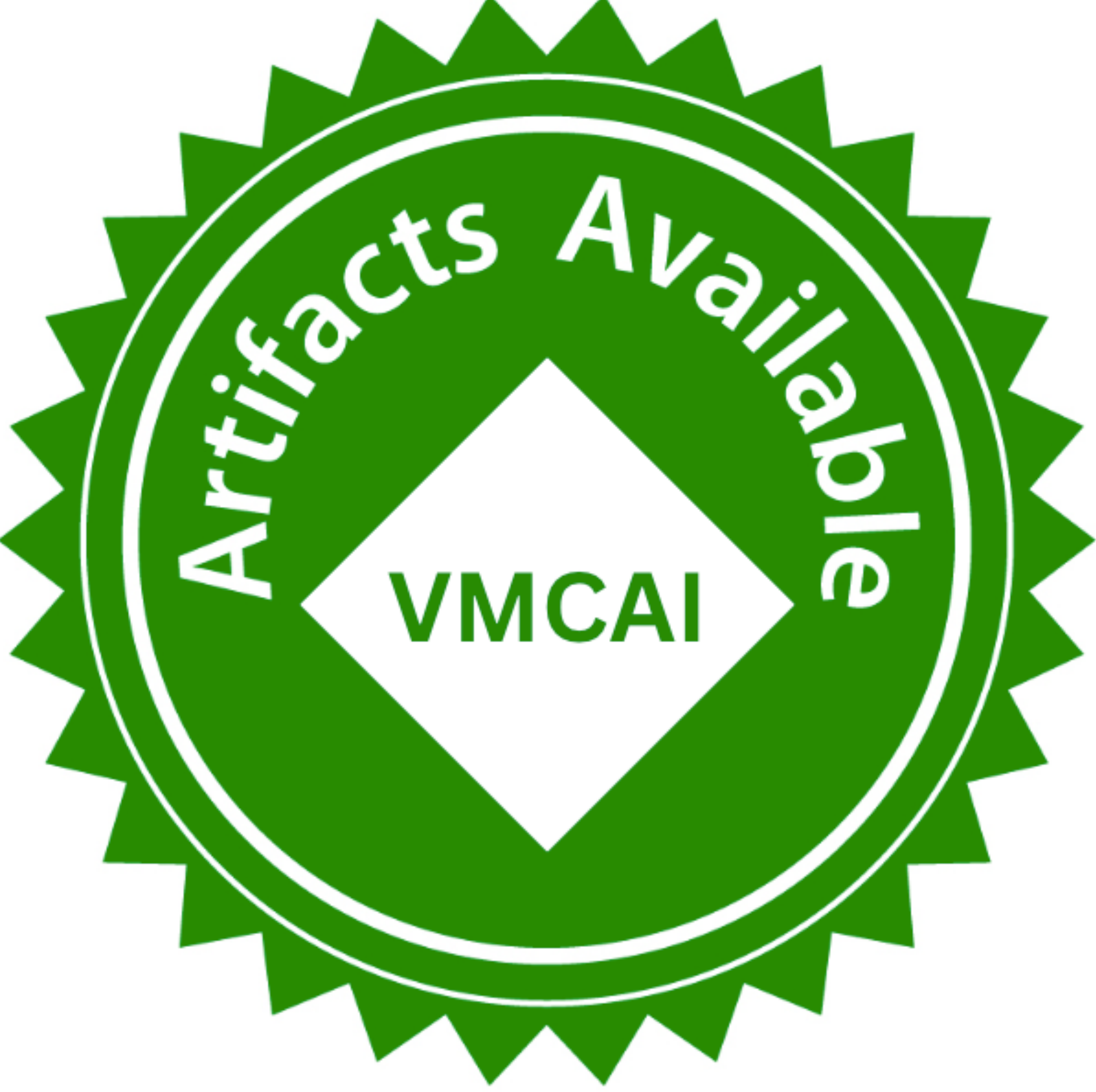}}
  \hspace{8cm}%
  \includegraphics[scale=0.09]{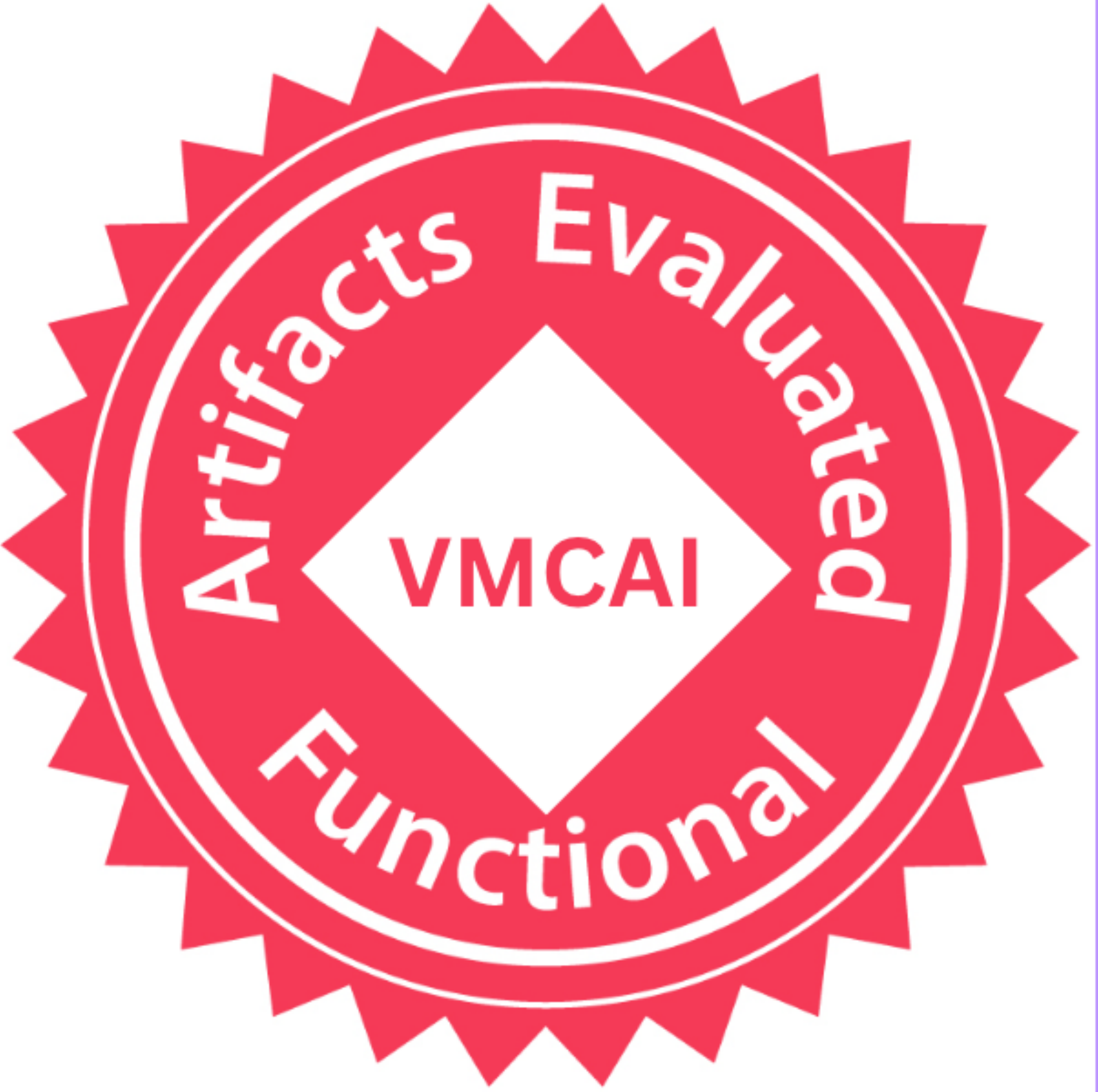}%
}}
\title{Taming Reachability Analysis of DNN-Controlled Systems via  Abstraction-Based Training}

\author{
 Jiaxu Tian\inst{1} \and Dapeng Zhi\inst{1} 
	\and Si Liu\inst{2} \and Peixin Wang\inst{3} \and Guy Katz\inst{4} \and Min Zhang\inst{1}
}

\institute{
Shanghai Key Laboratory of Trustworthy Computing,\\ East China Normal University, Shanghai, China
	\and ETH Zurich, Zurich, Switzerland \\   
	\and University of Oxford, Oxford, UK\\
	\and The Hebrew University of Jerusalem, Jerusalem, Israel\\
}

\maketitle

\begin{abstract}
	The intrinsic complexity of deep neural networks (DNNs) makes it challenging to verify not only the networks themselves but also the hosting DNN-controlled systems. Reachability analysis of these systems faces the same challenge. Existing approaches rely on over-approximating  DNNs using simpler polynomial models. However, they    
	suffer from low efficiency and large overestimation, and are restricted to specific types of DNNs. This paper presents a novel abstraction-based approach to bypass the crux of over-approximating DNNs in reachability analysis. Specifically, we 
	extend conventional DNNs by 
	inserting an additional abstraction layer, which abstracts a real number to an interval for training. The inserted abstraction layer ensures that the values represented by an interval are indistinguishable to the network for both training and decision-making. Leveraging this, we devise the first black-box reachability analysis approach for DNN-controlled systems, where trained DNNs are only queried as black-box oracles for the actions on abstract states. 
 Our approach is sound, tight, efficient, and agnostic to any DNN type and size. The experimental results on a wide range of benchmarks show that the DNNs trained by using our approach exhibit comparable performance, while the reachability analysis of the corresponding systems becomes more amenable with significant tightness and efficiency improvement over the state-of-the-art white-box approaches. 

\end{abstract}

\input{intro}
\input{preliminaries}

\input{motivation}

\input{abs-based-training}

\input{method}

\input{experiment}

\input{concl}

\bibliographystyle{splncs04}
\bibliography{ref}

\newpage 
\appendix
 \input{appendix}

\end{document}

%% file: intro.tex
\section{Introduction}
Deep neural networks (DNNs) have demonstrated their remarkable capability of driving systems to perform specific tasks intelligently in open environments. They determine optimal actions during interactions between the hosting systems and their surroundings.
Formally verifying DNNs can 
 provide safety guarantees~\cite{szegedy2014intriguing,gomes2016will,schmidt2021can}, which is,  however, 
difficult in practice due to their black-box nature and
lack of interpretability \cite{zhang2022qvip,baluta2021scalable}.
Furthermore, their  hosting systems aggravate the difficulty 
since determining system actions requires computations over nonlinear system dynamics \cite{dreossi2019verifai,sun2019formal}.



Reachability analysis, one of the powerful formal  methods, has been widely applied to the verification of continuous and hybrid   systems~\cite{alur1995algorithmic,bertsekas1971minimax,christakis2021automated}.
Its successful applications include 
invariant checking~\cite{hildebrandt2020blending,fang2021fast}, robust control~\cite{limon2005robust,schurmann2018reachset}, fault detection~\cite{scott2016constrained,su2017model},  set-based predication~\cite{althoff2016set,pereira2017overapproximative}, etc.
The essence of reachability analysis is to compute all reachable system states from given initial state(s), which can be used in various verification tasks such as model checking~\cite{baier2008principles}.
As an emerging approach to verifying DNN-controlled systems,
reachability analysis has already been shown to be promisingly effective~\cite{dutta2019reachability,fan2020reachnn,ivanov2021verisig}. 

\vspace{1ex}
\noindent \textbf{The Problem.} Compared to continuous hybrid systems, it is significantly more challenging to compute reachable states for DNN-controlled systems due to the embedded complex and inexplicable DNNs.
In addition to over-approximating nonlinear system dynamics~\cite{chen2013flow,frehse2011spaceex,lygeros1999controllers}, one also has to over-approximate the embedded DNNs for computing overestimated action sets~\cite{huang2022polar,ivanov2021verisig} of such systems. Specifically, given a set $S_i$ of continuous system states at time step $i$,\footnote{Continuous time is uniformly discretized into time steps.} one first
overestimates a set $\tilde{A}_i$ of actions that will be applied to 
$S_i$ by over-approximating the neural network on $S_i$, and then 
 overestimates a set $\tilde{S}_{i+1}$ of successors by applying $\tilde{A}_i$ to $S_i$ using over-approximated system dynamics. 
We consider such dual 
 over-approximations as \textit{white-box} approaches since all the information of DNNs, such as architectures, activation functions, and weights, shall be known before defining appropriate over-approximated models \cite{dutta2019reachability,fan2020reachnn,ivanov2021verisig}. Consequently, these approaches are restricted to certain types of DNNs.   
For instance, 
 Verisig 2.0~\cite{ivanov2021verisig} does not support neural networks with the ReLU activation functions; Sherlock~\cite{dutta2019reachability} is only applicable to ReLU-based networks; ReachNN*~\cite{fan2020reachnn} is not scalable against the network size and introduces more overestimation; Polar~\cite{huang2022polar} also suffers from the efficiency problem when dealing with networks with differentiable activation functions (e.g., Tanh). 
 Moreover, dual over-approximations introduce large overestimation accumulatively, which results in a considerable number of unreachable states in the overestimated sets.


\vspace{1ex}
\noindent \textbf{Our Approach.}
We present a novel abstraction-based approach for bypassing the over-approximation of  DNNs in computing the reachable states of DNN-controlled systems. Our approach introduces an \textit{ abstraction layer} into the neural network before training, which abstracts concrete system states into abstract ones. This abstraction ensures that concrete states that are abstracted into the same state share the same action determined by the trained DNN. 
Leveraging this property, we can therefore compute the actions of a set of concrete states by mapping them to the corresponding abstract states and by feeding the abstract states into the trained DNNs to query for the output action. As DNNs are used as \textit{black-box} oracles during the entire process, it suffices to know how system states are abstracted and to query the trained DNNs with the abstract states for the actions. Hence, the over-approximation of DNNs for computing actions is decently bypassed. Consequently, the overestimation due to the embedded network is avoided and no assumption is made on a network including its size, weight, architecture, and activation function.

The abstraction-based training also allows us to avoid state explosion during the computation of reachable states. This is because adjacent abstract states, e.g., the two intervals $[0,1]$ and $[1,2]$, can be efficiently aggregated, e.g., to $[0,2]$, which substantially restrains the exponential growth in the number of computed reachable states. Additionally, we propose a parallel optimization via initial-set partitioning, which
 further accelerates the process of computing reachable states.

We have implemented our proposed approach into a tool called \BBReach~and extensively evaluated over a wide range of benchmarks. The experimental results show that DNNs trained by using our abstraction-based approach achieve competitive performance in terms of system cumulative reward. Our approach provides a black-box alternative to the reachability analysis of DNN-controlled systems, which bypasses the crux of DNN over-approximation and significantly improves  the 
state-of-the-art white-box counterparts with respect to 
 the tightness and efficiency in reachable state computation.

\vspace{1ex}
\noindent\textbf{Contributions.} Overall, we provide: 
\begin{enumerate}
	
\item a novel abstraction-based training approach of DNNs, which 
mitigates the limitation of DNN over-approximation in the reachability analysis of DNN-controlled systems, without sacrificing the performance of trained DNNs  (Section~\ref{sec:abs-train});

\item the first, sound black-box approach for the reachability analysis of trained DNN-controlled systems, which not only enhances computational tightness and efficiency, but also are compatible with various DNNs  (Section~\ref{sec:reach_analysis_approach}); and 

\item a prototype \BBReach\ and an extensive assessment, which shows that \BBReach\ improves existing white-box tools with respect to both the precision of results and the computational efficiency (Section~\ref{sec:exp}).
\end{enumerate}

%% file: preliminaries.tex
\section{Preliminaries}



\subsection{DNN-Controlled Systems}
\label{subsec:drl}
\vspace{-1mm}
A DNN-controlled system is typically a cyber-physical system where a DNN is planted and trained as a decision-making controller. 
It can be modeled as a 6-tuple ${\cal D}=\langle S,S^0,A,\pi,f,\delta\rangle$, where $S$ is the set of $n$-dimensional system states on $n$ continuous variables,  $S^0\subseteq S$ is the set of initial states, $A$ is the set of system actions,  $\pi:S\rightarrow A$ is a policy function realized by the DNN in the system, $f:S\times A\rightarrow \dot{S}$ is a non-linear continuous environment dynamics represented by an
ordinary differential equation (ODE)~\cite{2019nonlinear}
that maps the current state and control input (i.e., action) into the derivative of states with respect to time $t_c$, and $\delta$ is the time step size.

In a DNN-controlled system, an agent reacts to the environment over time. The time is usually discretized by a time scale $\delta$ called the time step size, assuming that actions during each time scale $\delta$ are constants~\cite{park2021time}.
At each time step $i\in \mathbb{N}$, the agent first observes a state $s_i$ from the environment and feeds the state into the network to compute a constant action $a_i$. 
The agent then transits to the successor state $s_{i+1}$ by performing $a_i$ on $s_i$ according to some environment dynamics $f$. 
\begin{wrapfigure}{R}{0.7\textwidth}
	\vspace{-8mm}
	\centering
	\includegraphics[width=0.7\textwidth]{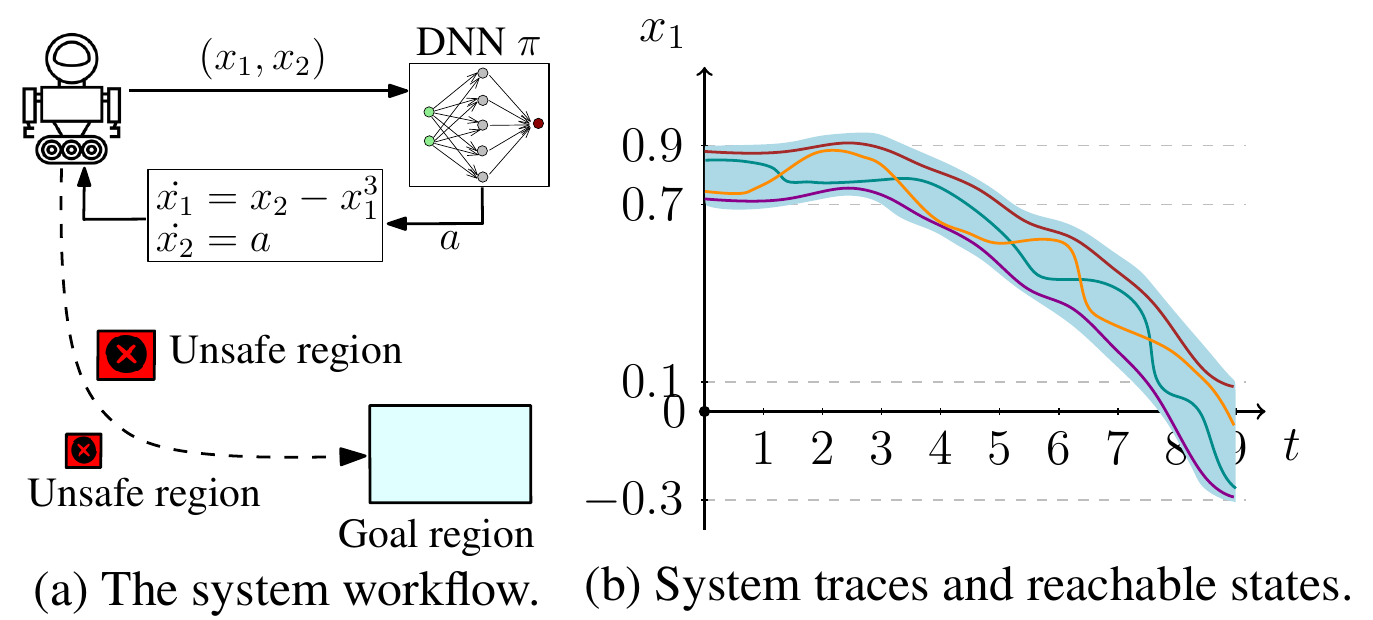}
	\caption{The workflow of the DNN-controlled system in Example~\ref{exa:DRLexample}, the execution traces (colored lines) and an over-estimated set of reachable states (blue region) with respect to the dimension of $x_1$.}
	\label{fig:traj}
	\vspace{-7mm}
\end{wrapfigure}
During the training phase of DNN-controlled systems, the agent also receives a reward $r_i$ which is determined by a reward function $r_i = R(s_i,a,s_{i+1})$ from the environment after each state transition. Once the task is finished, e.g., the agent reaches the goal region at time step $T$, we obtain the sequence of traversed system states from an initial state, called a \textit{trace}, and the cumulative reward $\sum_{i=0}^{T}r_i$ which quantitatively measures the system performance.



\vspace{-1mm}
\begin{example}[A DNN-Controlled System]
\label{exa:DRLexample}
Figure \ref{fig:traj}(a) shows a  DNN-controlled system where a two-dimensional agent moves from the region  $x_1 \in [0.7, 0.9]$, $x_2 \in [0.7, 0.9]$ to the goal region $x'_1 \in [-0.3, 0.1]$, $ x'_2 \in [-0.35, 0.05]$, trying to avoid the red unsafe regions. The environment dynamics  $f$ is defined by the following ODEs:
\begin{align}
\dot{x_1} =  x_2 - x_1^{3}\quad \quad 
\dot{x_2} =  a
\end{align}
\noindent The action $a=\pi(x_1,x_2)$ is computed by applying the DNN $\pi$ to the values of $x_1$ and $x_2$. Based on $a$ and $f$, the successor state $s'$ can be computed for the agent to move. Figure \ref{fig:traj}(b) shows 
the traces (colored lines) of the system from some selected concrete initial states and an over-approximated set of reachable states (blue area) on the $x_1$ dimension from all the initial states. 

\end{example}
 

The DNN planted in a system must be trained first so that it can determine optimal actions to complete a task. 
After making a decision, a loss is computed by a predefined loss function based on the reward that the agent receives for the decision. The parameters in the neural network are updated based on the loss by backpropagation \cite{lecun2015deep}. 
The objective of the training phase is to maximize the cumulative reward. Once the training is completed, the network implements a state-action policy function that maps each system state to its optimal action. It drives the system to run and to interact with the environment.

\vspace{-1mm}
\subsection{Reachability Problem of DNN-Controlled Systems}

Given a DNN-controlled system $\cal D$,  whether or not a state is reachable is known as the \textit{reachability problem}. 
The verification of safety properties can be reduced to the reachability problem. For instance, one can verify whether a system never moves to unsafe states or not, such as those in Example~\ref{exa:DRLexample}. Unfortunately, the  problem is \textit{undecidable} even for conventional cyber-physical systems that are controlled by explicit programmable rules, let alone uninterpretable neural networks. This is because such 
systems are more expressive than two-counter state machines whose reachability problem is proved to be undecidable~\cite{minsky1967computation}.

When the set of initial states is a singleton, it is straightforward to compute 
the reachable state at any given time $t_c$. 
Let $\delta$ be a time scale during which system actions can be considered constant. Given an initial state $s_0$, the state at time $t_c=k\delta+t'_c$ 
for some integer $k\geq 0$ and $0\leq t'_c\leq \delta$ is defined as follows: 
\vspace{-2mm}
\begin{equation}\label{for:single}
\nonumber
\setlength{\abovedisplayskip}{4pt}
\setlength{\belowdisplayskip}{4pt}
	\varphi_f(s_0,\pi,t_c)  = s_k + \int_{0}^{t'_c} f(s,\pi(s_{k})) dx, 
\end{equation}
where $s_{i+1}=s_{i}+\int_{0}^{\delta} f(s,\pi(s_{i})) dx$ for all $i\in \{0,\ldots,k-1\}$. Intuitively, we can compute the state $s_{i+1}$ at the $(i+1)$-th time step based on the state $s_{i}$ at its preceding time step $i$ and the corresponding action $\pi(s_{i})$. The state at $t_c$ can be computed based on $s_k$, plus the offset caused by performing action $\pi(s_k)$ on state $s_k$ with $t'_c$ time scale.


\begin{definition}[Reachable States of DNN-controlled Systems]\label{def:reach}
Given a DNN-controlled system ${\cal D}=\langle S,S^0,A,\pi,f,\delta\rangle$, the sets of all the reachable states of the system at and during time $t_c$ are denoted as  
$\mathit{Reach}^{t_c}_f(S_0)$ and 
$\mathit{Reach}^{[0,t_c]}_f(S_0)$, respectively. We have $\mathit{Reach}^{t_c}_f(S_0)=\{\varphi_f(s,\pi,t_c)|s\in S_0\}$ and   $\mathit{Reach}^{[0,t_c]}_f(S_0)=\{\varphi_f(s,\pi,t)|s\in S_0,t\in [0,t_c]\}$. 
\end{definition}

	\begin{wrapfigure}[8]{r}{0.65\textwidth}
	\centering
		\vspace{-9mm}
	\includegraphics[width=.64\textwidth]{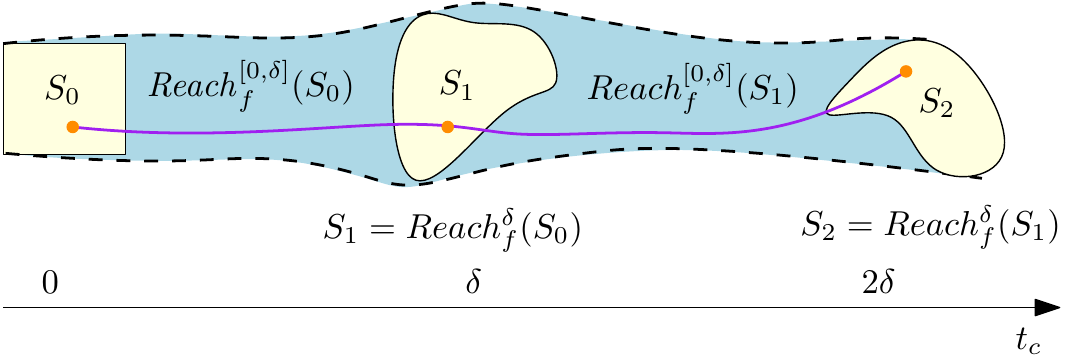}
\vspace{-3mm}
	\caption{Reachable states of DNN-controlled systems.}
	\label{fig:reachable_states_dnn_control}
\end{wrapfigure} 
Figure \ref{fig:reachable_states_dnn_control} depicts an example of the reachable states from $S_0$. 
For each time step $i$, we compute the set 
$\mathit{Reach}^{[0,\delta]}_f(S_i)$ of all the reachable states during the time period from $i$ to $i+1$. The actions used for computing $\mathit{Reach}^{[0,\delta]}_f(S_i)$ are the constants determined by the DNN $\pi$ on the states in $S_i$. 
In particular, we 
compute the set $S_{i+1}=\mathit{Reach}^{\delta}_f(S_i)$ of the reachable states at step $i+1$. Note that  $S_{i+1}$ is a subset of $\mathit{Reach}^{[0,\delta]}_f(S_i)$. We need to compute $S_{i+1}$ independently from $\mathit{Reach}^{[0,\delta]}_f(S_i)$ because 
it is the basis of computing the reachable states in next step. 

The procedure depicted in Figure \ref{fig:reachable_states_dnn_control} indicates that the problem of computing $\mathit{Reach}^{[0,t_c]}_f(S_0)$ can be reduced to the problem of computing one-time-step reachable states, i.e.,   $\mathit{Reach}^{[0,\delta]}_f(S_0)$ and 
 $\mathit{Reach}^{\delta}_f(S_0)$. 
 However, the reduced problem is still intractable. This is because $S_0$ is usually an infinite set, meaning that it is impractical to enumerate each  state in $S_0$, feed it into the DNN to compute the corresponding action, and then compute the state by Formula $\ref{for:single}$ for the set $\mathit{Reach}^{\delta}_f(S_0)$.  
Computing the states in 
$\mathit{Reach}^{[0,\delta]}_f(S_0)$ 
is even more challenging due to the continuous time in $[0,\delta]$.

%% file: motivation.tex
\section{Motivation}
The combination of nonlinear dynamics and neural network controllers makes the calculation of $Reach_f^{[0,\delta]}(S_0)$  intractable. This is because the function $\varphi_f$ (Formula~\ref{for:single}) can not be expressed in a known closed form for most nonlinear dynamics $f$~\cite{chen2012taylor}.
Additionally, a DNN $\pi$ neither can be replaced by a known form equivalent function. A pragmatic solution is to compute tight over-approximation for $\varphi_f$ and $\pi$.
Most of the state-of-the-art approaches, such as Verisig~\cite{ivanov2020verifying}, Polar~\cite{huang2022polar}, and ReachNN~\cite{huang2019reachnn}, adopt this strategy.


\begin{wrapfigure}[11]{r}{0.55\textwidth}
\vspace{-9mm}
		\centering
	\includegraphics[width=0.55\textwidth]{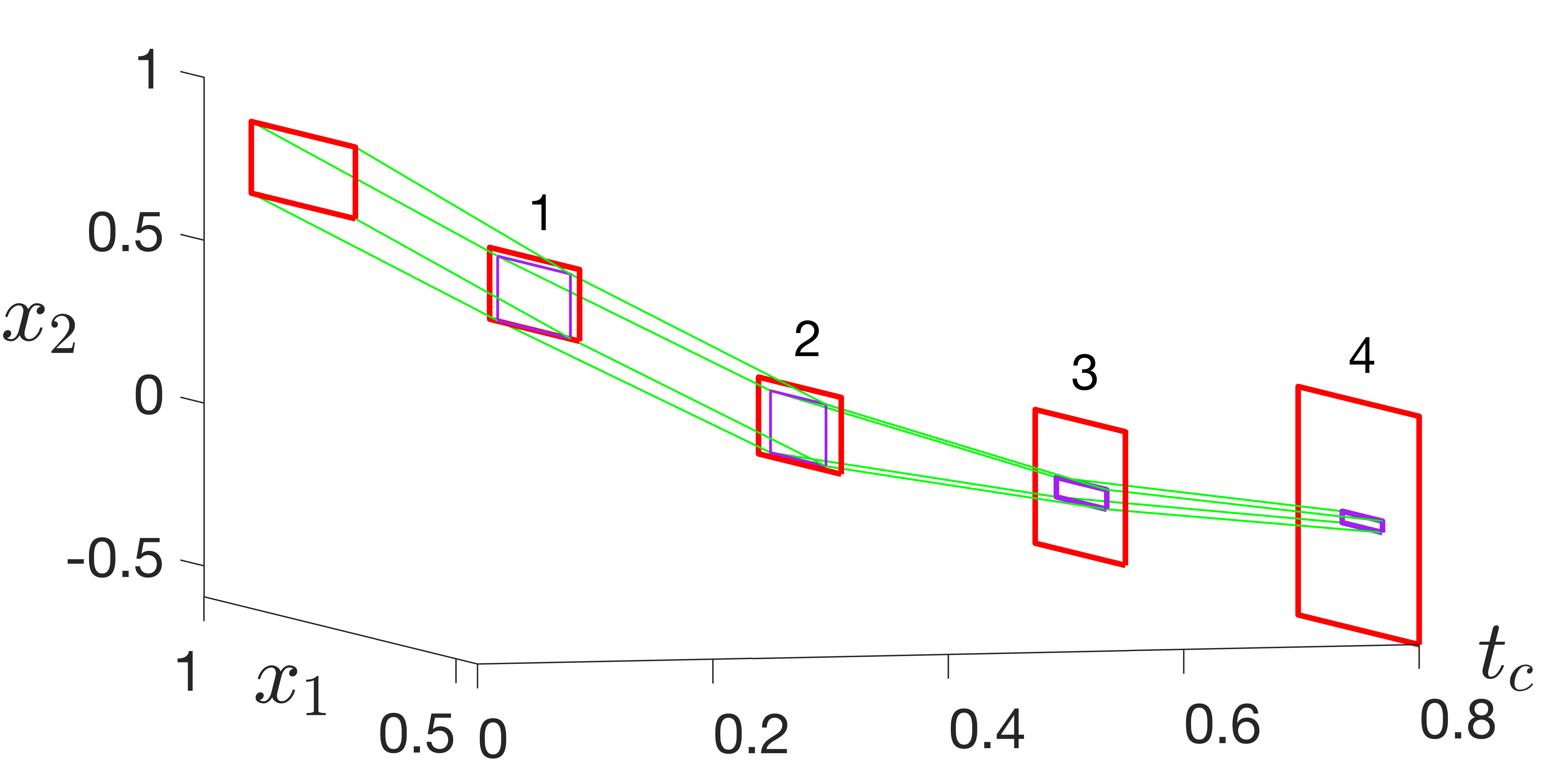}
		\caption{An example of overestimation blowup of computed reachable states.}
		\label{fig:range_explosion}
 \vspace{-7mm}
\end{wrapfigure}

Without loss of generality, we show the process of over-approximating $\mathit{Reach}^{\delta}_f(S_0)$ in Example~\ref{exa:DRLexample} using Polar. 
Given a set $S_0$ of states, Polar first over-approximates the neural network using a Taylor model $(p,I_r)$~\cite{makino2003taylor} on domain $S_0$ such that $\forall s \in S_0, \pi(s) \in p(s) + [-\epsilon,\epsilon]$, where $p$ is a polynomial over the set of state variables $x_1,\dots,x_n$ such as $p(x_1,x_2) = 0.5+ 0.1x_1+ 0.6x_1x_2+0.3x_1^2x_2$ and $I_r = [-\epsilon,\epsilon]$ is called the remainder interval. The range of $\pi(s)$ can be  overestimated based on the Taylor model. Next, Polar over-approximates the solution of environment dynamics $\varphi_f$  using another Taylor model  over domains $s_0 \in S_0, \pi(s_0) \in p(s_0) + [-\epsilon,\epsilon], t_c \in [0,\delta]$ and obtains $x_1'\in p_1(x_1,x_2,t_c)+[-\epsilon_1,\epsilon_1], \ x_2' \in p_2(x_1,x_2,t_c)+[-\epsilon_2,\epsilon_2]$. Finally, Polar produces an overestimated set of  $S_1$ at time $\delta$ based on $x_1'$ and $x_2'$. A smaller range of $I_r$ means less over-approximation error.



Suppose the initial region in Example~\ref{exa:DRLexample} is $x_1 \in [0.7,0.9], x_2\in [0.7,0.9]$. The overestimated reachable states can be calculated over 4 time steps according to the aforementioned method,  which are depicted as red boxes (\textcolor{red}{$\Box$}) in Figure~\ref{fig:range_explosion}. 
For comparison, Figure~\ref{fig:range_explosion} also shows the reachable states by simulation with 1000 samples, which are shown as the small violet boxes (\textcolor{violet}{$\Box$}). 
We observe that the overestimation is amplified at the third and the fourth time step. At the third time step, the calculated remainder interval of the Taylor model for network is $[-0.98,0.98]$ while the one at the fourth time step is $[-4.47,4.47]$. Correspondingly, the remainder intervals of the Taylor model for dynamics are $[-0.17,0.17]$ and $[-0.46,0.46]$ at the third and the fourth time step. The overestimation is accumulated and amplified step by step.

The above example shows that overestimation is mainly introduced by the over-approximation of the DNN. We further observe that if we could group the states in $S_0$ into several  subsets such that all the states in the same subset have the same action according to $\pi$, we do not need to over-approximate $\pi$  but, instead, replace $\pi(s)$ with its corresponding action. That is, if we know that all the states in a set $S'_0$ share the same action, e.g., $a$, according to $\pi$, the problem of computing $\mathit{Reach}^{\delta}_f(S'_0)$ can be simplified to solving the following problem:
\begin{equation}
\setlength{\abovedisplayskip}{1pt}
\setlength{\belowdisplayskip}{2pt}
\bigcup_{s'_0\in S'_0} \{s'_0+\int_{0}^{\delta}f(s'_0,a)dx\}.\label{for:post}
\end{equation}

\newpage 
Naturally, we only need to over-approximate $\varphi_f$ to solve the above problem. Therefore, we identify a condition of bypassing over-approximating $\pi$: $S_0$ can be divided into a finite number of subsets such that the states in the same subset have the same action according to $\pi$. This will be elaborated in our following abstraction-based approach.

%% file: abs-based-training.tex
\section{Abstraction-Based Training}
\label{sec:abs-train}
Given a DNN $\pi$ and a set $S_0$ of system states, it is almost intractable to group the states in $S_0$ that have the same action according to $\pi$. 
It becomes even worse when actions are continuous, where each state in $S_0$ may have a different action from others. Instead of calculating these states \emph{ex post facto}, we propose an \emph{ex ante} approach by abstraction-based training,  in which  system states are first grouped by abstraction before training, and a trained DNN provably yields a unique action for the states in the same group. 

\subsection{Approach Overview}

The process of grouping a set of system states and making them indistinguishable to neural networks is called \emph{abstraction}. A group is considered as an abstract state. 
The indistinguishability of the states in the same group guarantees that a DNN computes a unique action for those states. 
This idea is inspired by the abstraction approaches in formal methods, by which system states are abstracted to reduce state space and improve verification scalability without losing the soundness of verification results~\cite{cousot1977abstract}. 
The same idea is also studied in the AI communities. State abstraction has been proved useful for conventional Reinforcement Learning (RL) \cite{singh1995reinforcement,akrour2018regularizing,abel2019theory} and recently applied to Deep RL for training DNN controllers~\cite{jin2022cegar}. 
Studies show that one can train nearly optimal system policies via approximate state abstraction, while the trained policies are more concise and amenable for reasoning and verification than those trained on concrete states \cite{jin2022cegar,abel2019theory}.


To implement state abstraction into deep learning, we extend ordinary DNN architectures by introducing an \emph{abstraction layer}  between the input layer and the first hidden layer. This layer is used to map concrete system states in a group to the same abstract representation, which is propagated throughout the hidden layers 
\begin{wrapfigure}[11]{r}{0.45\textwidth}
	\vspace{-7mm}
	\centering
	\includegraphics[width=0.38\textwidth]{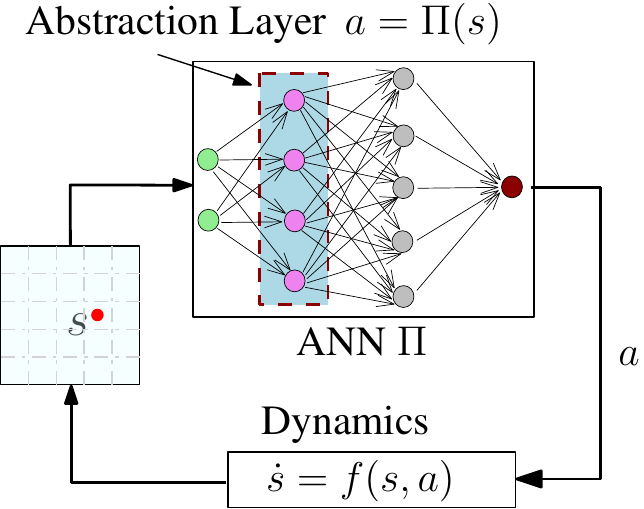}
	\vspace{-2mm}
	\caption{Abstraction-based training.}
	\label{fig:asdrl}
	\vspace{-5mm}
\end{wrapfigure}
for training. 
We call a neural network that contains such an abstraction layer an \textit{abstract neural network} (ANN). Note that an ANN is a special model of DNN. In what follows, we call the systems with ANN controllers \emph{ANN-controlled systems} to differ from those  controlled by conventional DNNs. 


%

The training of ANNs is almost the same as for conventional DNNs. Figure~\ref{fig:asdrl} shows the training workflows with  ANNs.  
We can simply replace DNNs with ANNs in existing training algorithms, such as Deep Q-Network (DQN) \cite{DBLP:journals/nature/MnihKSRVBGRFOPB15} and Deep Deterministic Policy Gradient (DDPG)~\cite{DBLP:journals/corr/LillicrapHPHETS15}, as 
the inserted abstraction layers in ANNs are invisible to these algorithms.

Therefore, an algorithm that supports training DNN-controlled systems can be seamlessly adapted to train ANN-controlled systems. 
When applying these algorithms to ANNs, the only difference is that we need to freeze the parameters on the edges between the input layer and the abstraction layer because they are determined and fixed according to the way in which system states are abstracted. 
Parameter freezing is a common operation in deep learning and is supported by most of the training platforms such as TensorFlow \cite{abadi2016tensorflow} and PyTorch \cite{paszke2019pytorch}. After encoding the abstraction layer and freezing the parameters, a network can be trained just like conventional DNNs by these training algorithms.





\vspace{-1mm}
\subsection{Interval-Based State Abstraction}
\label{subsec:interval_based_abs}
\vspace{-1mm}
We propose a general approach for encoding interval-based abstractions into equivalent abstraction layers. Interval-based state abstraction is a very primitive, yet effective abstraction approach. In the domain of abstract interpretation~\cite{cousot1977abstract}, it is known as \emph{interval abstract domain} and has been well studied for system \cite{afzal2019veriabs} and program verification \cite{heo2019resource}, as well as neural network approximation \cite{wang2022interval}. By interval-based abstraction, the domain of each dimension is evenly divided into several intervals. The Cartesian product of the intervals in all the dimensions constitutes a finite and discrete set, with each element representing an infinite set of concrete states. 

\begin{definition}[Interval-Based State Abstraction]
	\label{def:abs_fun}
	Given an $n$-dimensional continuous state space $S$ and an abstract state space $S_\phi$ obtained by discretizing $S$ based on an abstraction granularity $\gamma$, for every concrete state $s=(x_1, \dots, x_n) \in S$  and abstract state $s_\phi = (l_1,u_1,\dots,l_n,u_n) \in S_\phi$, the interval-based abstraction function $\phi:S\rightarrow S_\phi$ is defined as $\phi(s) = s_\phi$ if and only if for each dimension $1 \le i \le n: l_i \le x_i < u_i$.
\end{definition} 
\looseness=-1
Specifically, the abstract state space $S_\phi$ is obtained by dividing each dimension in the original $n$-dimensional state space $S$ into a set of intervals, which means that each abstract state can be represented as a $2n$-dimensional vector $(l_1,u_1,\dots,l_n,u_n)$.
We also call the $2n$-dimensional vector as interval box.
In what follows, an interval box is used to represent a set of concrete states that fall into it. That is, for a $2n$-dimensional vector $(l_1,u_1,\dots,l_n,u_n)$, we use it to represent the set of $n$-dimensional concrete states $\{(x_1, \dots, x_n) \mid l_i \le x_i < u_i, \forall 1 \le i \le n\}$. In this work, we divide the state space uniformly for better scalability so that we do not need extra data structure to store the mapping between $S$ and $S_\phi$. More specifically, let $L_i$ and $U_i$ be the lower and upper bounds for the $i$-th dimension of $S$. We  define the abstraction granularity as an $n$-dimensional vector $\gamma=(d_1,d_2,\ldots,d_n)$, and  then   evenly divide the $i$-th dimension into $(U_i-L_i)/d_i$ intervals.

\begin{wrapfigure}[7]{r}{0.6\textwidth}
	\centering
	\includegraphics[width = 0.6\textwidth]{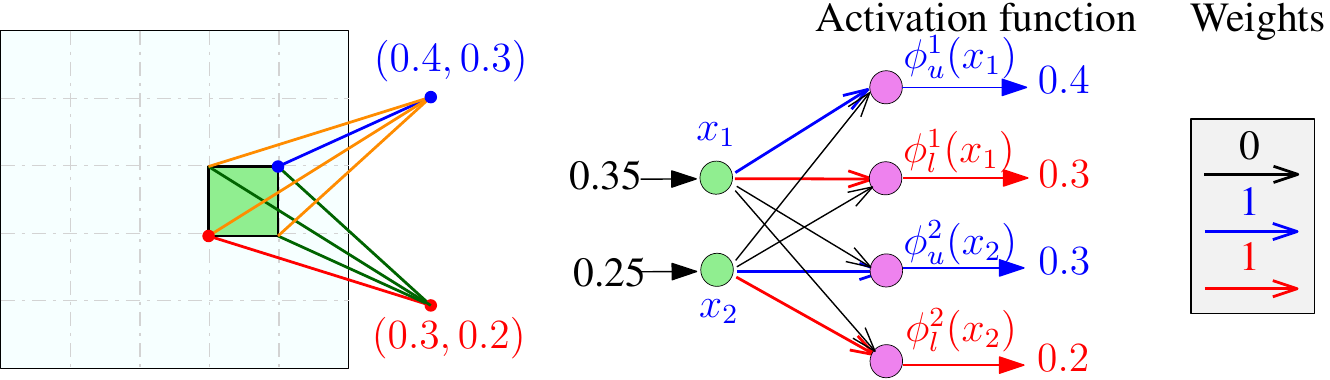}
	\vspace{-6mm}
	\caption{An example of defining abstraction layers.}
	\label{fig:abstraction_layer_encode}
	\vspace{-7mm}
\end{wrapfigure}

An interval-based abstraction can be naturally encoded as an abstraction layer. The layer consists of $2n$ neurons, each of which represents an element in the $2n$-dimensional vector $(l_1,u_1,\dots,l_n,u_n)$. Each neuron has an activation function in either of the following two forms:  
\begin{align}
\nonumber
\phi_l^i(x_i)= L_i + \lfloor \frac{(x_i-L_i)}{d_i}\rfloor d_i, \quad  
\phi_u^i(x_i)= L_i + \lfloor \frac{(x_i-L_i+d_i)}{d_i}\rfloor d_i
\end{align}
for converting the value $x_i$ in a concrete state to its lower and upper bounds, respectively. The sign $\lfloor \cdot \rfloor$ is the floor function.
The weights of the edges connecting the $i$-th neuron in the input layer to the $(2i-1)$-th and $2i$-th neurons in the abstraction layer are assigned a value of 1, whereas the weights of all other edges are set with 0.


\begin{example}
	Suppose that the ranges of both $x_1$ and $x_2$ in Example \ref{exa:DRLexample} are $[0,0.5]$, and they are evenly partitioned into $5$ intervals. 
	The state space $[0,0.5]\times [0,0.5]$ is then uniformly partitioned into $25$ interval boxes, as shown in Figure~\ref{fig:abstraction_layer_encode}. 
	A concrete state such as $(0.35, 0.25)$ is mapped to an \textit{interval box}  represented by the corresponding lower bounds $(0.3,0.2)$ of the first dimension and upper bounds $(0.4,0.3)$ of the second dimension. 
	
	This abstraction can be realized by an abstraction layer, where there are four  neurons and their activation functions are $\phi_u^1(x) = \phi_u^2(x)=\lfloor \frac{x+0.1}{0.1}\rfloor\times 0.1$ and $\phi_l^1(x) = \phi_l^2(x) =\lfloor \frac{x}{0.1}\rfloor\times 0.1$, respectively. 
 \vspace{-2mm}
\end{example}

%% file: method.tex
\section{Abstraction-Based Reachability Analysis}
\label{sec:reach_analysis_approach}

\vspace{-2mm}
\subsection{Approach Overview} \label{sec:overview}
With the abstraction layer, we propose our abstraction-based black-box reachability analysis approach for ANN-controlled systems. 
Given an ANN-controlled system, a set $S_0$ of initial states and a maximal time step $T$, our task is to calculate a sequence of over-approximation sets consisting of interval boxes, denoted by  $X_0,X_1,\dots,X_T$,
which are over-approximations of the actually reachable state sets $S_0,S_1,\dots,S_T$ with $S_{t+1} = Reach_f^{\delta}(S_t), 0 \le t < T$. The overall process is presented in Algorithm~\ref{alg:reachsets_cal}. It is an iterative process of calculating 
 an over-approximated set $X_{t}$
 of states that are reachable from a set $X_{t-1}$ of states after time $\delta$. After we determine the range of $\pi(s)$ over $s\in X_{t-1}$, the reachable states during the time slot $(t\delta,(t+1)\delta]$ can be over-approximated as a continuous system without a neural network. In what follows, we focus on the computation of the over-approximation sets $X_0, X_1,...X_T$.

	\begin{figure}[t]
	\centering
	\includegraphics[width=1\textwidth]{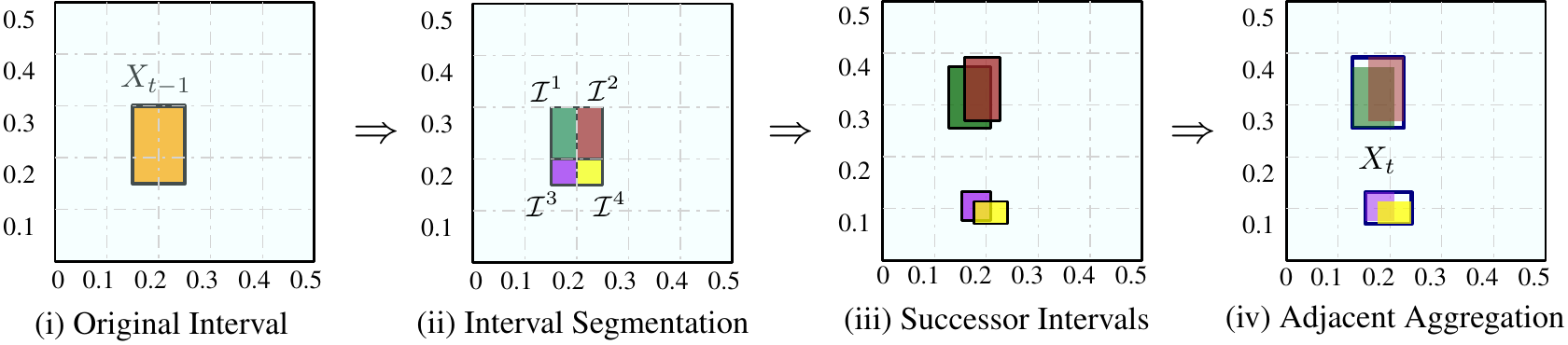}
	\vspace{-5mm}
	\caption{An example of over-approximating one-step reachable states.}
	\vspace{-4ex}
	\label{fig:onestep}
\end{figure}


Figure~\ref{fig:onestep} depicts an example of one time-step iteration. 
Without loss of generality, we suppose that $X_{t-1}$ is a singleton, e.g., $X_{t-1} = \{I\}$, where $I$ is an interval box. 
We segment $I$ into four smaller interval boxes (Figure~\ref{fig:onestep}(ii) and Line 5 of Algorithm~\ref{alg:reachsets_cal}) based on the abstraction function $\phi$ that is used for training the network. 
\begin{wrapfigure}[18]{r}{0.6\textwidth}
	\begin{algorithm}[H]
		\setlength{\floatsep}{1mm}
		\caption{Overall process.}
		\label{alg:reachsets_cal}
		\SetKwInOut{Input}{Input}
		\SetKwInOut{Output}{Output}
		\Input{Initial set $S_0$, ANN $\pi$, step size $\delta$, dynamics $f$,
			abstraction function $\phi$, maximal time step $T$}
		\Output{Over-approximation sets $\bigcup_{t=1}^{T}X_t$}
		
		Compute $I_0$ satisfying $S_0 \subseteq I_0$, $X_0 \leftarrow [I_0]$\\ 
		\ForEach{$t$ in $\{1,..., T\}$}{
			interval\_arr $\leftarrow$ \{\}\\
			\ForEach{$I$ in $X_{t-1}$}{
				$\mathbf{B}_{I}\leftarrow $  \textit{segment}($I$, $\phi$)\\
				\ForEach{$\mathcal{I}$ in $\mathbf{B}_{I}$}{
					$a\leftarrow\pi(\hat{s})$ for some $\hat{s}\in {\cal I}$\\
					${\cal I'}\leftarrow  post(\mathcal{I},a,f)$\\
					interval\_arr $\leftarrow$ interval\_arr $\cup \{\cal I'\}$ 
				}
			}
			$X_t$ = \textit{aggregate}(interval\_arr)
		}
            \Return $\bigcup_{t=1}^{T}X_t$
	\end{algorithm}
\end{wrapfigure}
\looseness=-1
We then compute the action for the states in each segmented interval box by arbitrarily selecting a state $\hat{s}$ in the box and then feeding $\hat{s}$ into $\pi$ to get the output (Line 7), e.g., $a$. Next, we compute a set ${\cal I'}$ of successor states of the states in $\cal I$ by over-approximating the environment dynamic $f$ in Formula \ref{for:post} (Figure~\ref{fig:onestep}(iii) and Line 8 of Algorithm \ref{alg:reachsets_cal}).
 Finally, we aggregate those adjacent successor interval boxes (Figure~\ref{fig:onestep}(iv) and Line 10 of Algorithm \ref{alg:reachsets_cal}) and obtain an over-estimated set $X_t$ of reachable states at time step $t$.

\vspace{-2mm}
\subsection{Key Operations in Algorithm \ref{alg:reachsets_cal}}
\vspace{-1mm}
We  now describe in detail three key operations  in Algorithm \ref{alg:reachsets_cal}, namely \textit{interval segmentation}, \textit{post operation}, and \textit{adjacent interval aggregation}. We fulfill the interval set propagation at each time step $t\in \mathbb{N}$ for the ANN-controlled systems based on these interval operations.

\vspace{1ex}
\noindent \textbf{Interval Segmentation.}
Given an interval box $I$ and an abstraction function $\phi$, $segment(I,\phi)$ returns 
a set $B_{I}$ of interval boxes which satisfy the following three \textit{segmentation conditions}: 
\begin{enumerate}
	\item All the interval boxes constitute $I$;
	\item Interval boxes do not overlap each other;
	\item All the states in the same interval box have a unique action according to the trained ANN. 
\end{enumerate}
\looseness=-1
For conventional DNNs, one has to resort to brute-force interval splitting to find consistent regions that satisfy the above three conditions;  this approach is only applicable to discrete action space  \cite{bacci2020probabilistic}. We can easily partition $I$ into such a set $B_I$, thanks to the specialized design of ANN. First, we determine the set of abstract states that intersect with $I$ and denote the set by  $\mathbf{S}_{I} = \{s_\phi \mid s_\phi \cap I \neq \emptyset \}$. We then calculate the intersection part between $I$  and the abstract states in $\mathbf{S}_{I}$ individually. Each intersection part is a segmented interval box. 
In this way, we	obtain a set of segmented interval boxes that satisfy the aforementioned three conditions and denote it by $\mathbf{B}_{I}= \{\mathcal{I} \mid \mathcal{I} = s_\phi \cap I \wedge s_\phi \in \mathbf{S}_{I} \}$. With the interval segmentation, through feeding an arbitrary state in the segmented interval box $\mathcal{I} \in B_I$ into ANN, we can obtain the corresponding unique action performed on $\mathcal{I}$. Since $B_I$ is a finite set, the decisions of the network controller on $I$ can be directly obtained without the layer-by-layer analysis process as in the white-box approaches~\cite{ivanov2021verisig,huang2022polar}. This makes our reachability analysis approach  a black-box one.

Recall the example in Figure~\ref{fig:onestep}(i), where the black dotted lines denote the partition of the state space with abstraction granularity $\gamma = (0.1,0.1)$. There exists an interval box $I=(0.15,0.25,0.15,0.3)$ that intersects with four abstract states. 
The intersection of each abstract state with $I$ is a segmented interval box. We have four  interval boxes $B_I=\{\mathcal{I}^1,\mathcal{I}^2,\mathcal{I}^2,\mathcal{I}^4\}$. Apparently, the segmented interval boxes in $B_I$ satisfy the three segmentation conditions.


\noindent \textbf{Post Operation.}
\looseness=-1
Given an interval box $\cal I$, the action $a$ applied to $\cal I$ and environment dynamics $f$, 
$post({\cal I},a,f)$ returns an interval box $\cal I'$, which is an over-approximation set of all the successor states by applying $a$ to the states in $\cal I$ after $\delta$. 


We can solve $post({\cal I},a,f)$ as an ordinary continuous system without neural networks. Suppose that the environment dynamics is an ODE $\dot{s} = f(s,a)$. We use a Taylor model $p'(s,a,t_c)+I_r'$ to over-approximate the function $\varphi_f(s,a,t_c)$ over the domain $s\in \mathcal{I}, t_c\in[0,\delta]$. That is, 
\begin{equation}
\nonumber
\setlength{\abovedisplayskip}{4pt}
\setlength{\belowdisplayskip}{4pt}
Reach_f^{[0,\delta]}(\mathcal{I}) = \bigcup_{s\in {\cal I}, t_c\in[0,\delta]} \{\varphi_f(s,a,t_c)\} \subseteq  p'(s,a,t_c)+I_r', \notag 
\end{equation}
where $I_r'$ is a remainder interval. The successor interval box $\mathcal{I}'$ can be calculated through evaluating the range of $p'(s,a,\delta)+I_r'$. 


 Let us consider an example for the segmented interval box $\mathcal{I}^1 = (0.15,0.2,0.2,\\0.3)$ in Figure~\ref{fig:onestep}(ii). The dynamics is defined as in Example~\ref{exa:DRLexample}. Suppose the action for the states in the interval box is $a = 0.5$ and the time scale $\delta = 0.1$. We can compute an over-approximated Taylor model for the solution of dynamics $f: \dot{x_1} = x_2- x_1^3, \dot{x_2} = 0.5$ over $s\in \mathcal{I}^1, t_c\in [0,0.1]$. 
The Taylor models for state variable $x_1,x_2$ are as follows:
\vspace{-2mm}
\begin{align}
	\nonumber
	x'_1=&\ 1.75\times 10^{-1}+1.91\times  10^{-8}x_2+2.5\times 10^{-2}x_1+0.245t_c\\ 
 &-1.25\times 10^{-10}x^2_1+5\times 10^{-2}x_2t_c-2.3\times 10^{-3}x_1t_c +0.239t^2_c\nonumber\\
	&-5\times 10^{-10}x_1x_2t_c+\ldots+[-1.03\times 10^{-4},8.94\times 10^{-5}]\nonumber\\
	x'_2=&\ x_2 + 0.5t_c + [-0,0]\nonumber
\end{align}
%
%
%

Using these two expressions, we can over-approximate the set of reachable states at every moment during $[0,0.1]$. In particular, we have $(0.172,0.232,0.25,\\0.35)$ when $t_c=0.1$. 

\noindent\textbf{Adjacent Interval Aggregation.}
Interval segmentation may lead to the exponential  blowup in the number of intervals as the number of time steps increases. As exemplified in Figure~\ref{fig:onestep}(iii), four successor intervals are obtained after applying corresponding actions and environment dynamics to the states in ${\cal I}^1,\ldots,{\cal I}^4$. 
\begin{wrapfigure}[21]{r}{0.55\textwidth}
\vspace{-4mm}
	\begin{algorithm}[H]
		\footnotesize
		\SetKwInOut{Input}{Input}
		\SetKwInOut{Output}{Output}
		\Input{An interval array $IntArr$}
		\Output{The aggregation results $Arr$}
		\caption{Adjacent interval aggregation.}
		\label{alg:interval_agg}
		Initialize flag $\leftarrow$ [False, False,...], $Arr$ $\leftarrow$ [] \\
		Construct the adjacency matrix $M$
		\\
		\ForEach{$I_{p}$ in $IntArr$}{
			\If{not flag[$I_p$]}{
				Initialize queue $\leftarrow$ [$I_p$]\\
    flag[$I_p$] $\leftarrow$ True\\
				\While{queue is not empty}{
					$I$ $\leftarrow$ queue.pop()\\
					$I_{adjs}$ $\leftarrow$ getAdjacent($I$, $M$)\\
					\ForEach{item in $I_{adjs}$}{
						$I_p$ $\leftarrow$ aggInterval($I_p$, item)\\
						\If{not flag[item]}{
							queue.put(item)\\
                                flag[item] $\leftarrow$ True\\
						}
						
					}
				}
				$Arr$.add($I_p$)\\
			}
		}
		\Return $Arr$

	\end{algorithm}
\end{wrapfigure}

To cope with the explosion of successor intervals, we provide a dual operation of segmentation called \textit{adjacent interval aggregation}, which aggregates multiple intervals together at the price of introducing a little overestimation. This operation is based on the interval hull operation~\cite{moore2009introduction} except that we establish a criterion for determining which intervals can be aggregated into their interval hull.
For instance, the green and brown intervals in Figure~\ref{fig:onestep}(iii) can be aggregated, while the other small ones can be aggregated too.  
However, large overestimation would be introduced if the four interval boxes were aggregated to be one. 

To balance the number of intervals and the overestimation introduced by aggregation, we define three cases for the adjacency relation between interval boxes, i.e., \textit{inclusion}, \textit{intersection}, and \textit{separation}. Only the intervals in the three cases are aggregated. 
Given two interval boxes $A=(l_1,u_1,\dots,l_n,u_n)$ and  $B=(l'_1,u'_1,\dots,l'_n,u'_n)$, as well as 
a preset distance threshold $h = (h_1,\dots,h_n)$, the three cases are defined as follows: 
%
%
%
%

\vspace{-2mm}
\begin{enumerate}
    \item \textbf{Inclusion:} An interval box is completely included in the other, i.e.,     
     $\forall i: (l_i \le l'_i \wedge u_i \ge u'_i) \vee 
    (l_i \ge l'_i \wedge u_i \le u'_i$).
    
    \item \textbf{Intersection:} $A$ and $B$ have a partial overlap, i.e.,  
    $\exists ! d : l'_d \le l_d \le u'_d \le u_d \vee l_d \le l'_d \le u_d \le u'_d$ and 
    $\forall i, i \neq d: \lvert l_i - l'_i \rvert \le h_i \wedge \lvert u_i - u'_i \rvert \le h_i$.
    
    \item \textbf{Separation:} $A$ is isolated from $B$,
    i.e., 
         $\exists ! d : l_d - u'_d \le h_d \vee l'_d - u_d \le h_d$; and 
        $\forall i, i \neq d  : \lvert l_i - l'_i \rvert \le h_i \wedge \lvert u_i - u'_i \rvert \le h_i$. 
\end{enumerate}
\vspace{-1mm}
	





To accelerate interval aggregation, we 
devise an efficient algorithm to 
aggregate three or more interval boxes each time if they constitute a sequence of adjacent intervals. 
Algorithm~\ref{alg:interval_agg} shows the pseudo code.  
We first pre-construct an adjacency matrix (Line 2) to store the adjacent relations between the interval boxes in $IntArr$ firstly. 
Then, we implement this adjacent interval aggregation procedure using breadth-first search (Lines 5-14). Specifically, we consider each interval box in $IntArr$ as a node and each adjacent relation as an undirected edge. For each interval box $I_p$ that is not traversed, all the  interval boxes connected to $I_p$ will be aggregated into their minimum bounding rectangle.



In Algorithm \ref{alg:interval_agg}, 
the time complexity of building the adjacency matrix is $O(n^2)$. In the aggregation procedure, each interval box is traversed at most once, and the complexity of searching for the adjacent interval boxes for each interval box is $O(n)$. Therefore, Algorithm \ref{alg:interval_agg} is in $O(n^2)$. 


\begin{example}
Let us revisit the system in Example~\ref{exa:DRLexample} and  suppose that $IntArr$ consists of 4 interval boxes, i.e.,  $\widehat{I}_1 =(0.08, 0.16, 0.3,0.4)$, $\widehat{I}_2=(0.17,0.25, 0.32, 0.42)$, $\widehat{I}_3=(0.19,0.27,0.07,0.2)$, $\widehat{I}_4=(0.2,0.28,0.1,0.21)$, and the distance threshold is $h = (0.02,0.02)$. According to the definition of adjacent relations,  $\widehat{I}_1$ is adjacent to $\widehat{I}_2$ (Separation) and  $\widehat{I}_3$ is adjacent to $\widehat{I}_4$ (Intersection). Hence, $\widehat{I}_1$ is aggregated with $\widehat{I}_2$, and $\widehat{I}_3$ is aggregated with $\widehat{I}_4$. Finally, we obtain $Arr = \{I_{1,2} = (0.08,0.25,0.3,0.42), I_{3,4} = (0.19,0.28,0.07,0.21)\}$.
\end{example}

\vspace{-4mm}
\subsection{The Soundness}\label{sec:sound}
\vspace{-1mm}
We show a proof sketch for the soundness of Algorithm \ref{alg:reachsets_cal}. The soundness means that any state that is reachable at time $t_c$ from some initial state of an ANN-controlled system must be in the over-approximation set at $t_c$. 
\vspace{-2mm}

\begin{theorem}[Soundness of Algorithm \ref{alg:reachsets_cal}]
	\label{correctness}
	Given an ANN-controlled system with a set $S_0$ of initial states and an environment dynamic $f$, 
	if a state $s'$ is reached at time $t_c =k\delta+ t_c', k\in \mathbb{N}, t_c' \in[0,\delta)$ from some initial state $s_0\in S_0$, then we must have $s'= Reach_{f}^{t_c}({s_0}) \in Reach_f^{t_c'}(X_k)$. 
\end{theorem}

To prove Theorem \ref{correctness}, we first show the soundness of the $post$ operation and interval aggregation.
The soundness of the two operations is formulated by the following two lemmas, respectively. 


\begin{lemma}[Soundness of $\mathit{post}$ Operation]
\label{lem:seg}
For each interval box $\mathcal{I} \in B_{I}$, there is $s_{t+1} \in post(\mathcal{I},\pi(s_t),f)$ for all $s_t \in \mathcal{I}$ where $s_{t+1}=\varphi_f(s_t,\pi(s_t),\delta)$.
\end{lemma}

\begin{proof}
	\vspace{-1ex}
	After the segmentation process, we have $\forall s \in \mathcal{I}: \pi(s) = \pi(s_t) = a$ where $a$ is a constant. With a constant action and the Lipschitz continuity of $f$, we can guarantee that there exists a unique solution of the ODE for a single initial state~\cite{meiss2007differential}. Then the solution of the ODE namely $\varphi_{f}(s,a,t_c)$ could be enclosed by a Taylor model~\cite{makino2003taylor} over $s(0)\in \mathcal{I}$ and $t_c \in [0,\delta]$. Thus, we could obtain the conservative result $s_{t+1} = \varphi_{f}(s_t,a,\delta)\in Reach_f^{\delta}(\mathcal{I}) \subset post(\mathcal{I},a,f)$. \qed 
\end{proof}



\begin{lemma}[Soundness of Interval Aggregation]
\label{lem:aggr}
Suppose $A$ is the aggregated set of successor intervals for a set $X$ of interval boxes. For all $I \in X$, there exists 
$\widehat{I} \in A$ such that $I \subseteq \widehat{I}$.
\end{lemma}

\begin{proof}
	In Algorithm~\ref{alg:interval_agg}, every interval box in $X$ needs to be traversed. For each interval box $I \in X$, there exist two cases: (i) $I$ is not involved in the adjacent interval aggregation process. In this case, $I$ will be directly added to $A$, thus $\exists \widehat{I} = I: I \subseteq \widehat{I}$. (ii) $I$ is aggregated into another interval box $I'$. Since the $aggregate$ operation produces the minimum bounding rectangle which encloses all interval boxes involved, we have $\exists \widehat{I} = I': I \subseteq \widehat{I}$. Consequently, we conclude that $\forall I \in X,\ \exists \widehat{I}:\ I \subseteq \widehat{I} \wedge \widehat{I} \in A$. \qed 
	\vspace{-1ex}
\end{proof}


According to  Algorithm \ref{alg:reachsets_cal}, Theorem \ref{correctness} can be proved by induction on the steps $t_c$ based on Lemmas~\ref{lem:seg} and \ref{lem:aggr}. The base case is straightforward when $t_c=0$. 
In the induction case, we can prove 
that Theorem~\ref{correctness} holds on $[t\delta,(t+1)\delta]$ according to the two lemmas and the hypothesis that it holds on an arbitrary  $t_c=t\delta$. 

\begin{proof}[Theorem \ref{correctness}]
	Starting from $s_0$, we can obtain the trajectory as $s_0, a_0, s_1,$ $a_1, ...$ in which $a_t = \pi(s_t)$ and $s_{t+1} = \varphi(s_t,a_t,\delta)$.
	Then by induction on the
	time step $t$, the induction schema is as follows: 
	
	\vspace{1ex}
	\noindent \textbf{Base Case:} $t_c = 0$. Since $s_0 \in S_0 \wedge S_0 \subseteq X_0$, we have $s_0 \in X_0 = Reach_f^0(X_0)$.
	
	\vspace{1ex}
	\noindent \textbf{Induction Step:} $t_c = t\delta$. Assume $s' = s_t \in X_t = Reach_f^{0}(X_t)$ holds. 
	Since $X_t$ consists of a set of interval boxes, there exists an interval box $I_{X_t}^{n_1}$ satisfying $s_t \in I_{X_t}^{n_1} \wedge I_{X_t}^{n_1} \in X_t$. 
	Then, let us consider the segmentation process for $I_{X_t}^{n_1}$, we divide $I_{X_t}^{n_1}$ into a set of interval boxes $\mathbf{B}_{I_{X_t}^{n_1}} = \{ \mathcal{I}_{X_t}^{1}, \mathcal{I}_{X_t}^{2}, \dots, \mathcal{I}_{X_t}^{max}\}$ with $I_{X_t}^{n_1} = \bigcup\limits_{n=1}^{max} \mathcal{I}_{X_t}^{n}$. Thus, there exists some $n_2 \in \mathbb{Z^+}$ such that $s_t \in \mathcal{I}_{X_t}^{n_2}$. 
	
	For $t_c \in [t\delta,(t+1)\delta)$, we have $s' = Reach_f^{t_c'}(s_t)$. Since $s_t \in \mathcal{I}_{X_t}^{n_2}$, we have $s' =  Reach_f^{t_c'}(s_t) \in  Reach_f^{t_c'}(\mathcal{I}_{X_t}^{n_2}) \subseteq  Reach_f^{t_c'}(X_t) $.
	
	For $t_c = (t+1)\delta$, we have $s' = s_{t+1}$,
	Based on Lemma~\ref{lem:seg}, we have $s_{t+1} \in post(\mathcal{I}_{X_t}^{n_2},\pi(s_t),f)$. 
	After the adjacent interval aggregation process, $X_{t+1}$ consists of the aggregation result. According to Lemma~\ref{lem:aggr}, we have $\exists \widehat{I}: post(\mathcal{I}_{X_t}^{n_2},\pi(s_t),f) \subseteq \widehat{\mathcal{I}} \wedge \widehat{\mathcal{I}} \in X_{t+1}$.
	Therefore, we have $s_{t+1} \in \widehat{\mathcal{I}} \wedge \widehat{\mathcal{I}} \in X_{t+1}$ and we can conclude that $s' = s_{t+1} \in X_{t+1} = Reach_f^0(X_{t+1})$.

Theorem \ref{correctness} is proved. \qed 	
\vspace{-3mm}
\end{proof}

\iftoggle{conf-ver}{
Due to space limit, the complete proofs of the theorem and lemmas are omitted in the paper but available in Appendix A of our technical report~\cite{ase2023tian-tr}. }{
}

%% file: experiment.tex
\vspace{-2mm}
\section{Implementation and Experiments}
\vspace{-1mm}
\label{sec:exp}
We conduct a comprehensive
assessment of our approach and compare it with the state-of-the-art white-box tools.
Our goal is to demonstrate the advances of the proposed abstraction-based training and black-box reachability analysis approaches. These include (i) comparable performance of trained systems and negligible time overhead in the training~(Section \ref{sub:perf}),  		
(ii) tighter over-approximated sets of reachable states, as well as higher scalability and efficiency   (Section~\ref{subsec:tightness}), and (iii) the effectiveness of the adjacent interval aggregation algorithm in reducing state explosion  (Section~\ref{subsec:discussion}).
We also explore how our approach performs under different abstraction granularity levels (Section~\ref{subsec:discussion}). 




\vspace{-1mm}
\subsection{Implementation and Benchmarks}
\vspace{-1mm}

\noindent \textbf{Implementation.}
We implement our approach in a tool called \BBReach~in Python. We use Ariadne~\cite{collins2012computing} to solve the reachability problems defined on segmented interval boxes (i.e., $post(\mathcal{I},a,f), \mathcal{I} \in B_{I}$). Additionally, we employ the parallelized computing by initial-set partition~\cite{chen2012taylor},  a standard approach used in the reachability analysis of hybrid systems to obtain tighter bounds of reachable states. 
With  the initial set partitioned into $k$ subsets, 
the $k$ sub-problems can be solved in parallel, which accelerates our approach with multiple cores. 
\vspace{1ex}
\noindent \textbf{Benchmarks. 
}
The benchmarks, as commonly adopted by most of the existing reachability analysis approaches such as Verisig 2.0~\cite{ivanov2021verisig} and Polar \cite{huang2022polar}, consist of seven reinforcement learning tasks with the dimensions ranging from 2 to 6. A reach-avoid property is defined for each task by specifying the goal region and unsafe region of the agent in the task. A trained DNN must guarantee that the reach-avoid property is satisfied when the agent is driven by the DNN. 

For each task, we train four neural networks (two smaller networks chosen from~\cite{ivanov2021verisig} and two larger networks), thus 28 instances in total,  with different activation functions and sizes of neurons. We also train the networks with different abstraction granularity levels to evaluate how abstraction granularity affects the efficiency.  We use the well-known DDPG  algorithm to train neural networks. Note that our approach makes no assumption on training algorithms and is applicable to other DRL algorithms. 
\iftoggle{conf-ver}{
The detailed settings are provided in Appendix C of technical report~\cite{ase2023tian-tr}.
}{
The detailed settings are provided in  Appendix~\ref{sec:benchmarks}. }

\begin{figure*}[t]
	\hspace{-2mm}
	\setlength{\tabcolsep}{.8mm}{
		\begin{tabular}{cccc}
			\begin{subfigure}[b]{0.24\textwidth}
				\includegraphics[width=\textwidth]{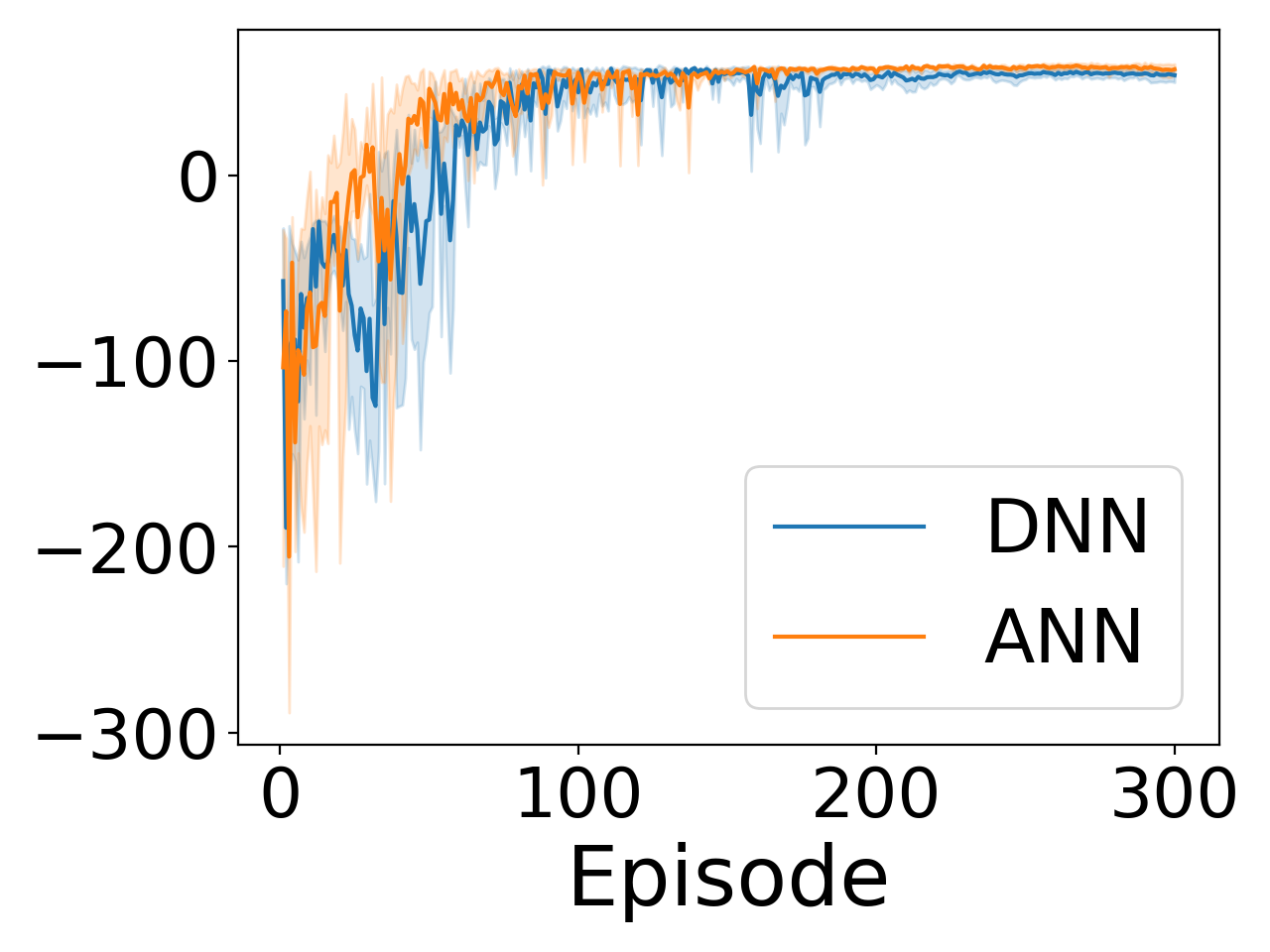}
				\caption{B1}
				\label{fig:reward_com_b1}
			\end{subfigure}
			&
			\begin{subfigure}[b]{0.24\textwidth}
				\includegraphics[width=\textwidth]{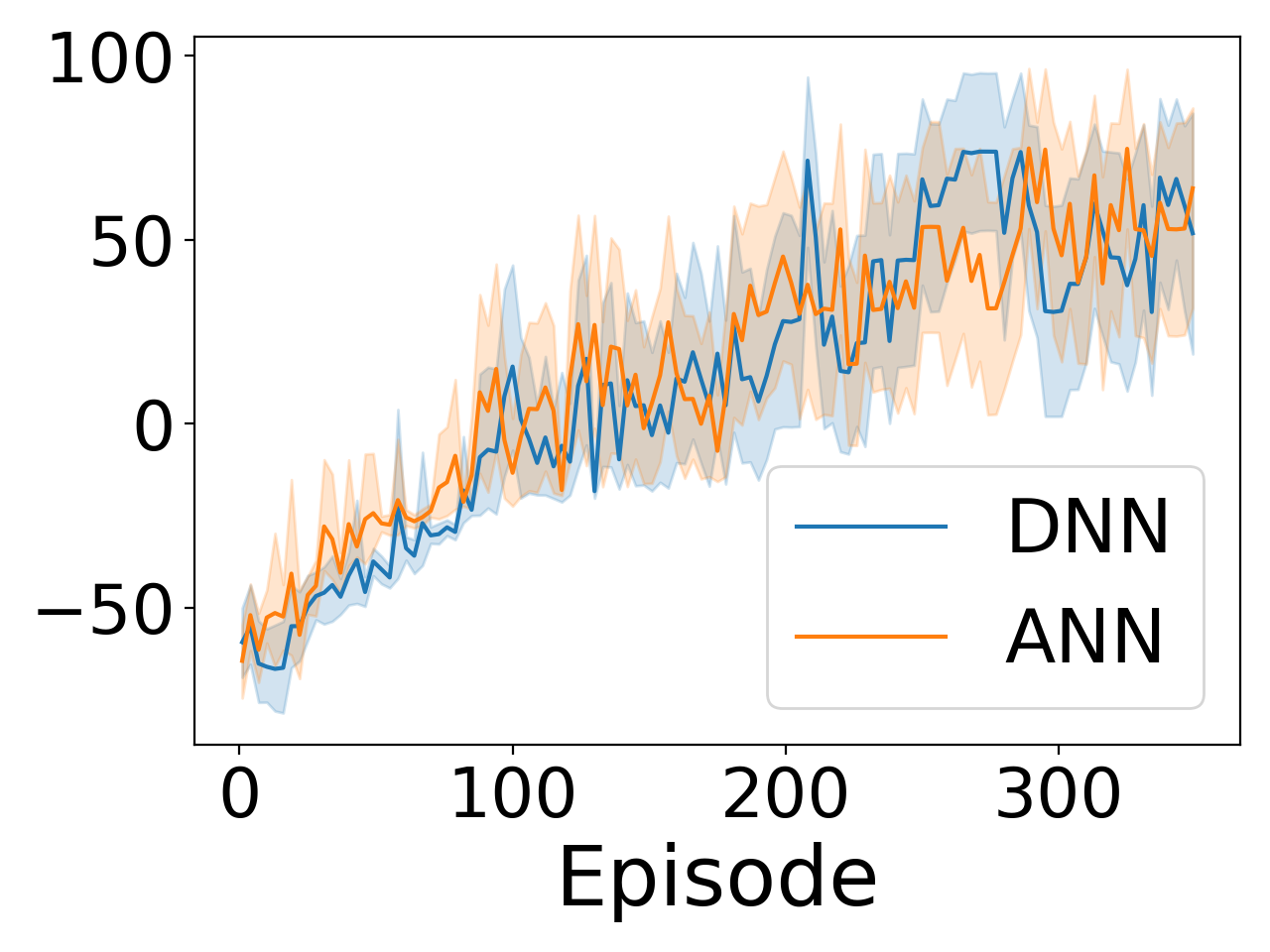}
				\caption{B2}
				\label{fig:reward_com_b2}
			\end{subfigure}&
			\begin{subfigure}[b]{0.24\textwidth}
				\includegraphics[width=1\textwidth]{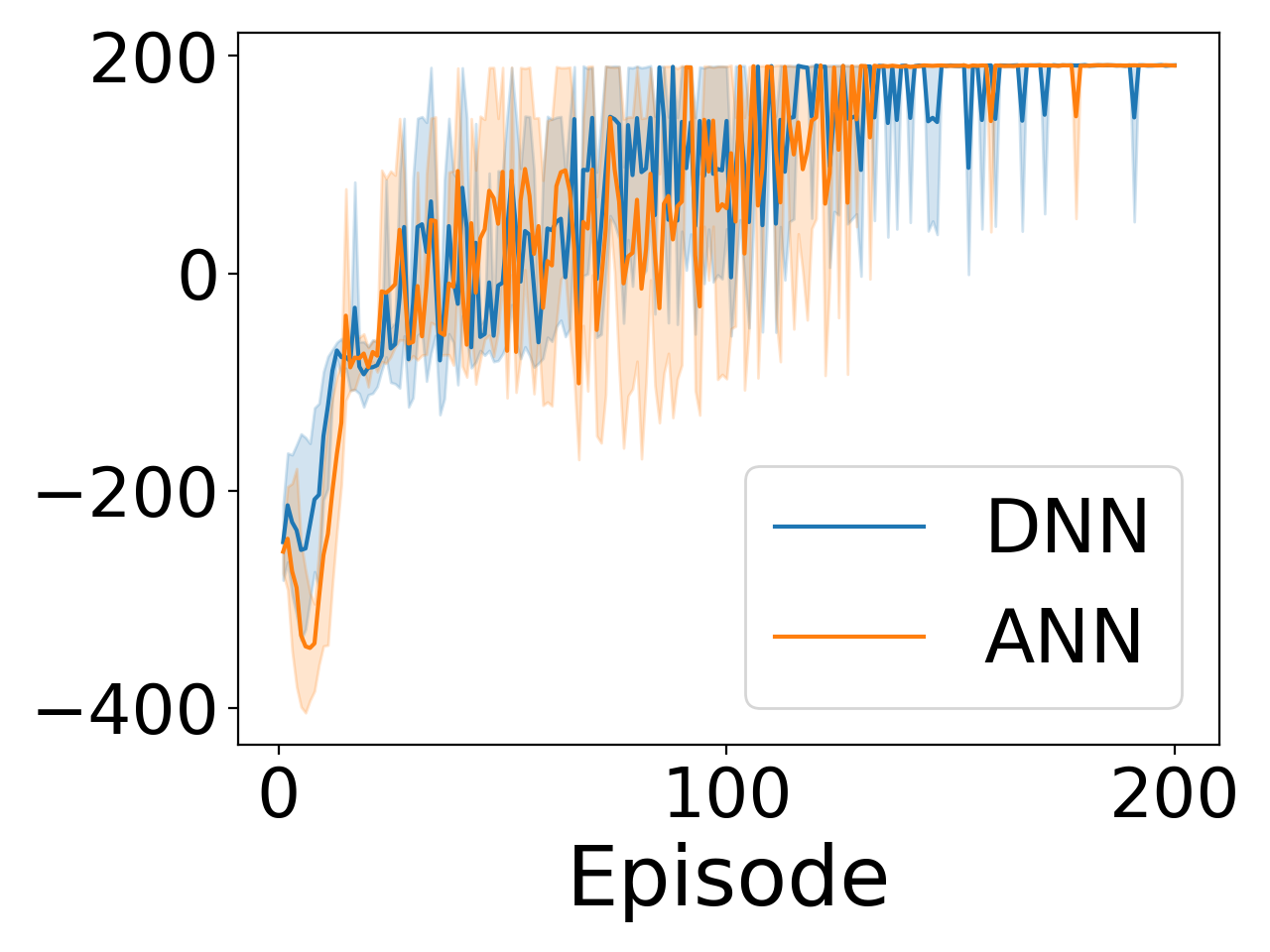}
				\caption{B3}
				\label{fig:reward_com_b3}
			\end{subfigure}
			&
			\begin{subfigure}[b]{0.24\textwidth}
				\includegraphics[width=\textwidth]{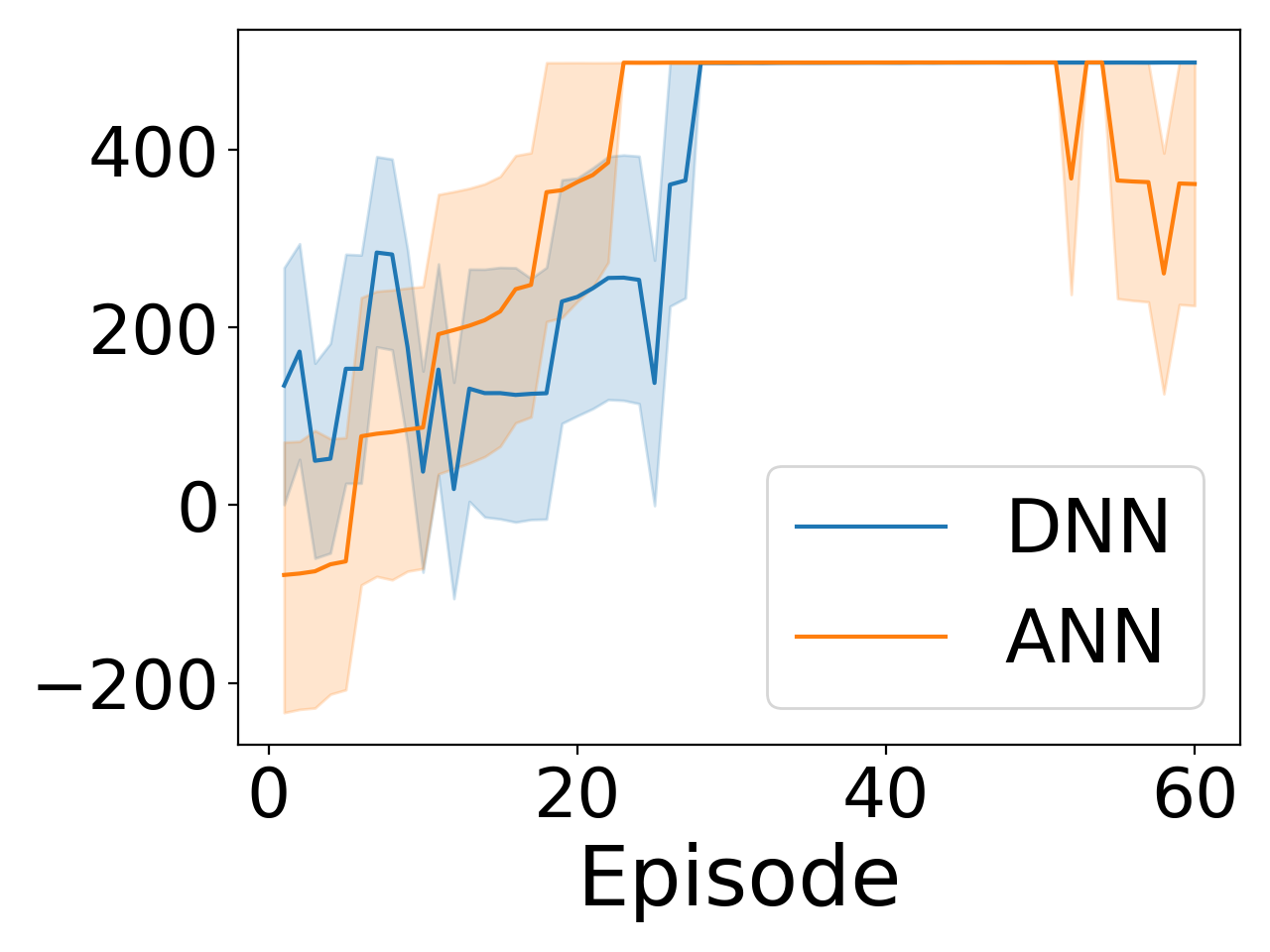}
				\caption{B4}
				\label{fig:reward_com_b4}
			\end{subfigure}
	\end{tabular}}
	\vspace{-2mm}
	\caption{Trend  of cumulative rewards (y-axis) of the systems controlled by ANNs (orange) and DNNs (blue) trained by  DDPG.}
	\label{fig:cumulative_reward}
	\vspace{-5mm}
\end{figure*}


\vspace{1ex}
\noindent \textbf{Experimental Setup.}
All experiments are conducted on a workstation equipped with a 32-core AMD Ryzen Threadripper CPU @ 3.6GHz and 256GB RAM, running Ubuntu 22.04.

\vspace{-1mm}
\subsection{Performance of Trained Neural Networks}
\label{sub:perf}
\begin{wrapfigure}[10]{r}{0.34\textwidth}
	\vspace{-2cm}
	\begin{minipage}{0.34\textwidth}
		\begin{table}[H]
			\centering
			\caption{Training time (s).}
			\setlength{\tabcolsep}{8pt}
			\label{tab:train_time_comparison}	
			\begin{tabular}{c|rr}
				\toprule
				\textbf{Task}   &   ANN        	&	DNN         \\\midrule
				\textbf{B1}    	& 13.7       	& 11.0       	  \\   
				\textbf{B2} 	& 7.4     		& 6.6     		  \\   
				\textbf{B3} 	& 6.4     		& 5.1     		  \\   
				\textbf{B4} 	& 5.8     		& 3.2     		  \\  
				\textbf{B5} 	& 57.4     		& 49.8     		    \\ 
				\textbf{Tora} 	& 	47.2      	& 44.3      	 \\    
				\textbf{ACC} 	&	23.4     	& 21.6     	\\ \bottomrule
				
			\end{tabular}
			\vspace{-4mm}
		\end{table}
	\end{minipage}
\end{wrapfigure}
We show that the extended abstract neural networks can be trained to achieve comparable performance against those conventional ones that have the same architectures and activation functions and are trained in the same approach. For each case, we train 5 times and record the cumulative reward during the training process with and without the abstraction layer. Figure~\ref{fig:cumulative_reward} unfolds a comparison of the trend of cumulative rewards during training between these two training approaches in B1-B4 (the other three are given in Appendix~\ref{app:train_reward_com}). The solid lines and the shadows indicate the average reward and 95\% confidence interval, respectively. The results show that an extended abstract neural network can make near-optimal decisions even under the constraint that it must yield the same action on each partitioned interval box. 
Importantly, the abstraction-based training incurs little and negligible time overhead only in several seconds, as shown in  Table~\ref{tab:train_time_comparison}. 


\subsection{Tightness and Efficiency}
\label{subsec:tightness}

We compare the tightness of the over-approximated reachable states by plotting the over-approximation sets computed by our approach and the state-of-the-art white-box tools including Polar \cite{huang2022polar} and Verisig 2.0 \cite{ivanov2021verisig}.
Because \BBReach~ is designed for ANNs-controlled systems, while the white-box tools are for DNNs-controlled ones, the policy models for each task are different. To make the comparison as fair as possible, 
we use the same network architecture to train the ANN and DNN for the same task except that the ANN includes an additional abstraction layer. 
We also guarantee that all the trained systems can achieve the best cumulative reward for the same task. For instance, we initialize the neural networks with smaller weights as otherwise Verisig 2.0 would introduce larger over-approximation error
(\iftoggle{conf-ver}{see Appendix B in~\cite{ase2023tian-tr}}{see our observations  in Appendix~\ref{subsec:big_weights}}). 
In particular, we also simulated the trained systems and recorded trajectories as the baseline.



\begin{figure}[t]
	\begin{minipage}[b]{0.08\linewidth}
\begin{tikzpicture}
\tikz \node [draw] at (0.3,-1)  {B1};
\end{tikzpicture}
	\end{minipage}
	\begin{minipage}[b]{0.29\linewidth}
		\centering
		{ \textbf{\BBReach}}\\
		 \vspace{1ex}
		\includegraphics[scale=0.20]{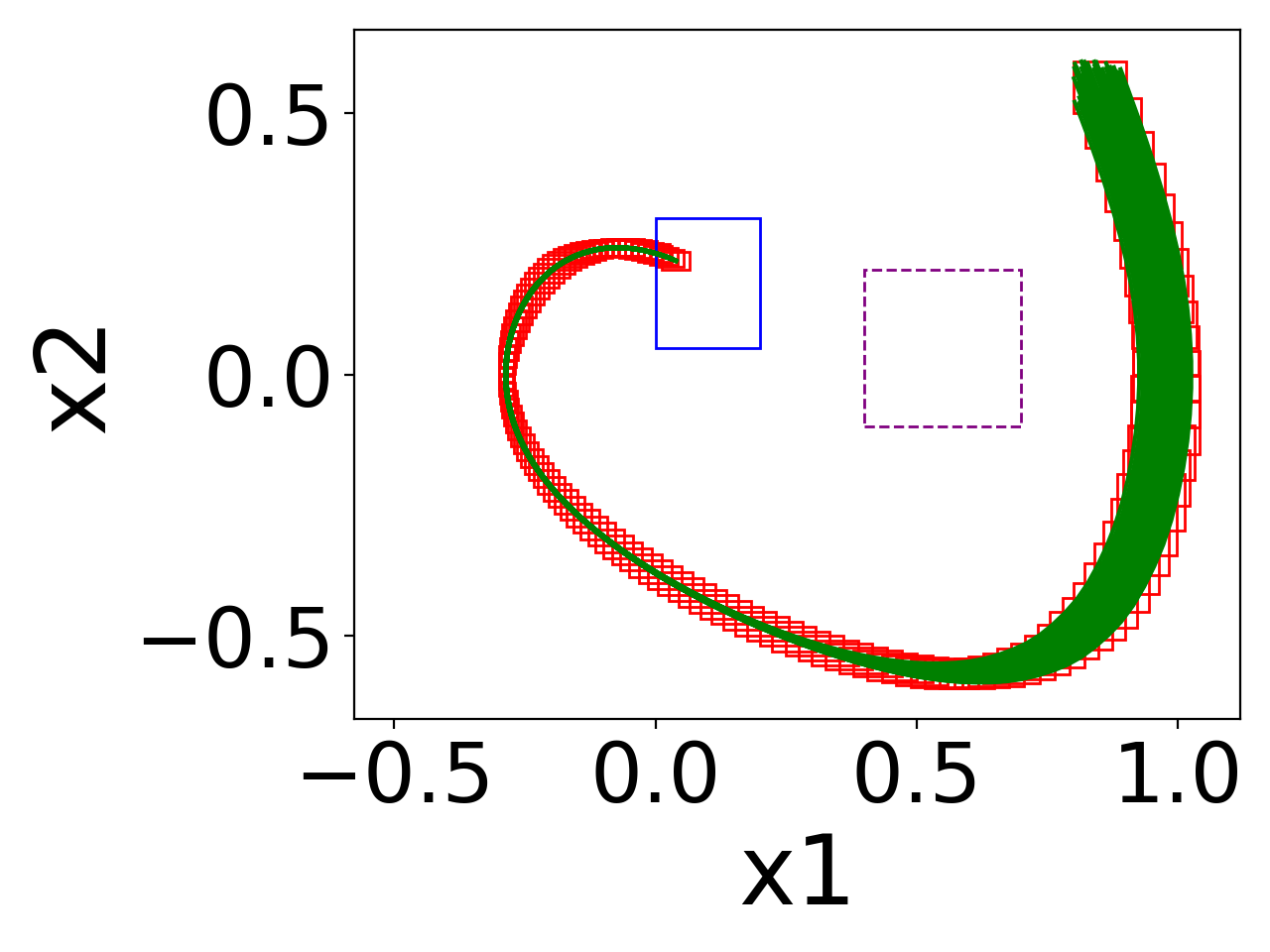}
	\end{minipage}
		\begin{minipage}[b]{0.28\linewidth}
		\centering
		{ \textbf{Verisig 2.0}}\\
		\vspace{1ex}
		\includegraphics[scale=0.20]{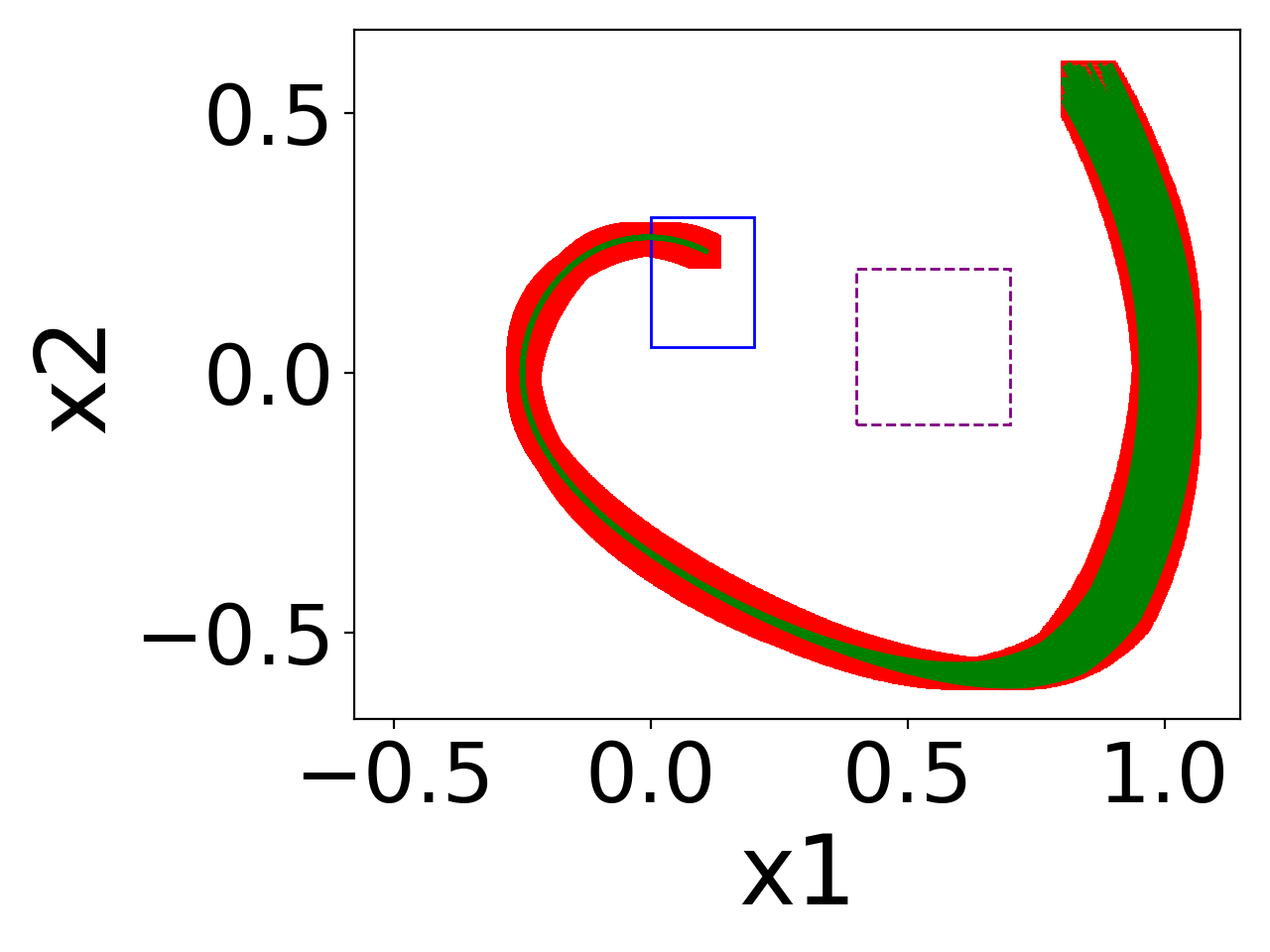}
	\end{minipage}
	\begin{minipage}[b]{0.29\linewidth}
		\centering
	{ \textbf{Polar}}\\
	\vspace{1ex}
	\includegraphics[scale=0.20]{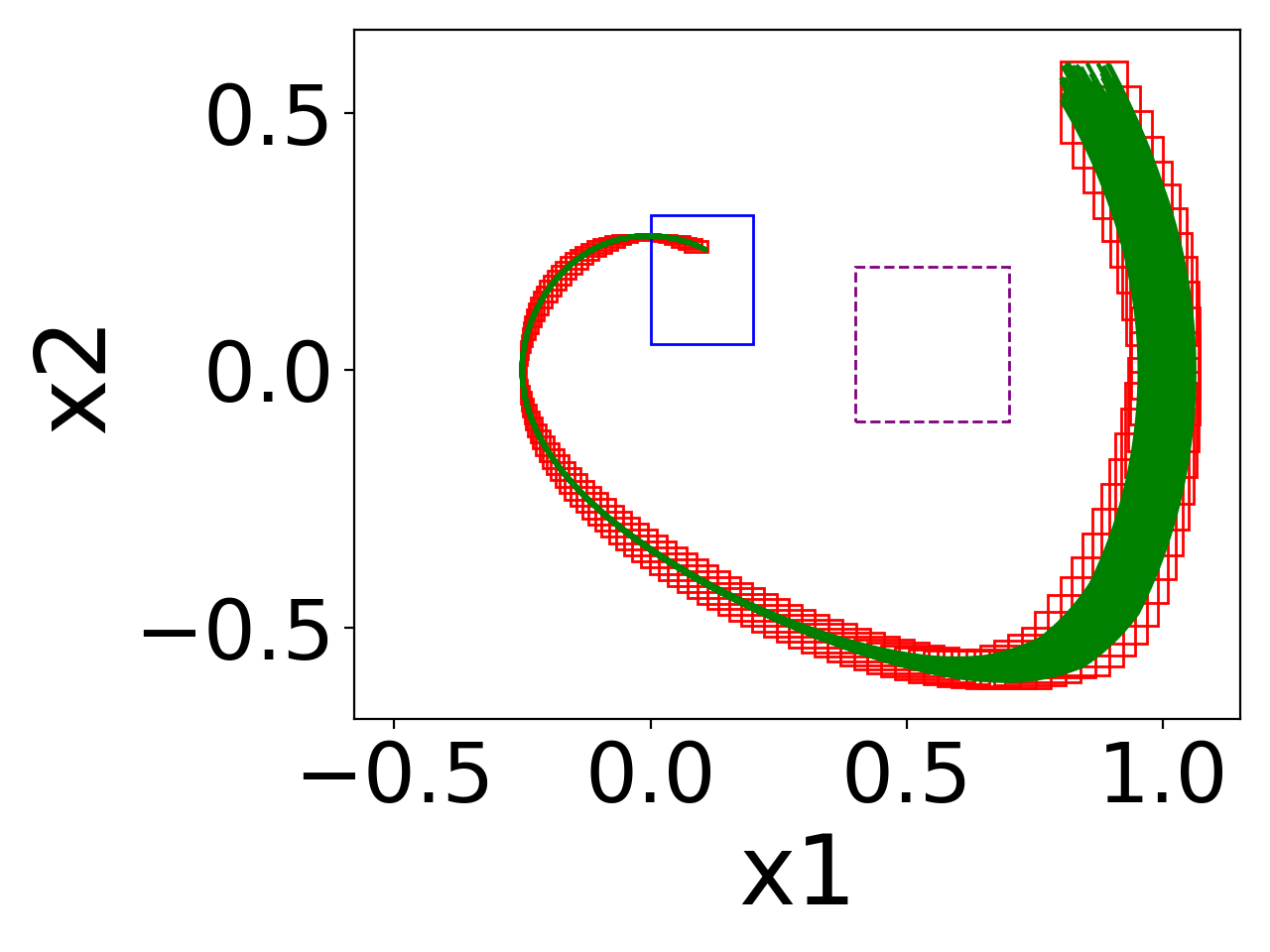}
	\end{minipage}
\\
\vspace{1ex}
	\begin{minipage}[b]{0.07\linewidth}
\begin{tikzpicture}
\tikz \node [draw] at (0.4,-1)  {B2};	
\end{tikzpicture}
	\end{minipage}
	\begin{minipage}[b]{0.29\linewidth}
		\centering
		\includegraphics[scale=0.20]{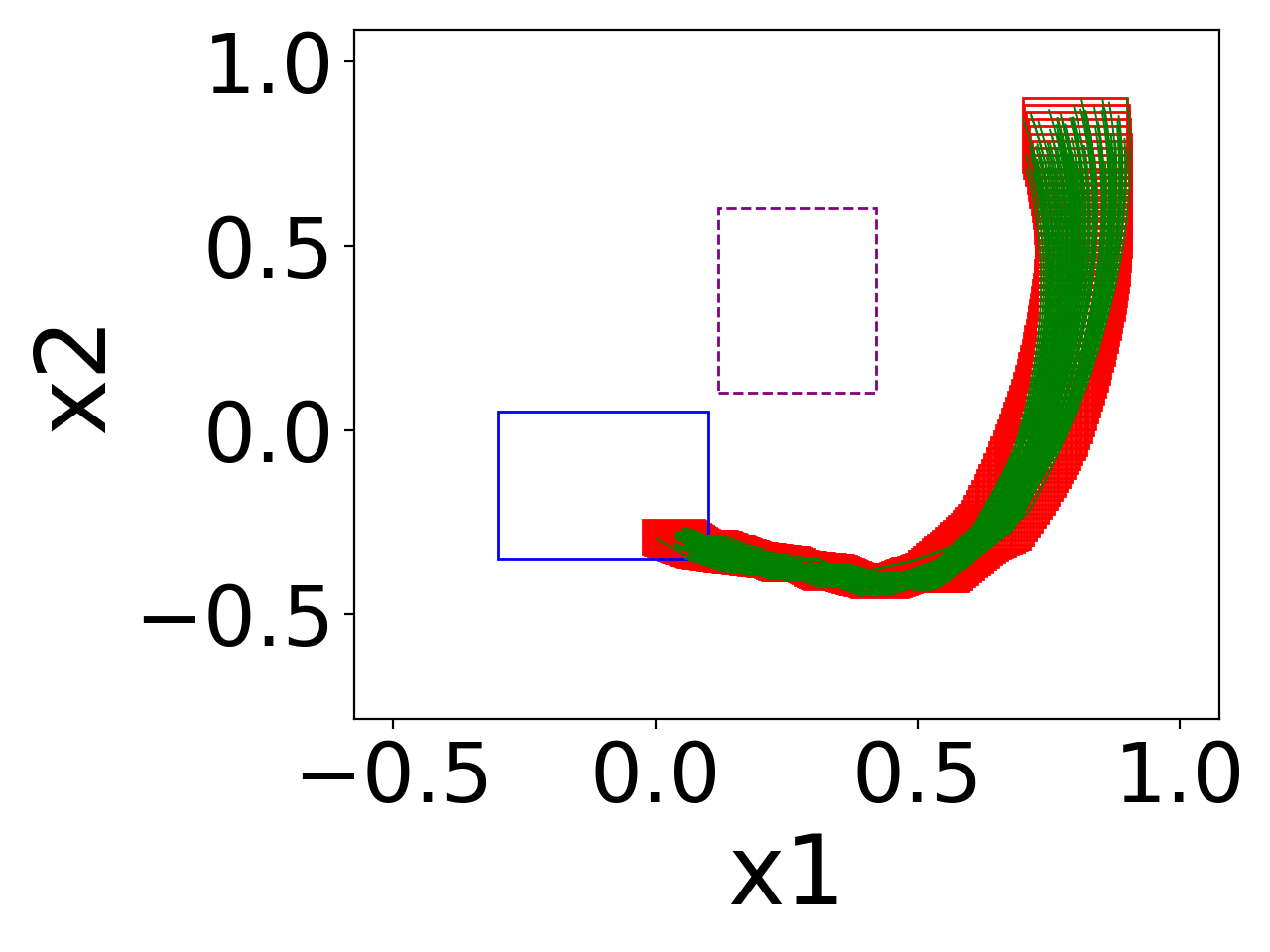}
	\end{minipage}
		\begin{minipage}[b]{0.29\linewidth}
		\centering
		\includegraphics[scale=0.20]{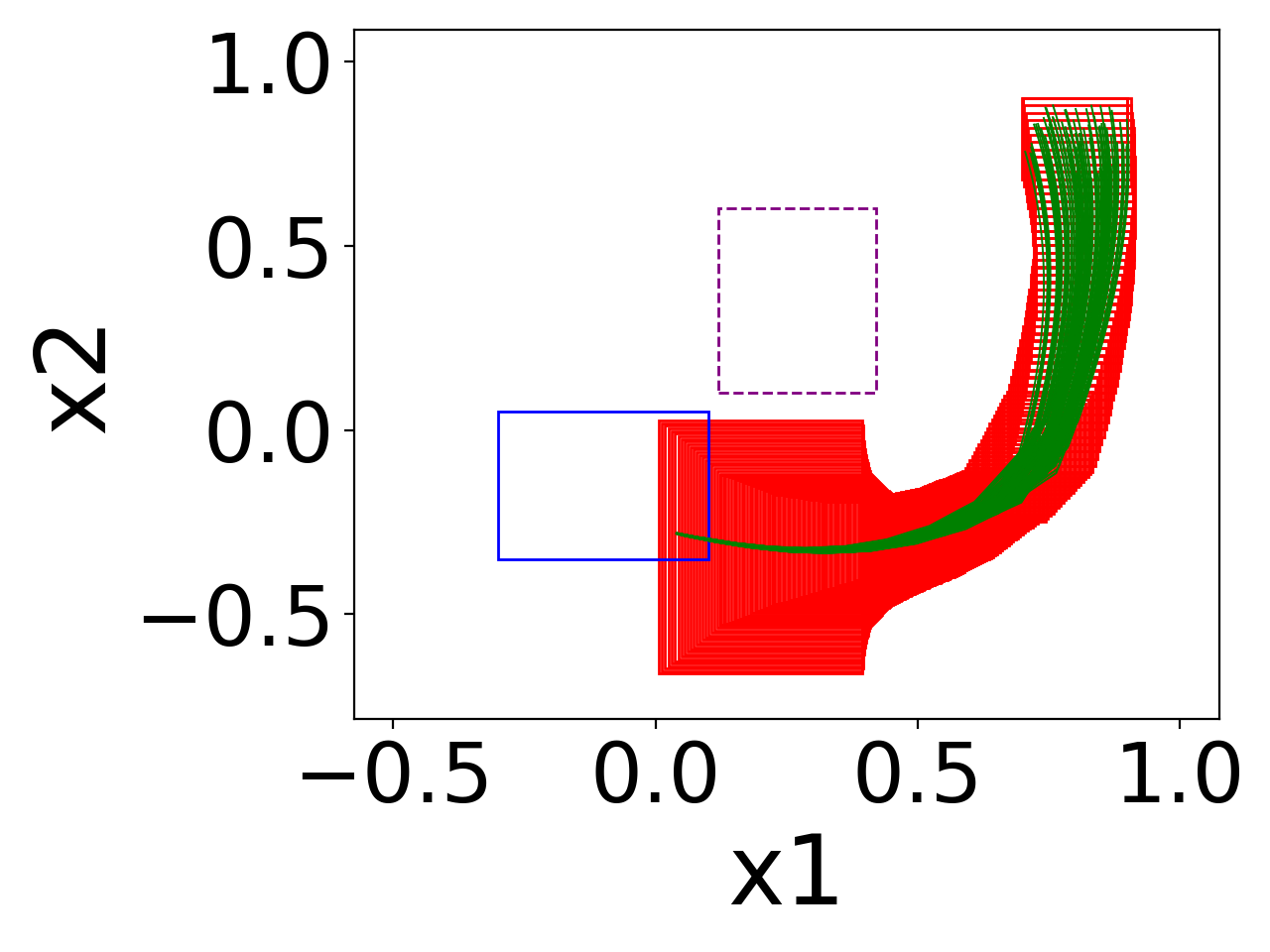}
	\end{minipage}
	\begin{minipage}[b]{0.29\linewidth}
		\centering
	\includegraphics[scale=0.20]{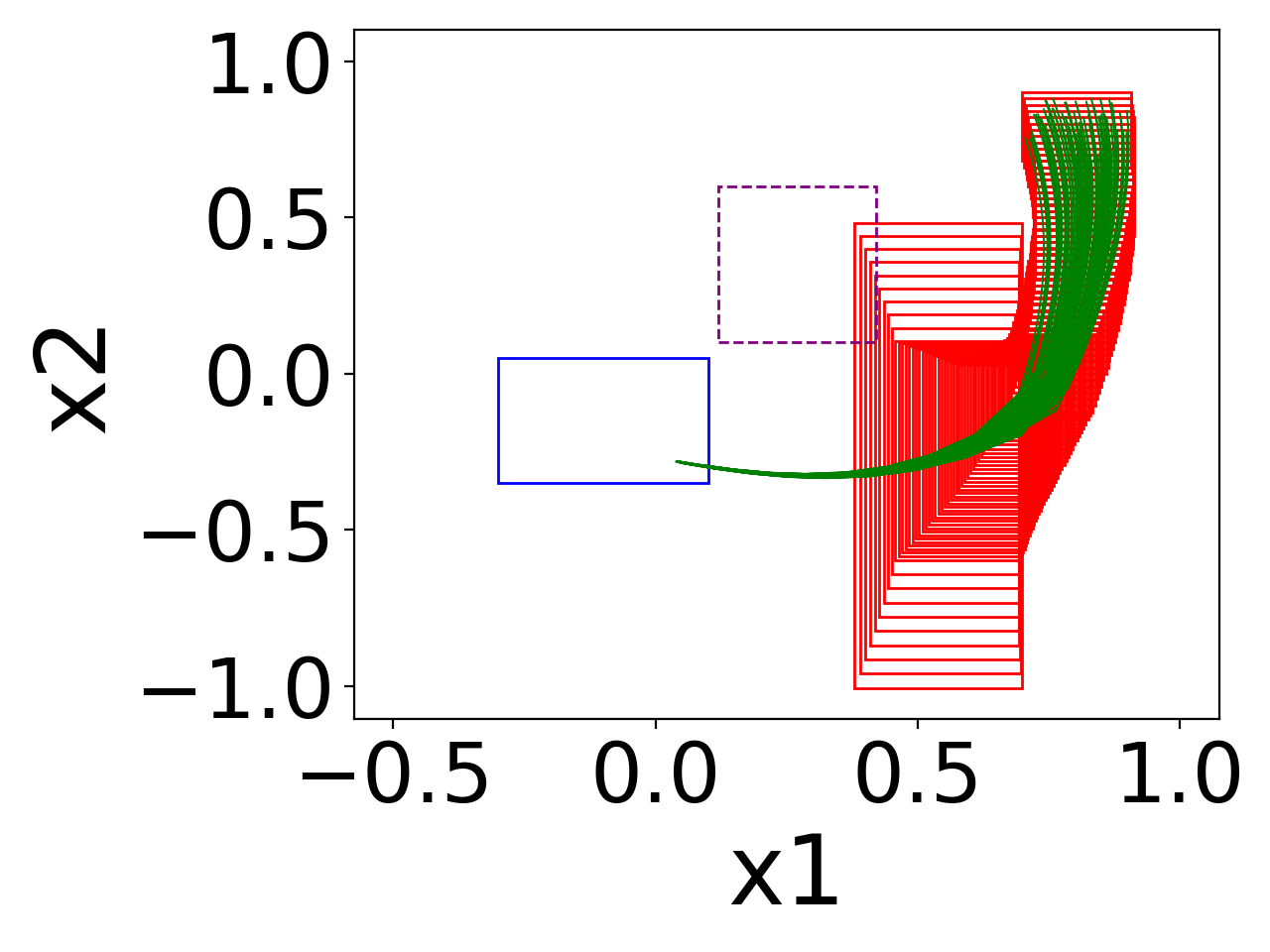}
	\end{minipage}
\\
\vspace{1ex}
	\begin{minipage}[b]{0.07\linewidth}
\begin{tikzpicture}
\tikz \node [draw] at (0.5,-1)  {Tora};		
\end{tikzpicture}
	\end{minipage}
	\begin{minipage}[b]{0.29\linewidth}
		\centering
		\includegraphics[scale=0.20]{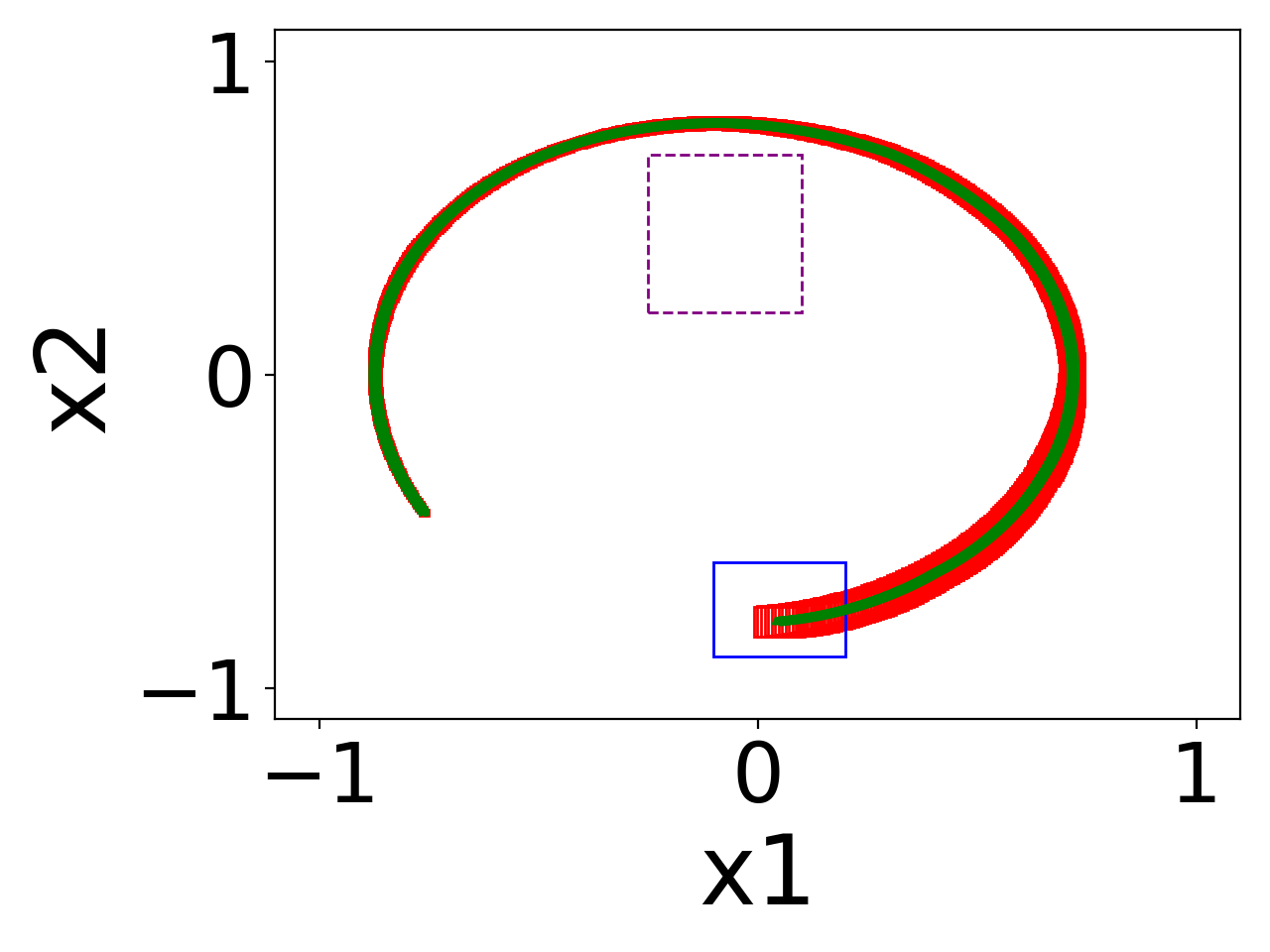}
	\end{minipage}
		\begin{minipage}[b]{0.29\linewidth}
		\centering
		\includegraphics[scale=0.20]{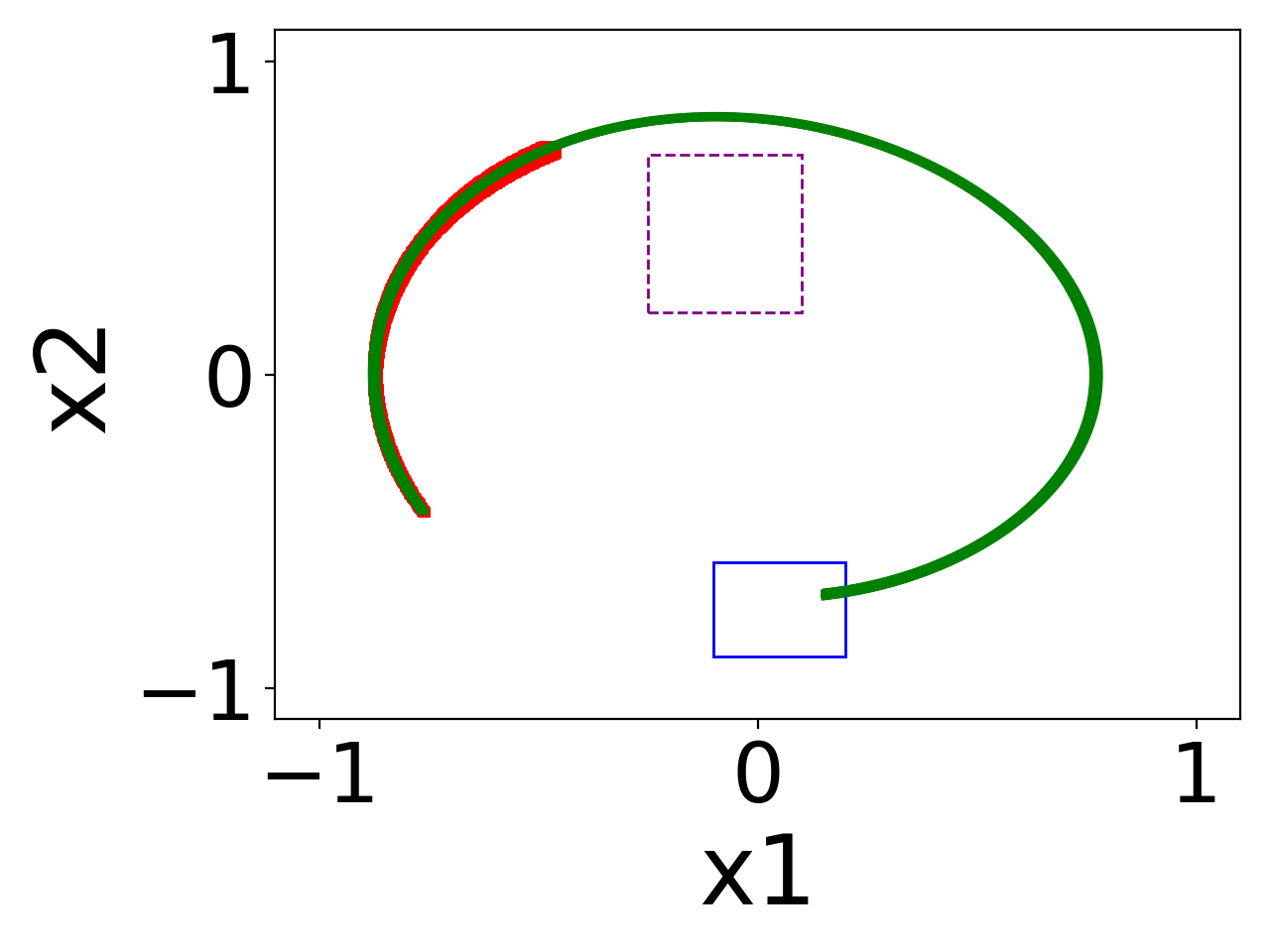}
	\end{minipage}
	\begin{minipage}[b]{0.29\linewidth}
		\centering
	\includegraphics[scale=0.20]{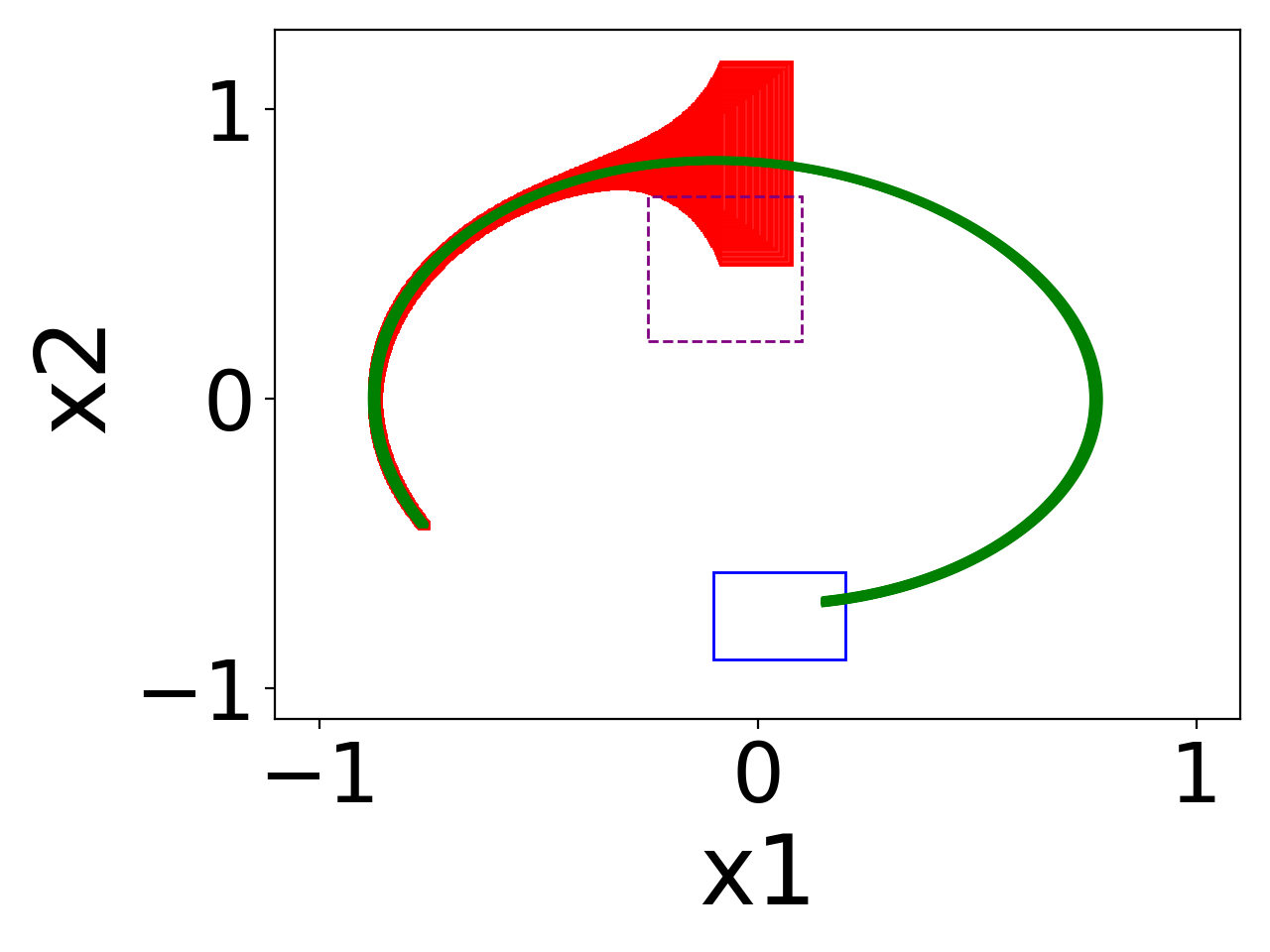}
	\end{minipage}
	\vspace{-3ex}
	\caption{Over-approximated  reachable states (red box: over-approximated set; green lines: simulation trajectories; blue box: goal    region; 
		purple box: unsafe region).}
	\label{fig:reachable_sets}
	\vspace{-5mm}
\end{figure}

Figure~\ref{fig:reachable_sets} shows three representative cases. Verification succeeds if the system never enters the unsafe region (purple box) before reaching the goal region (blue box) which is also known as satisfying the reach-avoid property. 
All the four tools successfully verify the reach-avoid property in case B1,  yet Verisig 2.0 is less tight than the other two. In case B2, \BBReach~outperforms the other two tools and succeeds in verifying the reach-avoid property. Both Verisig 2.0 and Polar terminate before reaching the goal region due to too large overestimation, and Polar outputs the over-approximation sets that intersect with the unsafe region. Nevertheless, the simulation results show the trained DNN-controlled system should satisfy the reach-avoid requirements. 
For Tora, \BBReach\ significantly surpasses other tools. None of the two white-box tools finishes before reaching the goal region because of the huge over-approximation error. For instance, the resulting bound of action, upon Verisig 2.0's termination, 
reaches $10^7$ which is too large to proceed, although the increase of reachable states by simulation is approximately in linear. 
The comparison results for B3, B4, B5, and ACC are similar as for B1. 
\iftoggle{conf-ver}{
	We refer to Appendix D of technical report~\cite{ase2023tian-tr} for more detailed results. }{
	We refer to Appendix~\ref{subsec:tight_com} for more detailed results.  
}

\begin{table}[t]
	\centering
	\footnotesize    		
	\caption{The verification results of reach-avoid properties and the time cost (s).}
	\label{tab:time_comparison}
	\renewcommand{\arraystretch}{0.93}
	\resizebox{\textwidth}{!}{
		\begin{tabular}{|c|c|r|r|r|c|r|r|r|r|c|r|r|r|c|}
			\hline
			\multirow{2}{*}{\textbf{Task}} & \multirow{2}{*}{\textbf{Dim}} & \multicolumn{1}{c|}{\multirow{2}{*}{\textbf{Network}}} &  \multicolumn{3}{c|}{\textbf{\BBReach}}   &                                                \multicolumn{5}{c|}{\textbf{Verisig 2.0}}                                                 &                                      \multicolumn{4}{c|}{\textbf{Polar}}                                      \\ \cline{4-15}
			~                &               ~               &                                                      ~ & \textbf{1C} & \textbf{20Cs} & \textbf{VR} & \textbf{1C} & \multicolumn{1}{c|}{\textbf{Impr.}} & \textbf{20Cs} & \multicolumn{1}{c|}{\textbf{Impr.}} &          \textbf{VR}           &   \textbf{1C} & \multicolumn{1}{c|}{\textbf{Impr.}} & \multicolumn{1}{c|}{\textbf{Impr.$^*$}} &  \textbf{VR}  \\ \hline\hline
			\multirow{4}{*}{B1}       &      \multirow{4}{*}{2}       &                                   Tanh$_{2 \times 20}$ &        45.7 &          6.88 & \checkmark  &          45 &      \color{Mahogany}{0.98$\times$} &            38 &    \color{OliveGreen}{5.52$\times$} &           \checkmark           &            17 &      \color{Mahogany}{0.37$\times$} &        \color{OliveGreen}{2.47$\times$} &  \checkmark   \\ 
			~                &               ~               &                                  Tanh$_{3 \times 100}$ &        42.8 &          5.53 & \checkmark  &         413 &    \color{OliveGreen}{9.65$\times$} &           123 &   \color{OliveGreen}{22.24$\times$} &           \checkmark           &           125 &    \color{OliveGreen}{2.92$\times$} &       \color{OliveGreen}{22.60$\times$} &  \checkmark   \\ \cline{3-15}
			~                &               ~               &                                   ReLU$_{2 \times 20}$ &        42.9 &          6.44 & \checkmark  &         --- &                                 --- &           --- &                                 --- & \multirow{2}{*}{\ding{55}$^c$} &             3 &      \color{Mahogany}{0.07$\times$} &          \color{Mahogany}{0.47$\times$} &  \checkmark   \\ 
			~                &               ~               &                                  ReLU$_{3 \times 100}$ &        52.5 &          8.65 & \checkmark  &         --- &                                 --- &           --- &                                 --- &                                &           --- &                                 --- &                                     --- & \ding{55}$^b$ \\ \hline
			\multirow{4}{*}{B2}       &      \multirow{4}{*}{2}       &                                   Tanh$_{2 \times 20}$ &        10.0 &          1.19 & \checkmark  &         5.2 &      \color{Mahogany}{0.52$\times$} &           4.1 &    \color{OliveGreen}{3.45$\times$} &         \ding{55}$^a$          &             5 &      \color{Mahogany}{0.50$\times$} &        \color{OliveGreen}{4.20$\times$} &  \checkmark   \\ 
			~                &               ~               &                                  Tanh$_{3 \times 100}$ &        10.8 &          1.36 & \checkmark  &         --- &                                 --- &           --- &                                 --- &         \ding{55}$^b$          &           --- &                                 --- &                                     --- & \ding{55}$^b$ \\ \cline{3-15}
			~                &               ~               &                                   ReLU$_{2 \times 20}$ &         8.6 &          1.30 & \checkmark  &         --- &                                 --- &           --- &                                 --- & \multirow{2}{*}{\ding{55}$^c$} &             3 &      \color{Mahogany}{0.35$\times$} &        \color{OliveGreen}{2.31$\times$} &  \checkmark   \\ 
			~                &               ~               &                                  ReLU$_{3 \times 100}$ &        12.4 &          1.42 & \checkmark  &         --- &                                 --- &           --- &                                 --- &                                &           --- &                                 --- &                                     --- & \ding{55}$^b$ \\ \hline
			\multirow{4}{*}{B3}       &      \multirow{4}{*}{2}       &                                   Tanh$_{2 \times 20}$ &         4.2 &          0.47 & \checkmark  &          36 &    \color{OliveGreen}{8.57$\times$} &            28 &   \color{OliveGreen}{59.57$\times$} &           \checkmark           &            18 &    \color{OliveGreen}{4.29$\times$} &       \color{OliveGreen}{39.29$\times$} &  \checkmark   \\ 
			~                &               ~               &                                  Tanh$_{3 \times 100}$ &         4.3 &          0.50 & \checkmark  &         357 &   \color{OliveGreen}{83.02$\times$} &            88 &  \color{OliveGreen}{176.00$\times$} &           \checkmark           &            91 &   \color{OliveGreen}{91.16$\times$} &      \color{OliveGreen}{182.00$\times$} &  \checkmark   \\ \cline{3-15}
			~                &               ~               &                                   ReLU$_{2 \times 20}$ &         4.1 &          0.47 & \checkmark  &         --- &                                 --- &           --- &                                 --- & \multirow{2}{*}{\ding{55}$^c$} &             8 &    \color{OliveGreen}{1.95$\times$} &       \color{OliveGreen}{17.02$\times$} &  \checkmark   \\ 
			~                &               ~               &                                  ReLU$_{3 \times 100}$ &         4.2 &          0.47 & \checkmark  &         --- &                                 --- &           --- &                                 --- &                                &            14 &    \color{OliveGreen}{3.33$\times$} &       \color{OliveGreen}{29.79$\times$} &  \checkmark   \\ \hline
			\multirow{4}{*}{B4}       &      \multirow{4}{*}{3}       &                                   Tanh$_{2 \times 20}$ &         1.3 &          0.32 & \checkmark  &           7 &    \color{OliveGreen}{5.38$\times$} &           5.1 &   \color{OliveGreen}{15.94$\times$} &           \checkmark           &             5 &    \color{OliveGreen}{3.85$\times$} &       \color{OliveGreen}{15.63$\times$} &  \checkmark   \\ 
			~                &               ~               &                                  Tanh$_{3 \times 100}$ &         1.0 &          0.24 & \checkmark  &         114 &  \color{OliveGreen}{114.00$\times$} &            31 &  \color{OliveGreen}{129.17$\times$} &           \checkmark           &            27 &   \color{OliveGreen}{27.00$\times$} &      \color{OliveGreen}{112.50$\times$} &  \checkmark   \\ \cline{3-14}
			~                &               ~               &                                   ReLU$_{2 \times 20}$ &         1.9 &          0.48 & \checkmark  &         --- &                                 --- &           --- &                                 --- & \multirow{2}{*}{\ding{55}$^c$} &             2 &    \color{OliveGreen}{1.05$\times$} &        \color{OliveGreen}{4.17$\times$} &  \checkmark   \\ 
			~                &               ~               &                                  ReLU$_{3 \times 100}$ &         1.8 &          0.43 & \checkmark  &         --- &                                 --- &           --- &                                 --- &                                &             5 &    \color{OliveGreen}{2.78$\times$} &       \color{OliveGreen}{11.63$\times$} &  \checkmark   \\ \hline
			\multirow{4}{*}{B5}       &      \multirow{4}{*}{3}       &                                  Tanh$_{3 \times 100}$ &        13.3 &          2.48 & \checkmark  &         157 &   \color{OliveGreen}{11.80$\times$} &            44 &   \color{OliveGreen}{17.74$\times$} &           \checkmark           &            38 &    \color{OliveGreen}{2.86$\times$} &       \color{OliveGreen}{15.32$\times$} &  \checkmark   \\ 
			~                &               ~               &                                  Tanh$_{4 \times 200}$ &         8.2 &          1.63 & \checkmark  &        1443 &  \color{OliveGreen}{175.98$\times$} &           191 &  \color{OliveGreen}{117.18$\times$} &           \checkmark           &           157 &   \color{OliveGreen}{19.15$\times$} &       \color{OliveGreen}{96.32$\times$} &  \checkmark   \\ \cline{3-15}
			~                &               ~               &                                  ReLU$_{3 \times 100}$ &         5.8 &          1.08 & \checkmark  &         --- &                                 --- &           --- &                                 --- & \multirow{2}{*}{\ding{55}$^c$} &             7 &    \color{OliveGreen}{1.21$\times$} &        \color{OliveGreen}{6.48$\times$} &  \checkmark   \\ 
			~                &               ~               &                                  ReLU$_{4 \times 200}$ &        13.5 &          2.50 & \checkmark  &         --- &                                 --- &           --- &                                 --- &                                &            49 &    \color{OliveGreen}{3.63$\times$} &       \color{OliveGreen}{19.60$\times$} &  \checkmark   \\ \hline
			\multirow{4}{*}{Tora}      &      \multirow{4}{*}{4}       &                                   Tanh$_{3 \times 20}$ &       133.2 &          8.61 & \checkmark  &          69 &      \color{Mahogany}{0.52$\times$} &            46 &    \color{OliveGreen}{5.34$\times$} &           \checkmark           &            45 &      \color{Mahogany}{0.34$\times$} &        \color{OliveGreen}{5.23$\times$} &  \checkmark   \\ 
			~                &               ~               &                                  Tanh$_{4 \times 100}$ &       112.3 &          9.78 & \checkmark  &         --- &                                 --- &           --- &                                 --- &              ---               & \ding{55}$^b$ &                                 --- &                                     --- & \ding{55}$^b$ \\ \cline{3-15}
			~                &               ~               &                                   ReLU$_{3 \times 20}$ &       124.7 &          9.97 & \checkmark  &         --- &                                 --- &           --- &                                 --- & \multirow{2}{*}{\ding{55}$^c$} &            30 &      \color{Mahogany}{0.24$\times$} &        \color{OliveGreen}{3.01$\times$} &  \checkmark   \\ 
			~                &               ~               &                                  ReLU$_{4 \times 100}$ &       128.1 &          7.54 & \checkmark  &         --- &                                 --- &           --- &                                 --- &                                &            53 &      \color{Mahogany}{0.41$\times$} &        \color{OliveGreen}{7.03$\times$} &  \checkmark   \\ \hline
			\multirow{4}{*}{ACC}      &      \multirow{4}{*}{6}       &                                   Tanh$_{3 \times 20}$ &        15.4 &          4.53 & \checkmark  &         113 &    \color{OliveGreen}{7.34$\times$} &            50 &   \color{OliveGreen}{11.04$\times$} &           \checkmark           &            84 &    \color{OliveGreen}{5.45$\times$} &       \color{OliveGreen}{18.54$\times$} &  \checkmark   \\ 
			~                &               ~               &                                  Tanh$_{4 \times 100}$ &        15.2 &          4.51 & \checkmark  &        2617 &  \color{OliveGreen}{172.17$\times$} &           375 &   \color{OliveGreen}{83.15$\times$} &           \checkmark           &           677 &   \color{OliveGreen}{44.54$\times$} &      \color{OliveGreen}{150.11$\times$} &  \checkmark   \\ \cline{3-15}
			~                &               ~               &                                   ReLU$_{3 \times 20}$ &        15.2 &          4.45 & \checkmark  &         --- &                                 --- &           --- &                                 --- & \multirow{2}{*}{\ding{55}$^c$} &            26 &    \color{OliveGreen}{1.71$\times$} &        \color{OliveGreen}{5.84$\times$} &  \checkmark   \\ 
			~                &               ~               &                                  ReLU$_{4 \times 100}$ &        18.4 &          5.49 & \checkmark  &         --- &                                 --- &           --- &                                 --- &                                &            58 &    \color{OliveGreen}{3.15$\times$} &       \color{OliveGreen}{10.56$\times$} &  \checkmark   \\ \hline
	\end{tabular}}
	\begin{tablenotes}
		\small \item \textbf{Remarks.}
		Improvement: time speedup of \BBReach~compared to Verisig or Polar ($ $Verisig or Polar$/\BBReach$). * denotes the comparison between BBReach with 20 cores (Cs) and Polar. 
		Tanh/ReLU$_{n \times k}$: a neural network with the activation function Tanh/ReLU, $n$ hidden layers, and $k$ neurons per hidden layer.   
		VR:  verification result.
		$\checkmark$: the reach-avoid problem is successfully verified.
		\ding{55}$^{type}$: the reach-avoid problem cannot be verified due to $type$: (a) large over-approximation error, 
		(b) the calculation did not finish, (c) not applicable.  
		---: no data available due to \ding{55}$^{b}$ or \ding{55}$^{c}$.
	\end{tablenotes}
	\vspace{-7mm}
\end{table}

Table \ref{tab:time_comparison} shows the verification results of all the 28 instances in column \textbf{VR}. \BBReach\ successfully verifies all the instances, while Verisig 2.0 succeeds in 11 instances and Polar in 24 instances. Verisig 2.0 reports 1 unknown case (marked by \ding{55}$^{a}$, indicating that  over-approximated sets get outside of the goal region). Additionally, Verisig 2.0 and Polar report 1 case and 4 cases of terminating before reaching the goal region, respectively, denoted by \ding{55}$^{b}$. These also reflect that \BBReach\  is tighter and introduces less overestimation than other tools.  





\color{black}

\looseness=-1
Table \ref{tab:time_comparison} also shows the time cost. Note that Verisig~2.0 is not applicable to ReLU neural networks (marked by  \ding{55}$^c$). 
\BBReach\ costs much less time than Verisig 2.0 (up to 176$\times$ speedup) with parallelization enabled. Even with a single core, \BBReach\ incurs less overhead than Verisig 2.0 in most cases. Compared to Polar, \BBReach~consumes more time in dealing with small-sized neural networks in B1, B2, and Tora because the finer-grained abstraction granularity is chosen in the three cases, which affects the performance (see  Section~\ref{subsec:discussion}). 
Nevertheless, \BBReach~consumes less time in all the remaining cases than Polar.
In addition,  
\BBReach\ outperforms Polar with up to 182$\times$ speedup (the latest release of Polar does not support parallelization), thanks to the parallel acceleration.


The efficiency advantage of \BBReach~becomes more notable with larger networks such as Tanh$_{4\times 200}$, thanks to the black-box feature of our approach. \BBReach~consumes almost the same time even for larger neural networks as for small neural networks (e.g., Tanh$_{2\times 20}$). In contrast, the time cost of both the white-box approach almost always increases significantly with larger neural networks.
Moreover, Polar incurs more overhead to process the neural networks with the Tanh activation function compared to ReLU, while \BBReach~consumes similar times for both activation functions.  
Consequently, it is fair to conclude that \BBReach~is more efficient and scalable to large-sized neural networks with any activation functions. It is also evident that, via a decent design of neural networks, the reachability analysis for DNN-controlled systems is achievable while the planted decision-making neural networks are treated as black-box oracles,  with significant rightness and efficiency outperformance over the white-box approaches.

	\begin{figure}[t]
		\captionsetup[subfigure]{aboveskip=-1pt,belowskip=1pt}
		\begin{center}
			\setlength{\tabcolsep}{0.6mm}{
				\begin{tabular}{cccc}
					\begin{subfigure}[b]{0.24\textwidth}
						\includegraphics[width=\textwidth]{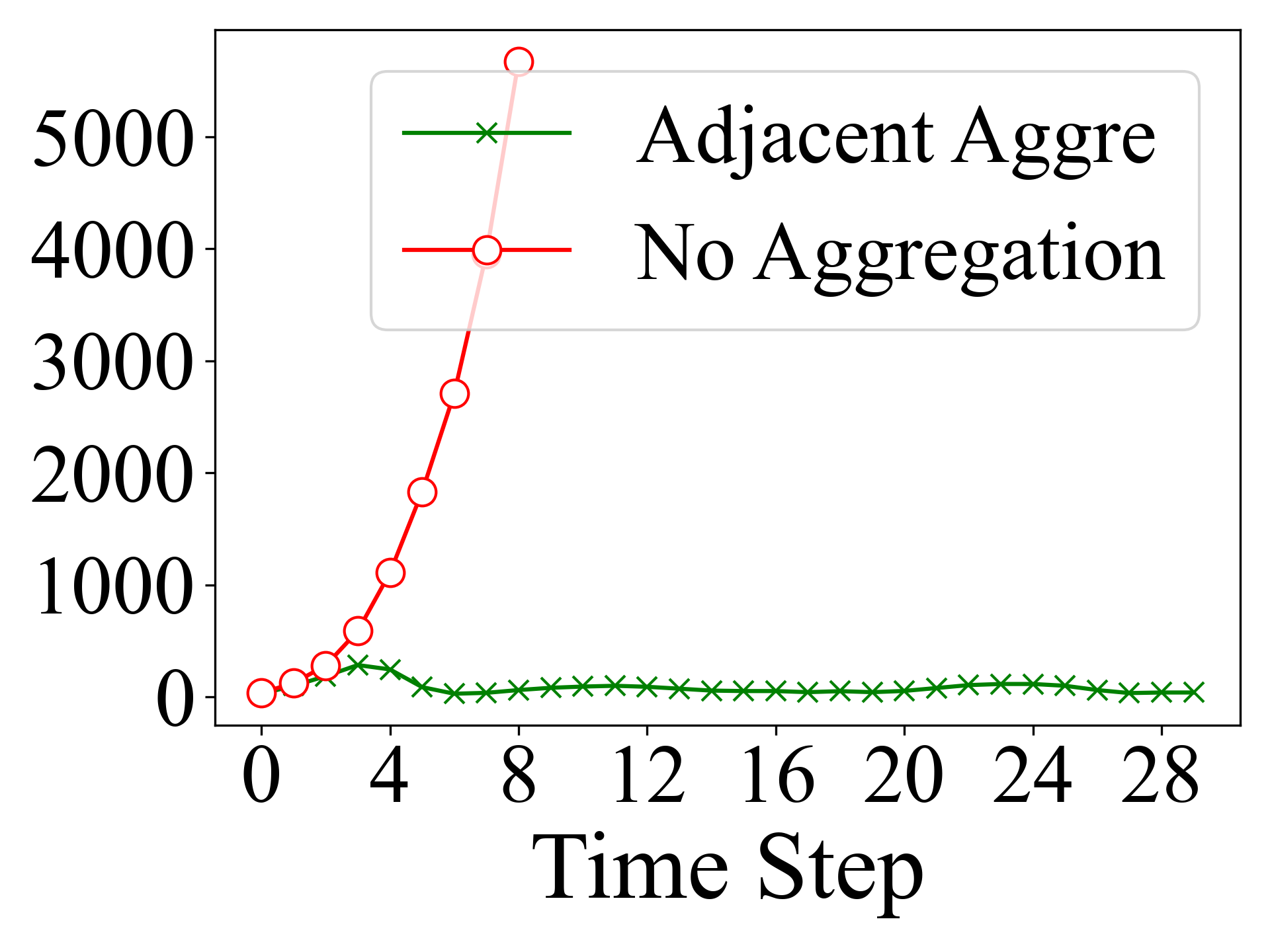}
						\caption{B1}
						\label{fig:b1_agg_com}
					\end{subfigure}&
					\begin{subfigure}[b]{0.24\textwidth}
						\includegraphics[width=\textwidth]{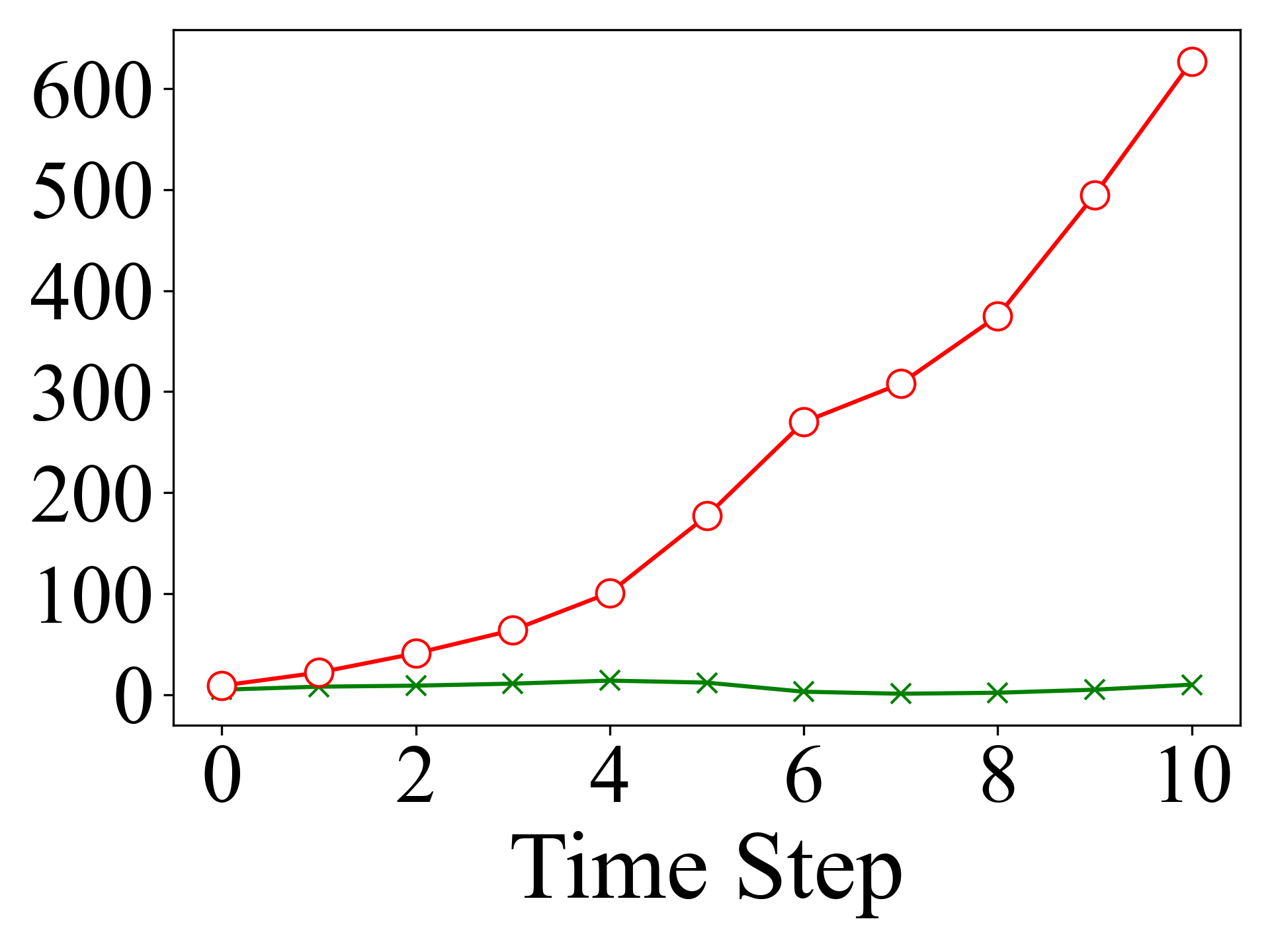}
						\caption{B2}
						\label{fig:b2_agg_com}
					\end{subfigure}&
					\begin{subfigure}[b]{0.24\textwidth}
						\includegraphics[width=\textwidth]{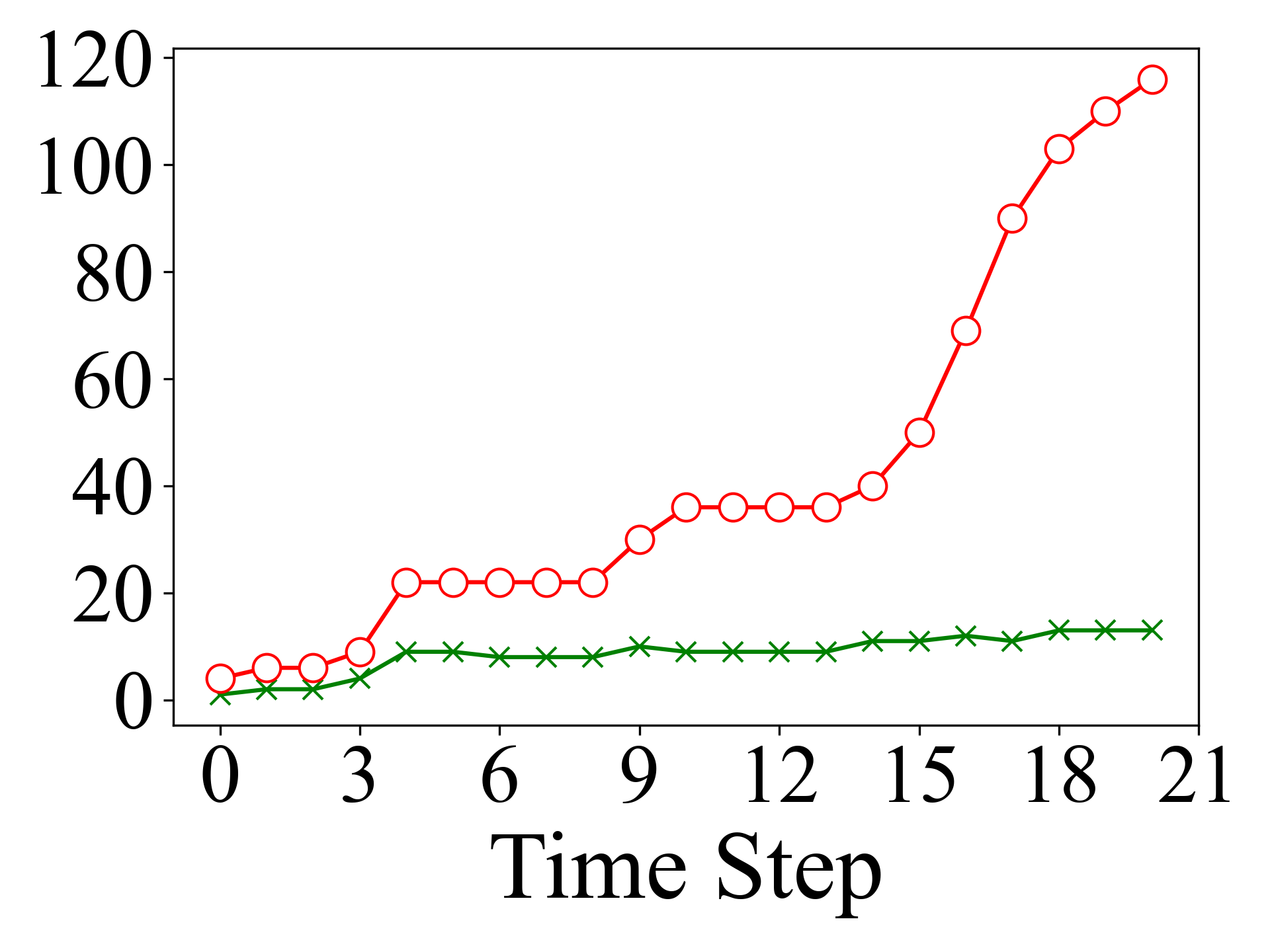}
						\caption{B3}
						\label{fig:b3_agg_com}
					\end{subfigure}
					&
					\begin{subfigure}[b]{0.24\textwidth}
						\includegraphics[width=\textwidth]{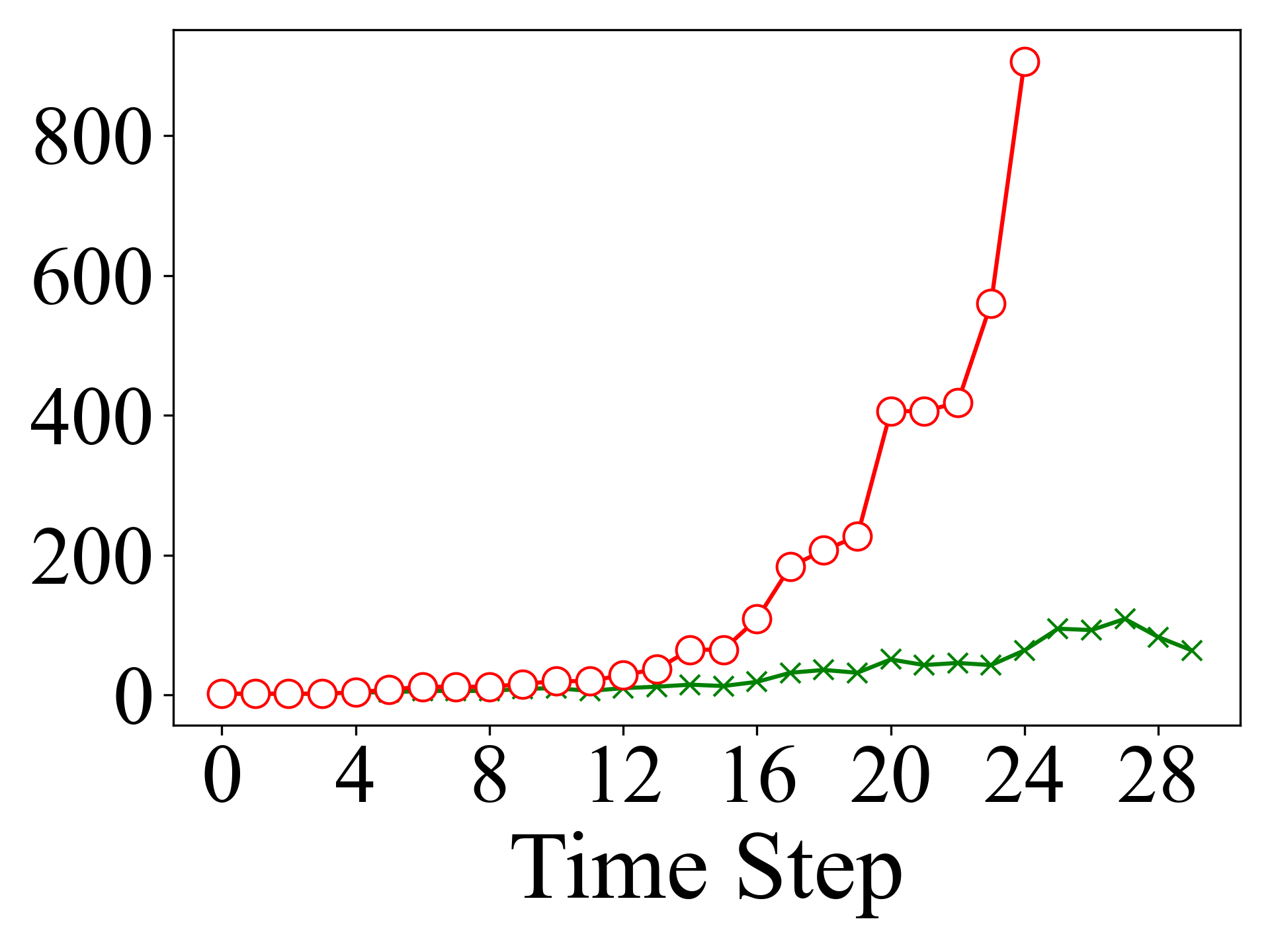}
						\caption{B4}
						\label{fig:b4_agg_com}
					\end{subfigure}
					
					
					\\


							\begin{subfigure}[b]{0.24\textwidth}
								\includegraphics[width=\textwidth]{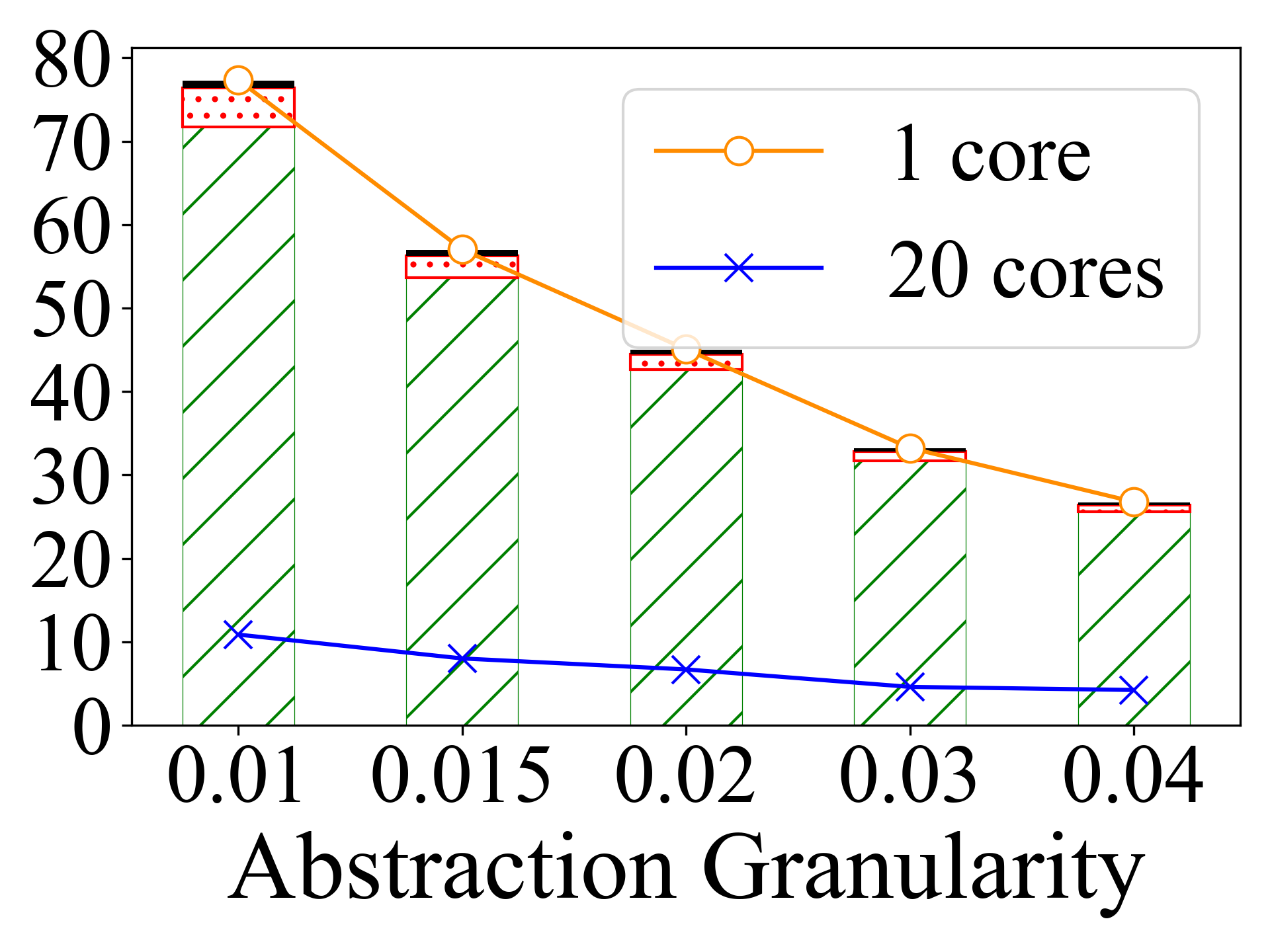}
								\caption{B1 (tanh)}
								\label{fig:b1_dec_tanh}
							\end{subfigure}&
							\begin{subfigure}[b]{0.24\textwidth}
								\includegraphics[width=\textwidth]{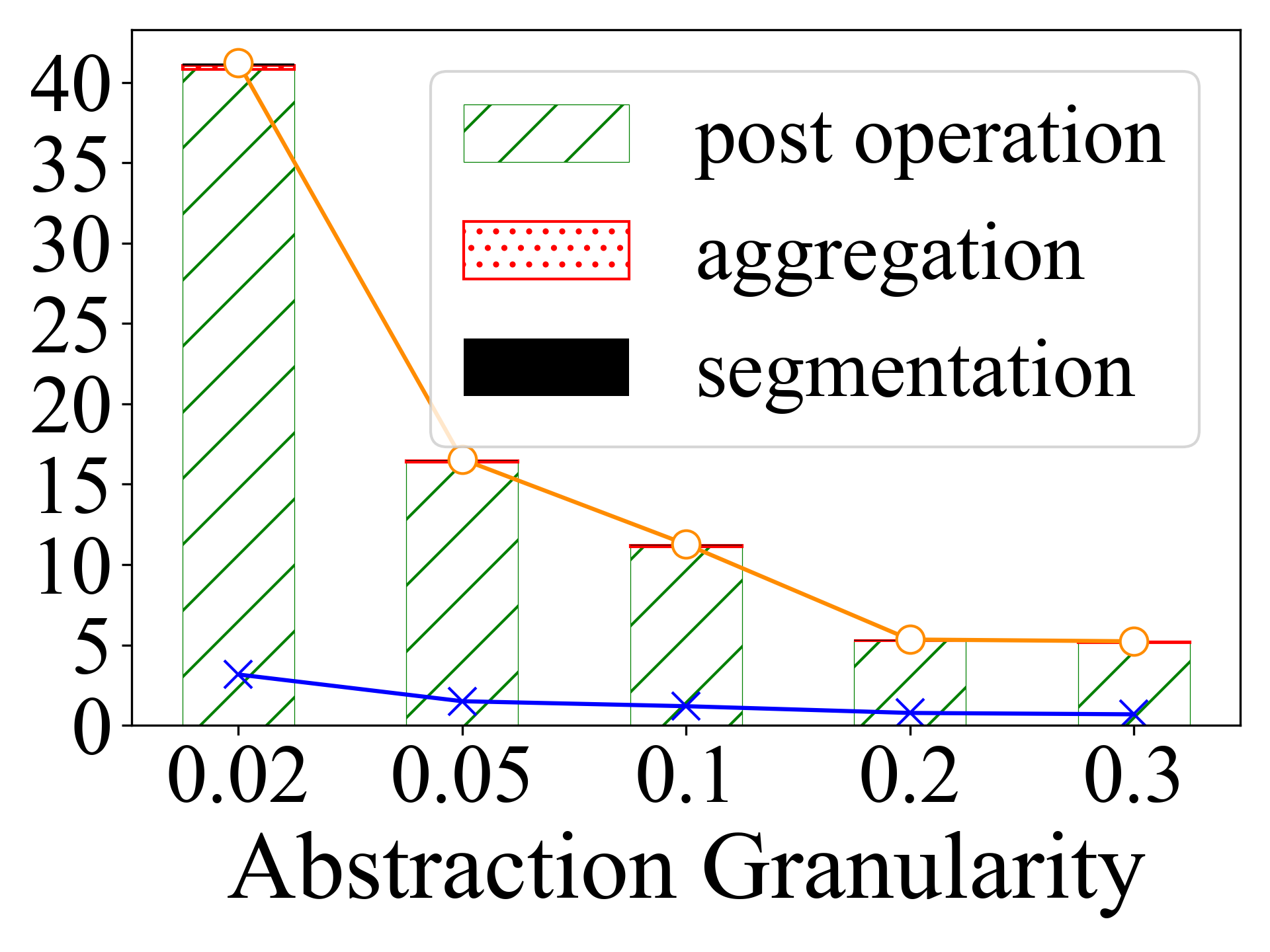}
								\caption{B2 (tanh)}
								\label{fig:b2_dec_tanh}
							\end{subfigure}&
							\begin{subfigure}[b]{0.24\textwidth}
								\includegraphics[width=\textwidth]{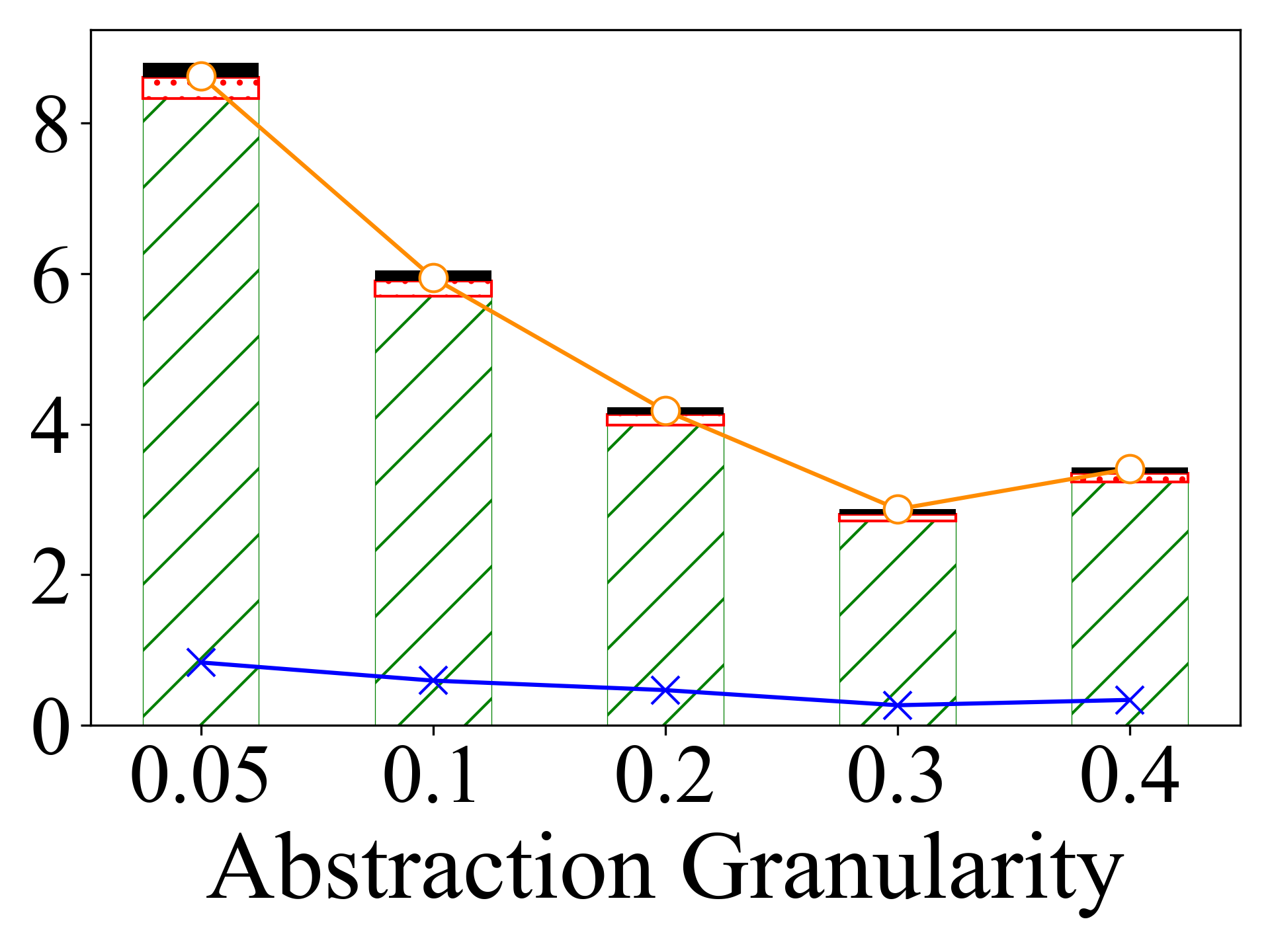}
								\caption{B3 (tanh)}
								\label{fig:b3_dec_tanh}
							\end{subfigure}&
							\begin{subfigure}[b]{0.24\textwidth}
								\includegraphics[width=\textwidth]{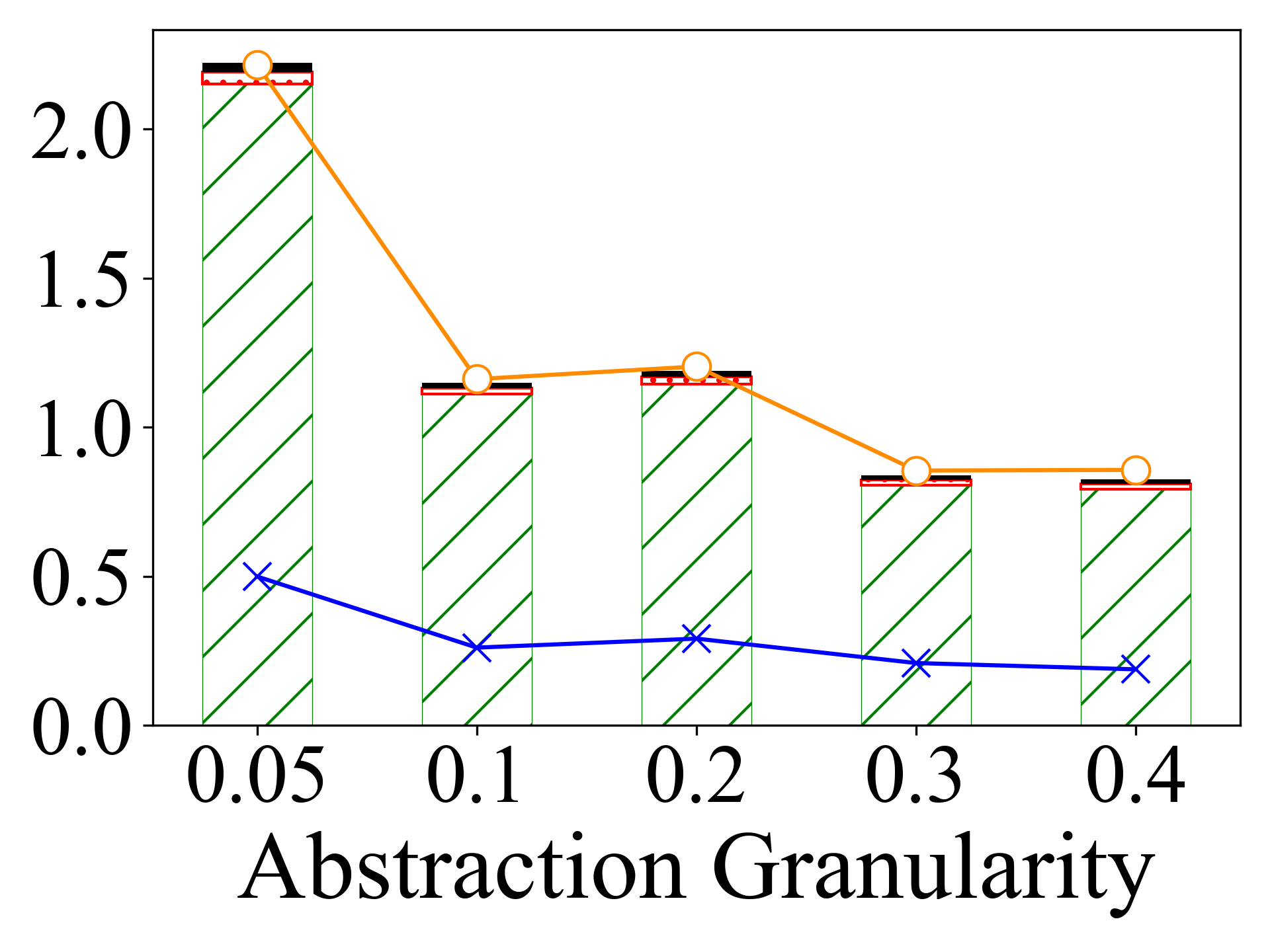}
								\caption{B4 (tanh)}
								\label{fig:b4_dec_tanh}
							\end{subfigure}
					\end{tabular}}
					\vspace{-5mm}
				\end{center}
				\caption{Differential (Row 1) and decomposing (Row 2) analysis results. Y-axis in (a--d) indicates the number of interval boxes while  
					in (e--h) the time overhead in seconds. 
					Due to the space limitation,  we use a scalar value $g_1$ to denote the $n$-dimensional abstraction granularity vector $\gamma = (g_1,...,g_1)$. 
				}
				\label{fig:decompose_analysis}
				\vspace{-6mm}
			\end{figure}
	
	\subsection{Differential and  Decomposing Analysis}
	\label{subsec:discussion}
 
	\noindent \textbf{Differential Analysis.}
	To demonstrate the significance of the adjacent interval aggregation in Algorithm~\ref{alg:interval_agg}, we measure the growth rate of the number of interval boxes with adjacent aggregation, as well as with no aggregation.
	Figure~\ref{fig:decompose_analysis}(a-d) shows the comparison results on B1-B4 (the results for the other six benchmarks are similar and given in Appendix~\ref{subsec:diff_dec_result}). 
	We observe that the number of interval boxes grows rapidly with no aggregation, which implies a dramatically increased verification overhead.  With the adjacent interval aggregation, the number of interval boxes is extremely small and stable. 
	
	\vspace{1ex}
	\noindent \textbf{Decomposing Analysis.}
	We evaluate how different abstraction granularity levels affect the performance of \BBReach~and its components.
	Abstraction granularity is a crucial hyper-parameter used in both training and calculation of over-approximation sets. To better understand the impact of abstraction granularity, for each benchmark, in addition to the default abstraction granularity levels (\iftoggle{conf-ver}{details can be found in Appendix  C}{in Appendix \ref{sec:benchmarks}}), we choose two finer and two coarser levels, respectively, to evaluate the verification efficiency on both Tanh and ReLU neural networks. We also measure the time consumed by each of the three steps, i.e., interval segmentation, post operation, and adjacent interval aggregation. 

 \looseness=-1
	We present in  Figure~\ref{fig:decompose_analysis}(e-h) the results with the Tanh neural network in B1-B4 (the remaining results are similar and given in 
	\iftoggle{conf-ver}{Appendix E of \cite{ase2023tian-tr}} {Appendix \ref{subsec:diff_dec_result}}). With a single core, as the abstraction granularity becomes coarse-grained, the verification time decreases; however, a fairly fine-grained abstraction granularity, e.g., (0.01, 0.01), could result in much higher verification overhead.  
	We also observe that the 
	post operation takes most of the verification time, while the
	overhead of the other two steps is negligible.  Finally, as expected, the parallelization (with 20 cores) can
	significantly accelerate \BBReach.


%% file: concl.tex
\vspace{-1mm}
\section{Related Work} \label{sec:related}
\vspace{-1mm}
Our work is a sequel of recently emerged approaches for the reachability analysis of DNN-controlled systems such as Verisig 2.0~\cite{ivanov2021verisig}, Polar~\cite{huang2022polar}, ReachNN*~\cite{huang2019reachnn}. 
Besides these states of the art, 
NNV~\cite{tran2020nnv} introduces the star set analysis technique~\cite{tran2020verification} to deal with the neural network and combines with the tool called CORA~\cite{althoff2015introduction} for the reachability analysis of non-linear systems. JuliaReach \cite{schilling2022verification} integrates the over-approximation of environment dynamics and DNNs together using Taylor models and zonotope. All these approaches treat DNNs as white boxes by over-approximating them with efficiently computable models such as Taylor models \cite{chen2012taylor}. Due to the intrinsic complexity of DNNs, these white-box approaches are applicable to only a limited type of DNNs on small scales. 



Our abstraction-based training method follows those machine learning methodologies which advocate a similar idea of pre-processing training data using either abstraction \cite{abel2019theory,jin2022cegar}, fuzzing \cite{de2020fuzzy} or granulation \cite{song2019granular} for various purposes of reducing the size of models, capturing uncertainties in input data and extracting abstract knowledge. 
Recent studies show that,
rather than training on concrete datasets,
training on symbolic datasets
is helpful to build  verification-friendly neural networks \cite{drews2020proving} and network-controlled systems \cite{jin2022cegar}. 
 The approach in  \cite{abel2019theory} is focused on the training of finite-state systems, while the one in \cite{jin2022cegar} needs to extend existing training methods to admit abstract states. Our design of abstract neural networks is more decent than the approach in \cite{jin2022cegar} because we only need to insert abstraction layers into neural networks and do not impose any other changes to training algorithms.

There are several black-box but unsound verification approaches for DNN-controlled systems. For instance, Fan et al. proposed a hybrid approach of combining black-box simulation and white-box transition graph for a probabilistic verification result \cite{fan2017dryvr}. Xue et al. proposed a black-box model-checking approach for continuous-time dynamical systems based on Probably Approximate Correctness (PAC)  learning \cite{xue2020pac}. 
Dong et al. built a discrete-time Markov chain from extracted trajectories of a DRL system and verified safety properties by probabilistic model checking \cite{dong2022dependability}. 
However, these approaches are not sound and can only compute error probability and confidence with probably approximate correctness guarantees. 
The fundamental reason for the unsoundness is that only partial behaviors of systems can be modeled when conventional neural networks are treated as black-box oracles, i.e.,  fixing concrete system states and feeding them into the networks to determine the state transitions. 

\vspace{-2mm}
\section{Conclusion and Future Work}
\vspace{-1mm}
We have presented an efficient and tight approach for the reachability analysis of DNN-controlled systems by bypassing the time-consuming and imprecise over-approximation of the DNNs in systems via abstraction-based training. Our method demonstrates the possibility of achieving sound but black-box reachability analysis through a decent abstraction-based training approach, breaking conventional intuitions that black-box methods only offer approximate correctness guarantees~
\cite{fan2017dryvr,xue2020pac} and that over-approximating DNNs is inevitable for sound verification~\cite{ivanov2021verisig,huang2019reachnn,huang2022polar}. Compared to white-box approaches, our black-box approach offers several benefits, including significant efficiency improvements, improved tightness of computed overestimation sets, applicability and scalability to a wider range of extended abstract DNNs, regardless of their architectures, activation functions, and neuron size.




Our work sheds light on a promising direction for studying efficient and sound formal verification approaches for DNN-controlled systems by treating black-box-featured DNNs as black boxes. 
We believe that this first black-box reachability analysis approach for DNN-controlled systems would stimulate more future work, such as new abstraction methods, runtime verification and model-checking of more complex safety and liveness properties. In our another work, we adopt the black-box training method to synthesize verification-friendly piece-wise linear controllers for ease the verification of trained systems \cite{jiaxu2023boosting}. 
Thanks to the representation power of DNNs \cite{tuckerman2022curse}, our training approach can support training DNNs for high-dimensional systems because the complexity of transforming an $n$-dimensional state into an interval box is $O(n)$. Nevertheless, the reachability analysis part may suffer from state explosion in the worst case when the number of reachable states increases exponentially, as faced by all the related white-box approaches ~\cite{ivanov2021verisig,huang2019reachnn,huang2022polar}. One possible solution is to coarsen the abstraction to reduce the size of abstract states, and learn an easy-to-verify linear policy for each coarsened abstract state. Such an approach has been successfully applied to reinforcement learning \cite{akrour2018regularizing} and requires further investigation in the DNN-based setting. 

\section*{Acknowledgment}
 The work has been supported by the National Key Project (2020AAA0107800), NSFC Programs (62161146001, 62372176), Huawei Technologies Co., Ltd., the Shanghai International Joint Lab (22510750100), the Shanghai Trusted Industry Internet Software Collaborative Innovation Center, the Engineering and Physical Sciences Research Council (EP/T006579/1), the National Research Foundation (NRF-RSS2022-009), Singapore, and the Shanghai Jiao Tong University Postdoc Scholarship. 

%% file: appendix.tex
\section{Assessing Verisig 2.0 with Big Weights}
\label{subsec:big_weights}
Verisig 2.0 produces large over-approximation error when dealing with neural networks with big weights.
To demonstrate this,
we initialize the weights of the neural network with larger values (random numbers $w_l \sim \mathbf{N}(\mu, \sigma^{2})$ with $\mu =0, \sigma = 0.1$) and show the experimental results in Figure~\ref{fig:big_weights_reachable_sets}. 
We observe that the calculated over-approximation sets contain large over-approximation error except for B4. In Tora, Verisig 2.0 fails to calculate the complete reachable sets due to too large over-approximation error. Hence, it is fairly to say that Verisig 2.0 is sensitive to the DNNs with big weights. 

\begin{figure}[h]
\begin{center}
		\begin{tabular}{cccc}
		\begin{subfigure}[b]{0.32\textwidth}
			\includegraphics[width=\textwidth]{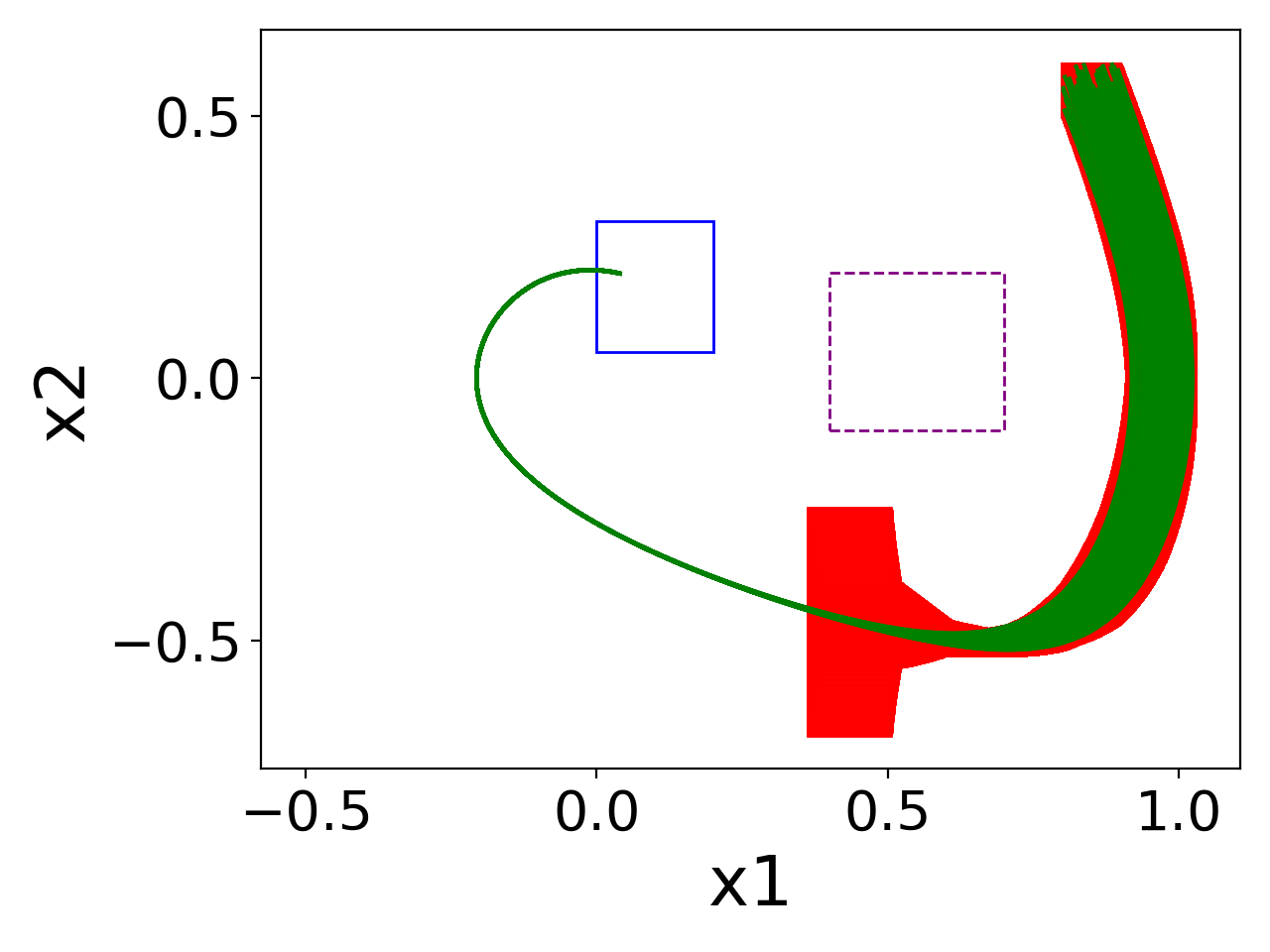}
			
			\caption{B1}
			\label{fig:b1}
		\end{subfigure}&
		\begin{subfigure}[b]{0.32\textwidth}
			\includegraphics[width=\textwidth]{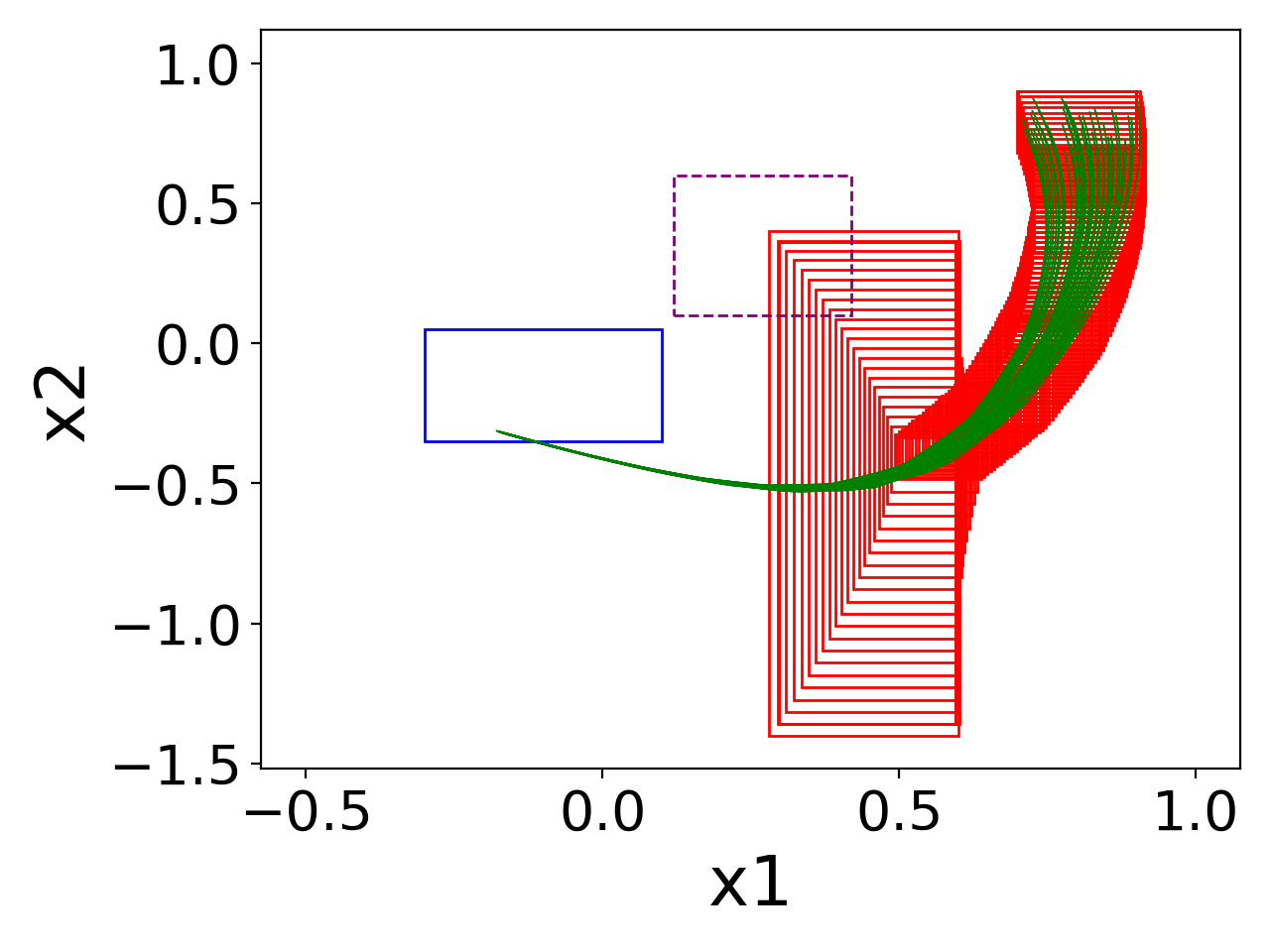}
			
			\caption{B2}
			\label{fig:b2}
		\end{subfigure}
  	&
		\begin{subfigure}[b]{0.32\textwidth}
			\includegraphics[width=\textwidth]{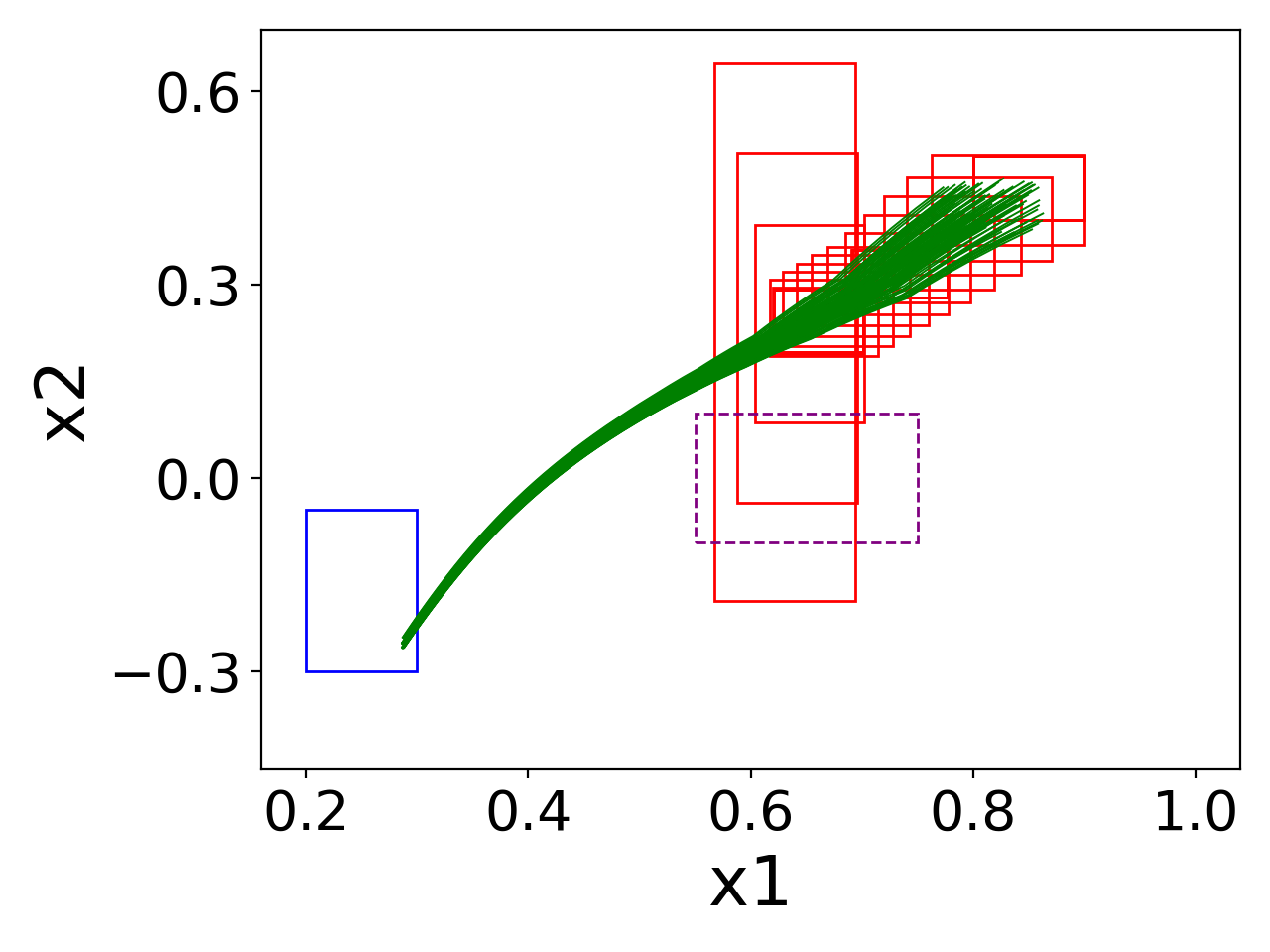}
			
			\caption{B3}
			\label{fig:b3}
		\end{subfigure}
		\\
		\begin{subfigure}[b]{0.32\textwidth}
			\includegraphics[width=\textwidth]{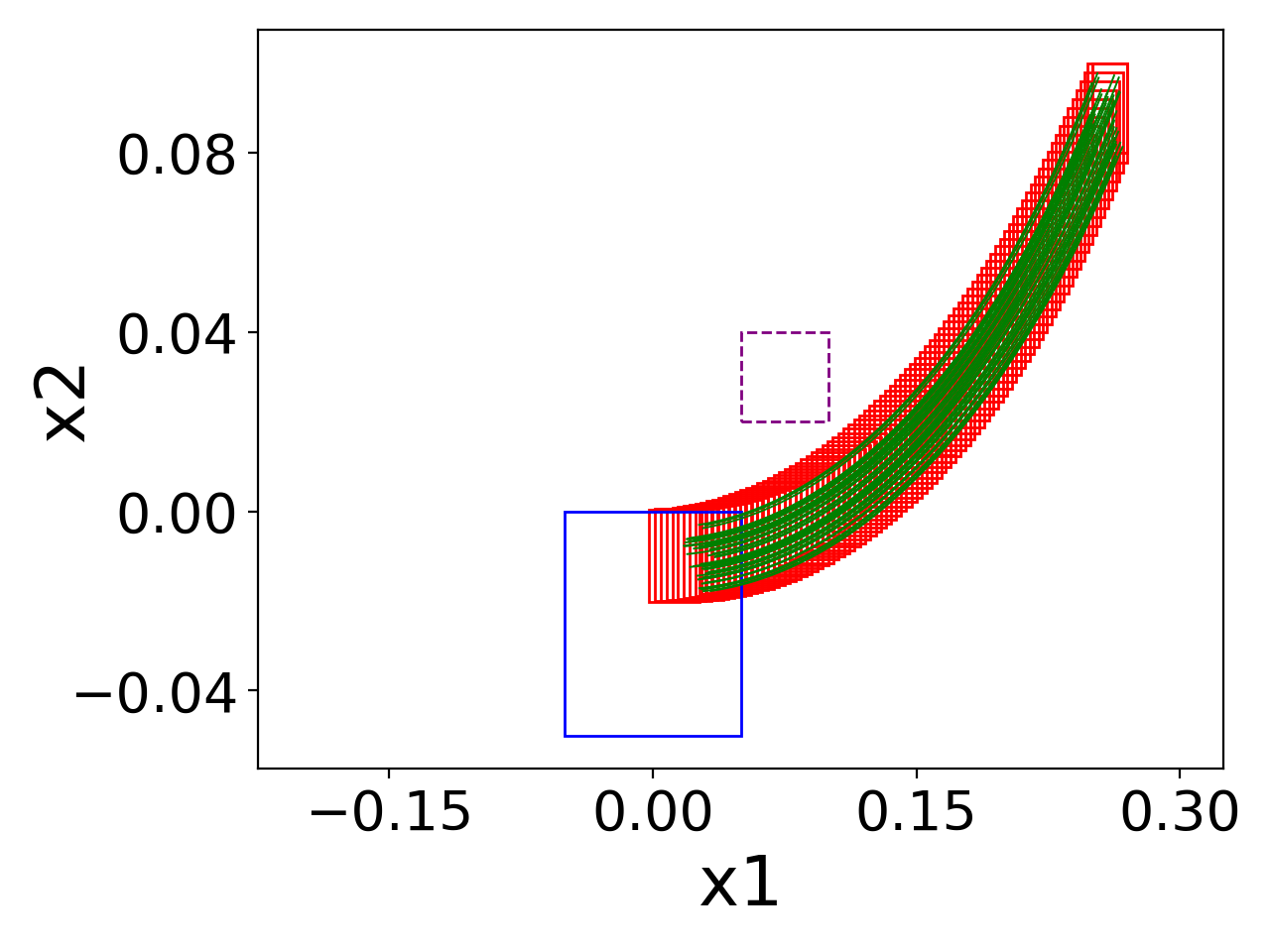}
			\caption{B4}
			\label{fig:b4}
		\end{subfigure} 
		&
		\begin{subfigure}[b]{0.32\textwidth}
			\includegraphics[width=\textwidth]{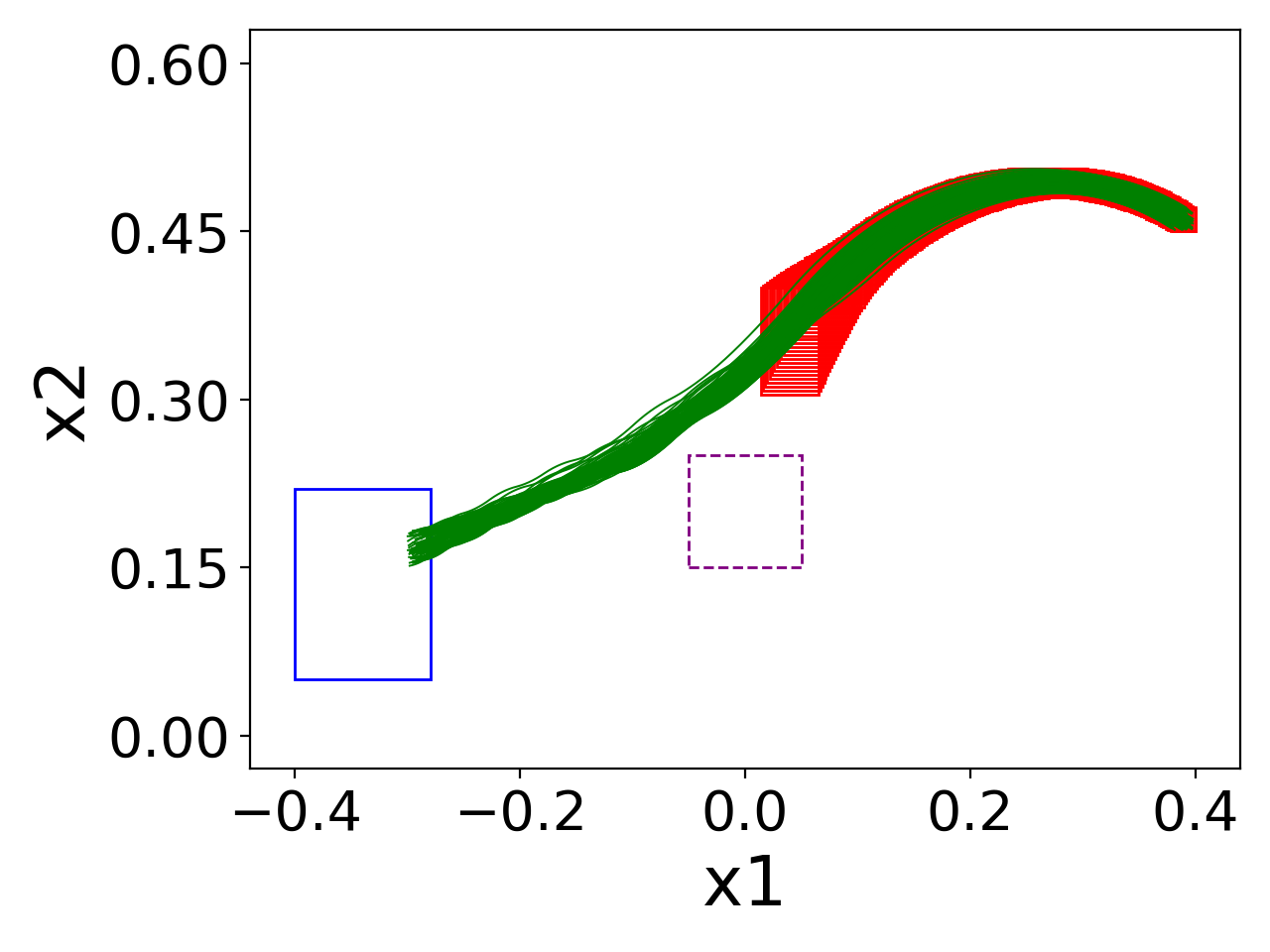}
			\caption{B5}
			\label{fig:b5}
		\end{subfigure} &
		\begin{subfigure}[b]{0.32\textwidth}
			\includegraphics[width=\textwidth]{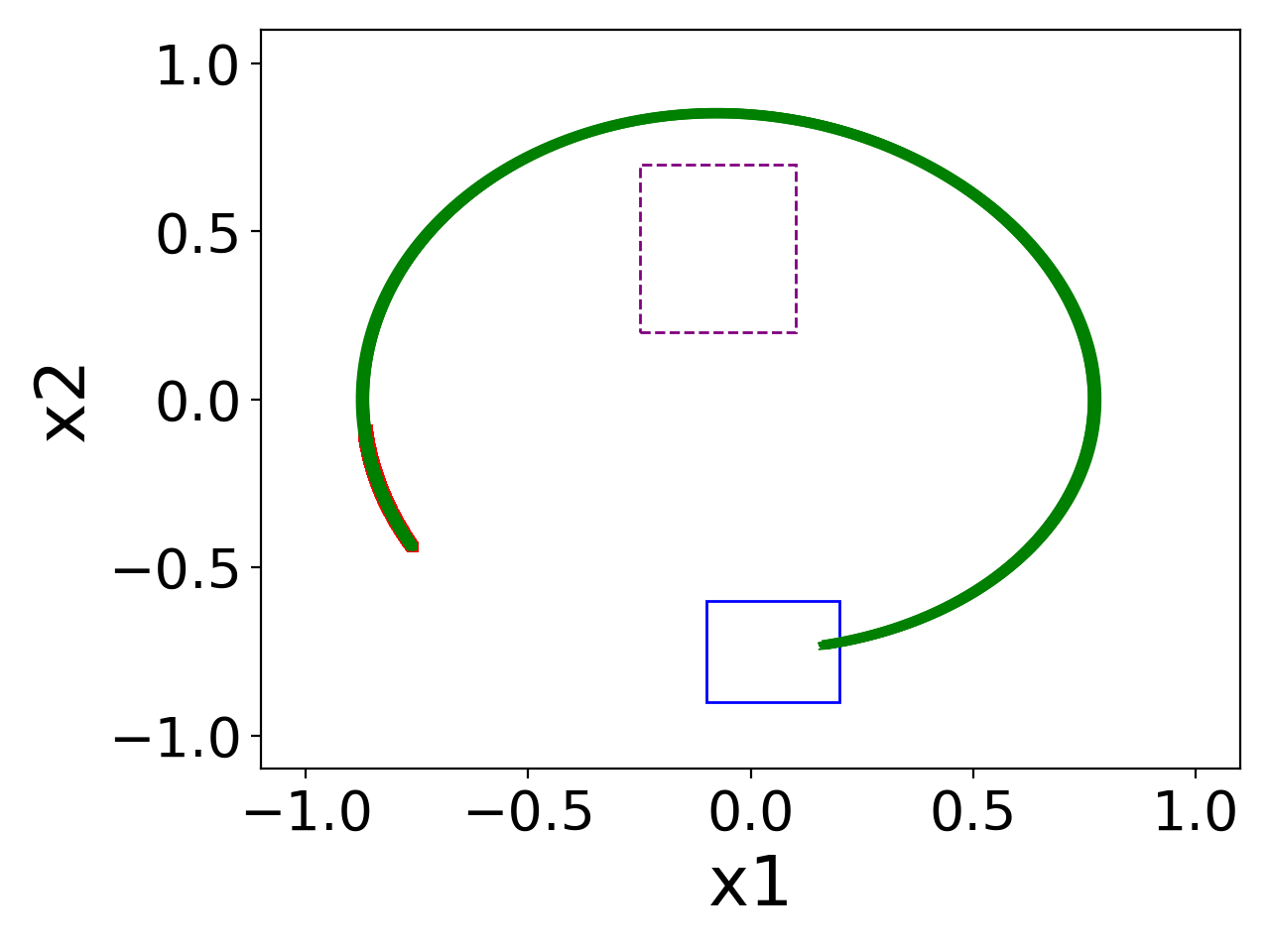}
			\caption{Tora}
			\label{fig:tora}
		\end{subfigure}
	\end{tabular}
\end{center}
\vspace{-3ex}
\caption{Assessing Verisig 2.0 on the larger networks with big weights. 
red box: over-approximation set; green lines: simulation trajectories; blue box: goal region; 
purple dashed box: unsafe region.}
\label{fig:big_weights_reachable_sets}
\end{figure}

 \clearpage 
\section{Benchmarks Setting}
\label{sec:benchmarks}
We provide the setting of seven benchmarks in Table~\ref{tab:benchmarks_setting}. The initial region and goal region are the same as the setting in~\cite{ivanov2021verisig}. The abstraction granularity is a hyper-parameter used in the abstraction-based training and the calculation of reachable sets. All the experimental results in Section~\ref{subsec:tightness}   are based on the following settings.
\vspace{-7ex}
\begin{table}[h!]
\scriptsize
    \centering
        \caption{Benchmarks setting.}
        \vspace{2mm}
    \label{tab:benchmarks_setting}
    \renewcommand{\arraystretch}{1}
    \setlength{\tabcolsep}{6pt}
    \begin{tabular}{c|l|l|l|l}
    \toprule
         \textbf{Task}&\centering \textbf{ Granularity}&\centering \hspace{5mm}\textbf{Initial Region}&\hspace{5mm}\centering \textbf{Goal Region}&\hspace{3mm}\textbf{Unsafe Region} \\
         \midrule
         B1 & [0.02, 0.02]&\makecell[l]{$x_1 \in [0.8, 0.9]$\\$ x_2 \in [0.5,0.6]$}& \makecell[l]{$x_1 \in [0, 0.2]$\\$ x_2 \in [0.05,0.3]$}&\makecell[l]{$x_1 \in [0.4, 0.7]$\\$ x_2 \in [-0.1,0.2]$}\\
         \hline
         B2 & [0.1, 0.1]&\makecell[l]{$x_1 \in [0.7, 0.9]$\\$x_2 \in [0.7,0.9]$}& \makecell[l]{$x_1 \in [-0.3, 0.1]$\\$ x_2 \in [-0.35,0.5]$}&\makecell[l]{$x_1 \in [0.12, 0.42]$\\$ x_2 \in [0.1,0.6]$}\\
         \hline
         B3 & [0.2, 0.2]&\makecell[l]{$x_1 \in [0.8, 0.9]$\\$ x_2 \in [0.4,0.5]$}& \makecell[l]{$x_1 \in [0.2, 0.3]$\\$ x_2 \in [-0.3,-0.05]$}&\makecell[l]{$x_1 \in [0.55, 0.75]$\\$x_2 \in [-0.1,0.1]$}\\
         \hline
         B4 & [0.2, 0.2, 0.2]&\makecell[l]{$x_1 \in [0.25, 0.27]$\\$ x_2 \in [0.08,0.1]$\\$x_3 \in [0.25, 0.27]$}& \makecell[l]{$x_1 \in [-0.05, 0.05]$\\$ x_2 \in [-0.05,0]$}&\makecell[l]{$x_1 \in [0.05, 0.1]$\\$x_2 \in [0.02,0.04]$}\\
         \hline
         B5 & [0.1, 0.1, 0.1]&\makecell[l]{$x_1 \in [0.38, 0.4]$\\$ x_2 \in [0.45,0.47]$\\$ x_3 \in [0.25,0.27]$}& \makecell[l]{$x_1 \in [-0.4, -0.28]$\\$ x_2 \in [0.05,0.22]$}&\makecell[l]{$x_1 \in [-0.05, 0.05]$\\$x_2 \in [0.15, 0.25]$}\\
         \hline
         Tora &\makecell[l]{[0.2, 0.2,0.2, 0.2]} &\makecell[l]{$x_1 \in [-0.77, -0.75]$\\$ x_2 \in [-0.45,-0.43]$\\$x_3 \in [0.51,0.54]$\\$ x_4 \in [-0.3,-0.28]$}& \makecell[l]{$x_1 \in [-0.1, 0.2]$\\$ x_2 \in [-0.9,-0.6]$}&\makecell[l]{$x_1 \in [-0.25, 0.10]$\\$x_2 \in [0.2, 0.7]$}\\
         \hline
         ACC & \makecell[l]{[1, 0.1, 0.1, 1, 0.1,\\ 0.1]}&\makecell[l]{$x_1 \in [90, 91]$\\$ x_2 \in [32, 32.05]$\\ $x_4 \in [10, 11]$\\ $x_5 \in [30, 30.05]$\\$x_3, x_6 \in [0, 0]$}&\makecell{$x_2 \in [22.81, 22.87]$\\$x_5 \in [29.88, 30.02]$}&\makecell[l]{$x_2 \in [26, 29]$\\$x_5 \in [30.05, 30.15]$}\\
\bottomrule
    \end{tabular}
\end{table}
\vspace{-5ex}

\section{More Results on Training Reward Comparison}
\label{app:train_reward_com}
Figure \ref{fig:app_cumulative_reward} complements the comparison results of the trend of cumulative rewards during training between different network structures (i.e. ANN and DNN) in B5, Tora and ACC. The solid lines and
the shadows indicate the average reward and 95\% confidence interval, respectively. We can also observe that even integrating an abstraction layer into ANN can achieve nearly the same cumulative reward as using a DNN.
\vspace{-2ex}

\begin{figure}[h!]
	    \setlength{\tabcolsep}{7pt}
		\begin{tabular}{ccc}
			\begin{subfigure}[b]{0.3\textwidth}
				\includegraphics[width=\textwidth]{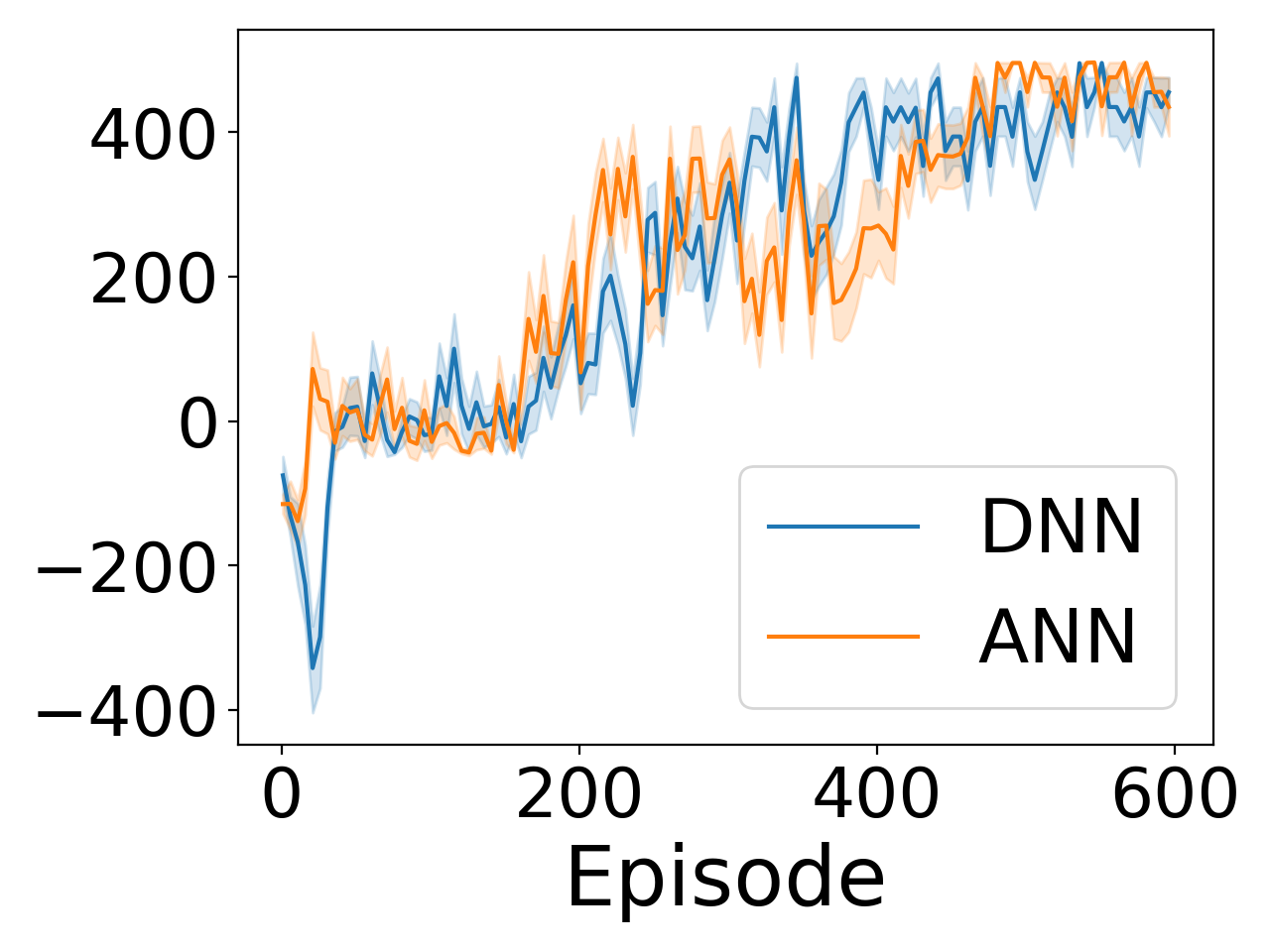}
				\caption{B5}
				\label{fig:reward_com_b5}
			\end{subfigure}&
			\begin{subfigure}[b]{0.3\textwidth}
				\includegraphics[width=\textwidth]{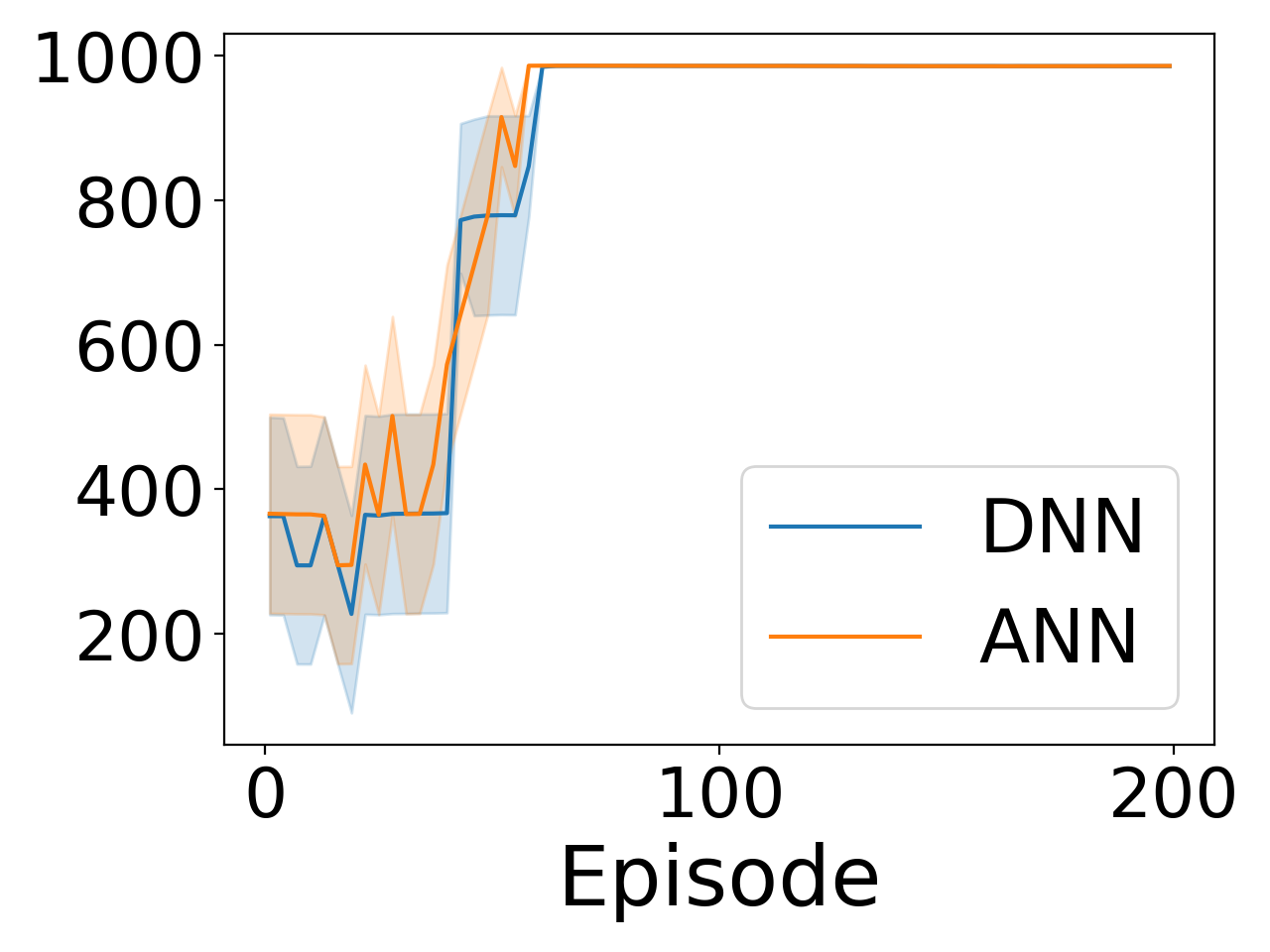}
				\caption{Tora}
				\label{fig:reward_com_tora}
			\end{subfigure}&
			\begin{subfigure}[b]{0.3\textwidth}
				\includegraphics[width=\textwidth]{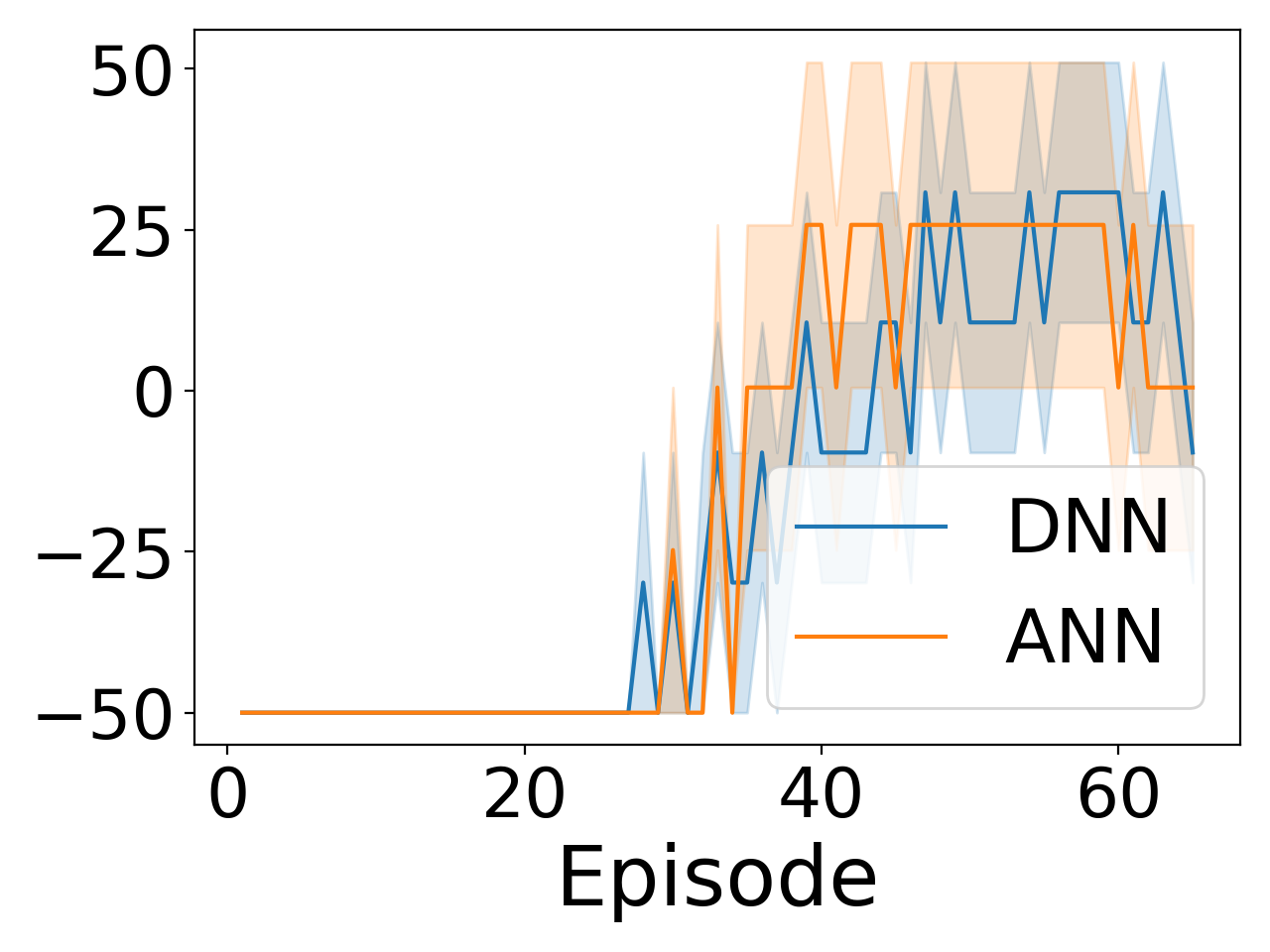}
				\caption{ACC}
				\label{fig:reward_com_acc}
			\end{subfigure}
	\end{tabular}
	\vspace{-3mm}
	\caption{Trend  of cumulative rewards (y-axis) of the systems controlled by ANNs (orange) and DNNs (blue) trained by  DDPG algorithm.}
	\label{fig:app_cumulative_reward}
	\vspace{-3mm}
\end{figure}

  \clearpage 
\section{More Results on Tightness Comparison}
\label{subsec:tight_com}
This section presents the tightness comparison results on B3, B4, B5 and ACC in Figure~\ref{fig:reachable_sets_b345}. We also initialize the neural networks with small weights. For these benchmarks, both Verisig 2.0 and Polar keep small over-approximation error when dealing with neural networks and thus obtain tight results.  \BBReach~does not have the over-approximation error for neural networks. Therefore, all methods achieve similar results: all reach-avoid problems are successfully verified and the range of sampled trajectories is almost indistinguishable with the over-approximation sets.  
\begin{figure}[h!]	
	\begin{minipage}[b]{0.08\linewidth}
\begin{tikzpicture}
\tikz \node [draw] at (0.5,-1)  {B3};
\end{tikzpicture}
	\end{minipage}
	\begin{minipage}[b]{0.3\linewidth}
		\centering
		{ \textbf{Our Approach}}\\
		 \vspace{1ex}
		\includegraphics[scale=0.23]{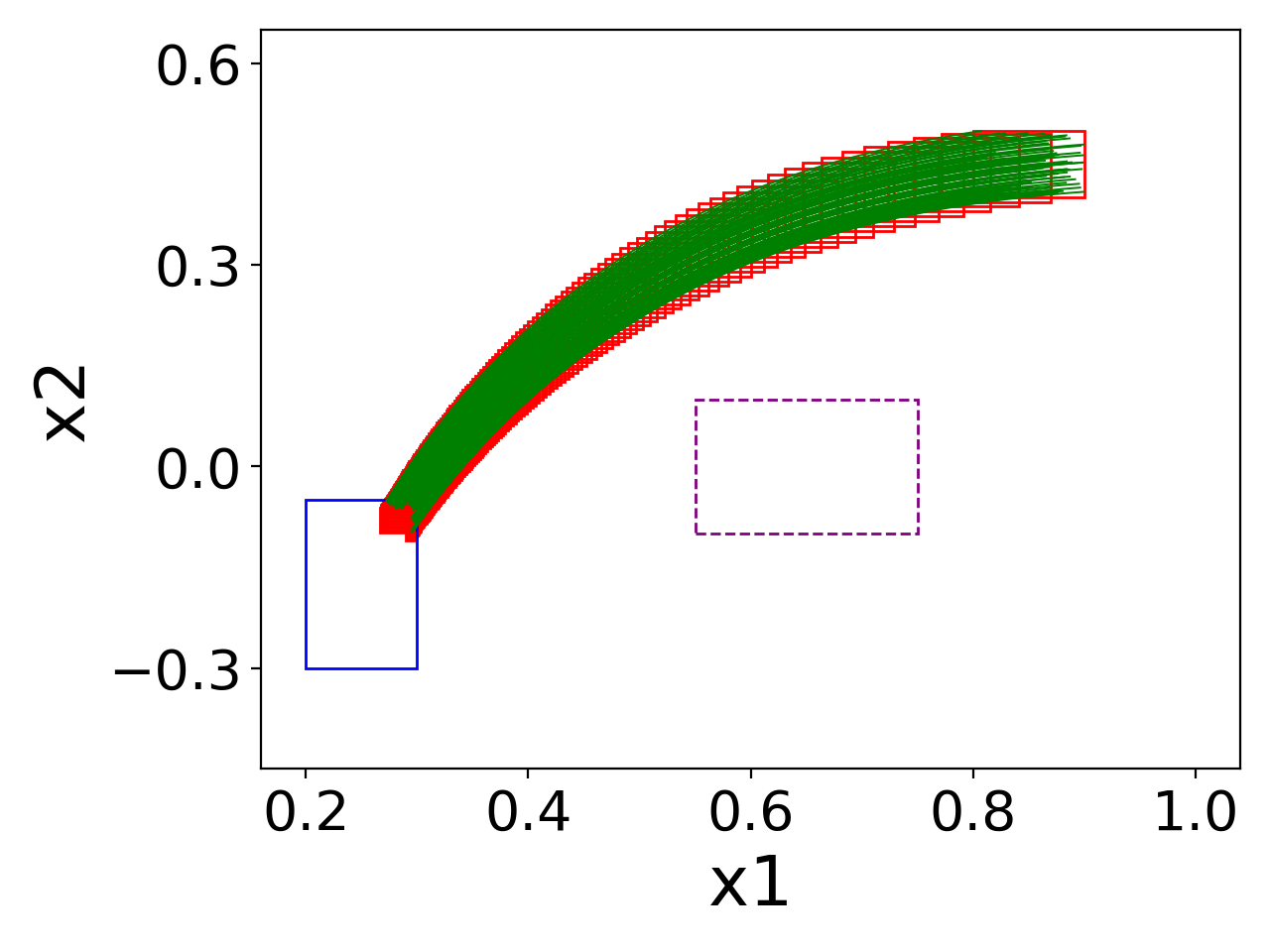}
	\end{minipage}
		\begin{minipage}[b]{0.3\linewidth}
		\centering
		{ \textbf{Verisig 2.0}}\\
		\vspace{1ex}
		\includegraphics[scale=0.23]{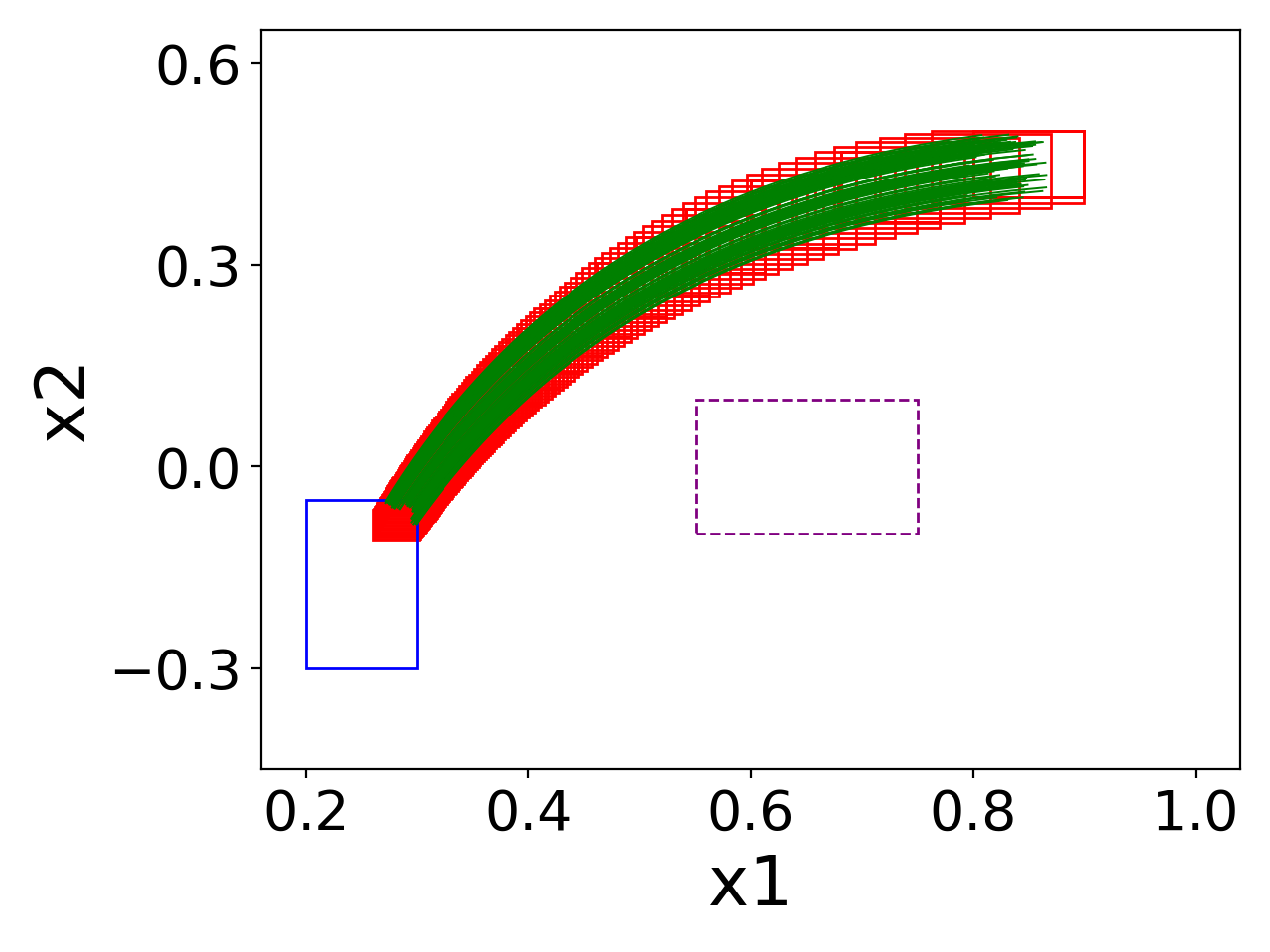}
	\end{minipage}
	\begin{minipage}[b]{0.3\linewidth}
		\centering
	{ \textbf{Polar}}\\
	\vspace{1ex}
	\includegraphics[scale=0.23]{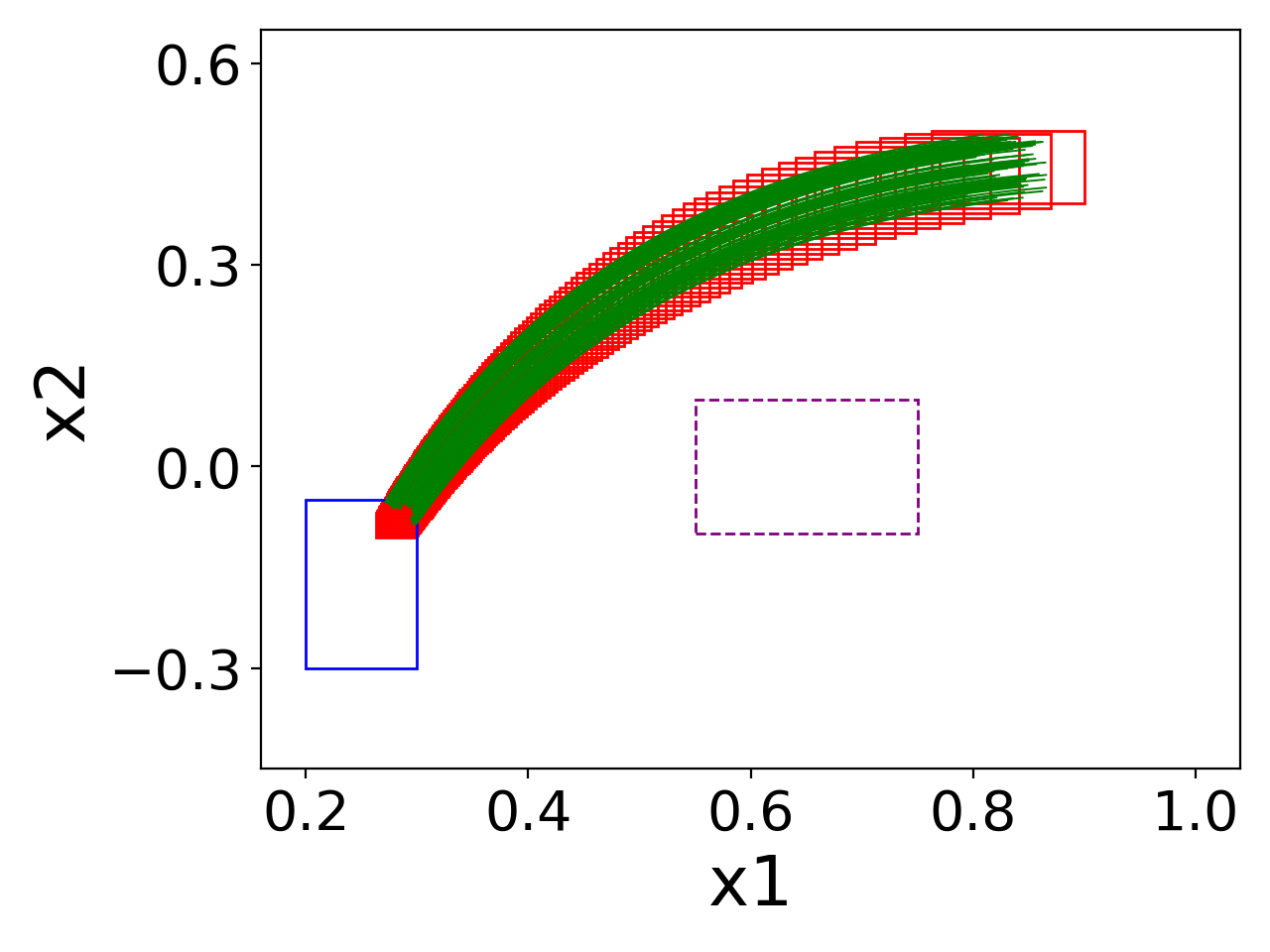}
	\end{minipage}
\\
\vspace{1ex}
	\begin{minipage}[b]{0.06\linewidth}
\begin{tikzpicture}
\tikz \node [draw] at (0.5,-1)  {B4};
\end{tikzpicture}
	\end{minipage}
	\begin{minipage}[b]{0.3\linewidth}
		\centering
		\includegraphics[scale=0.23]{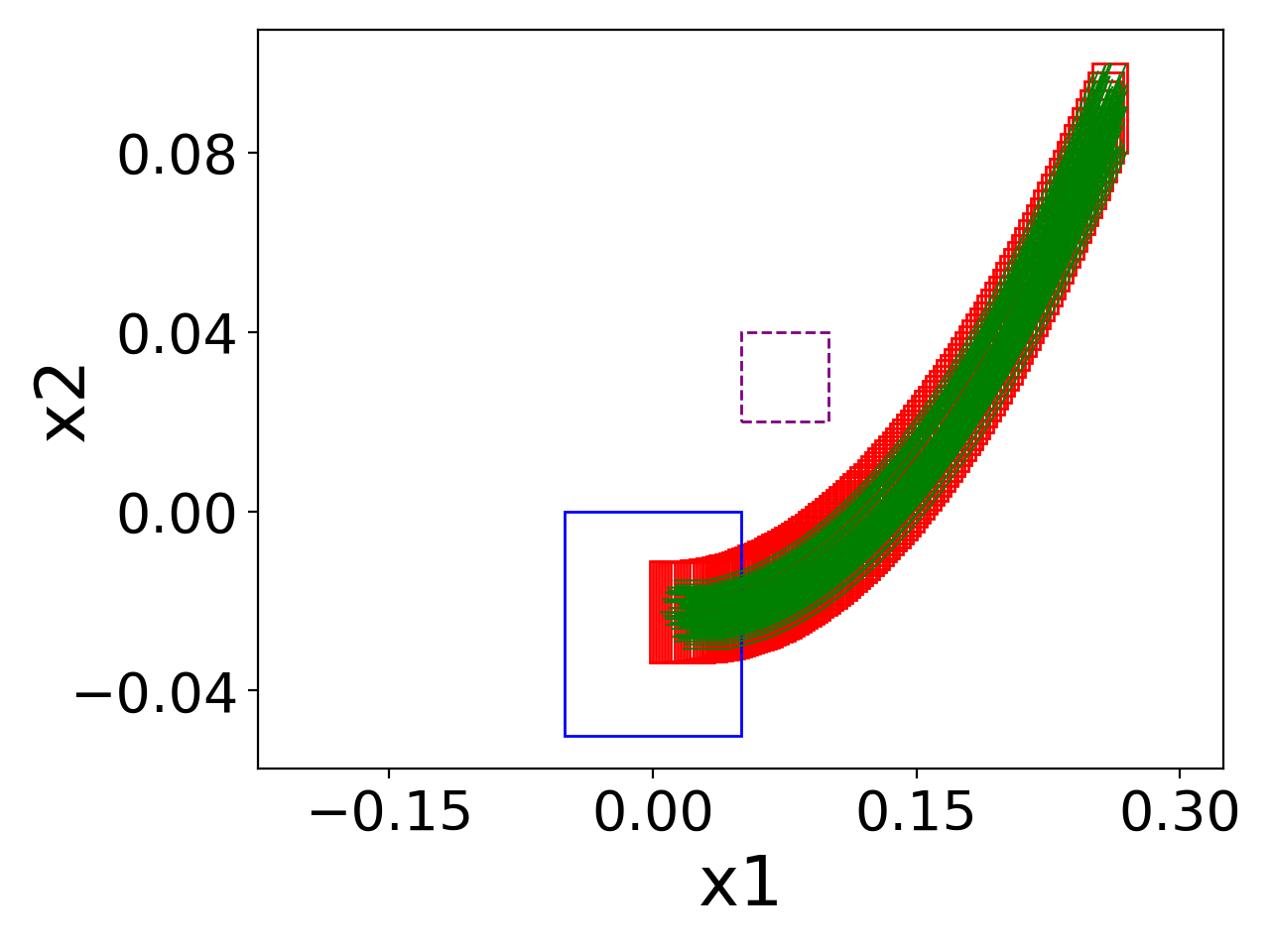}
	\end{minipage}
		\begin{minipage}[b]{0.3\linewidth}
		\centering
		\includegraphics[scale=0.23]{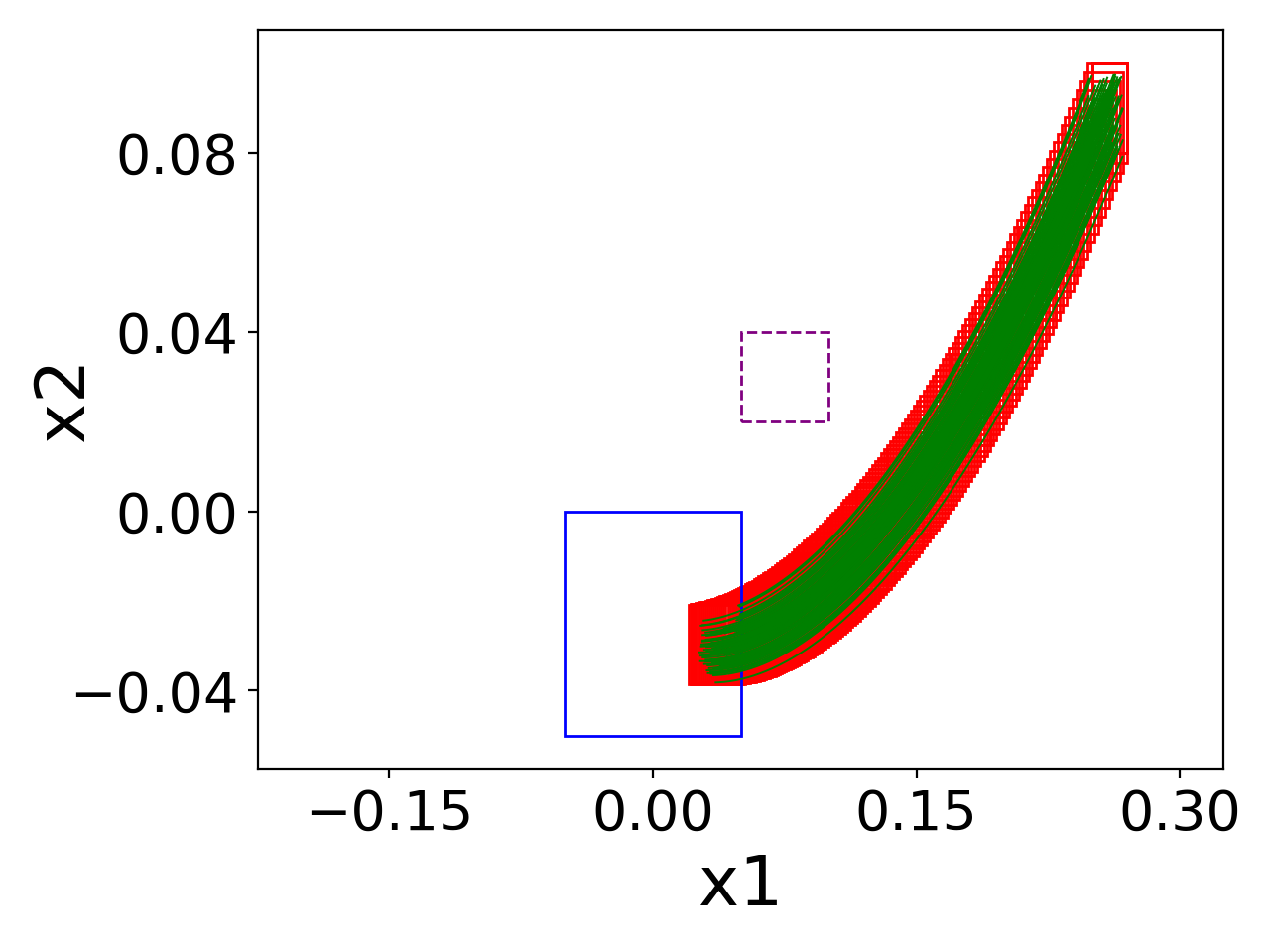}
	\end{minipage}
	\begin{minipage}[b]{0.3\linewidth}
		\centering
	\includegraphics[scale=0.23]{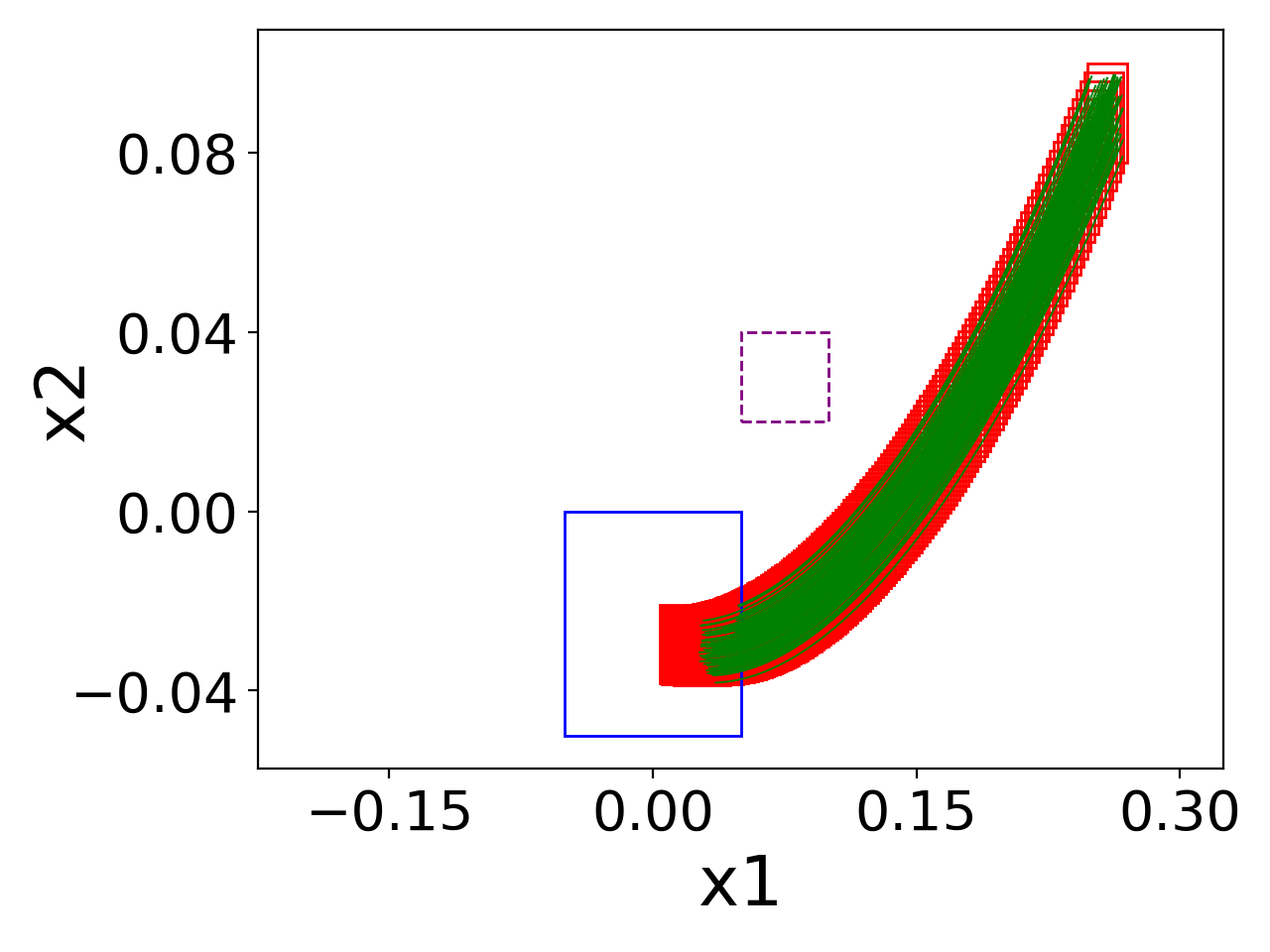}
	\end{minipage}
\\
\vspace{1ex}
	\begin{minipage}[b]{0.06\linewidth}
\begin{tikzpicture}
\tikz \node [draw] at (0.5,-1)  {B5};
\end{tikzpicture}
	\end{minipage}
	\begin{minipage}[b]{0.3\linewidth}
		\centering
		\includegraphics[scale=0.23]{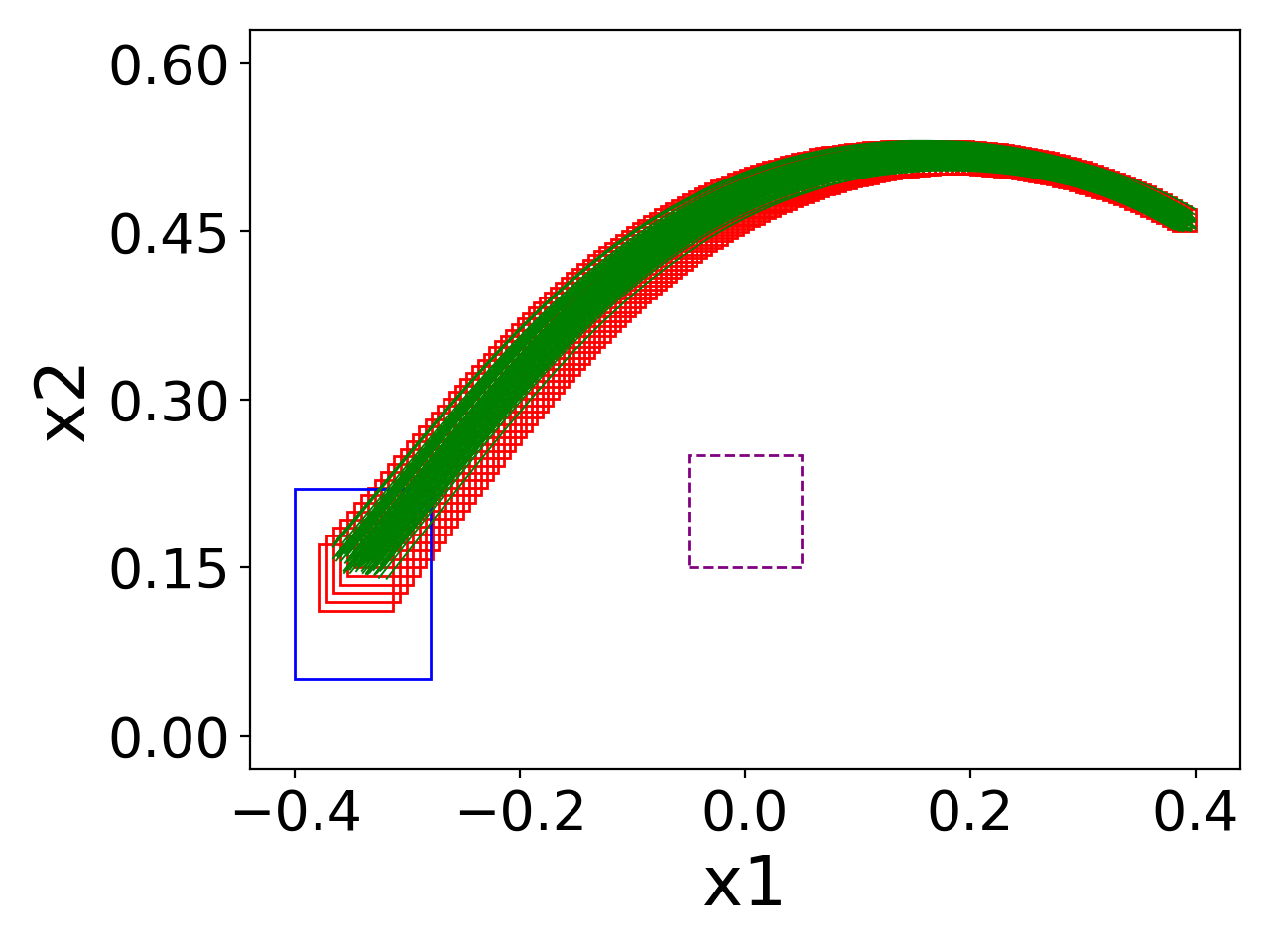}
	\end{minipage}
		\begin{minipage}[b]{0.3\linewidth}
		\centering
		\includegraphics[scale=0.23]{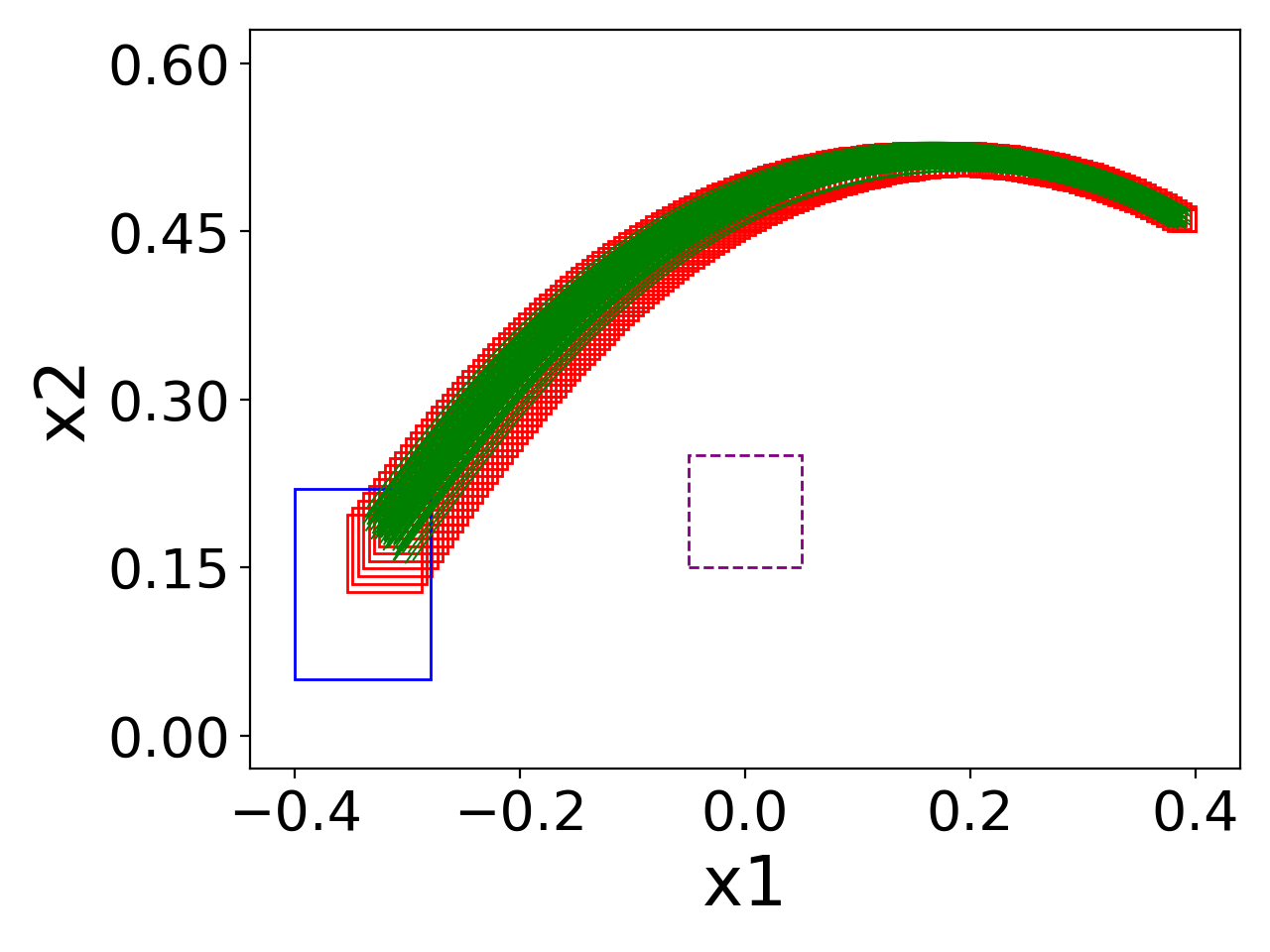}
	\end{minipage}
	\begin{minipage}[b]{0.3\linewidth}
		\centering
	\includegraphics[scale=0.23]{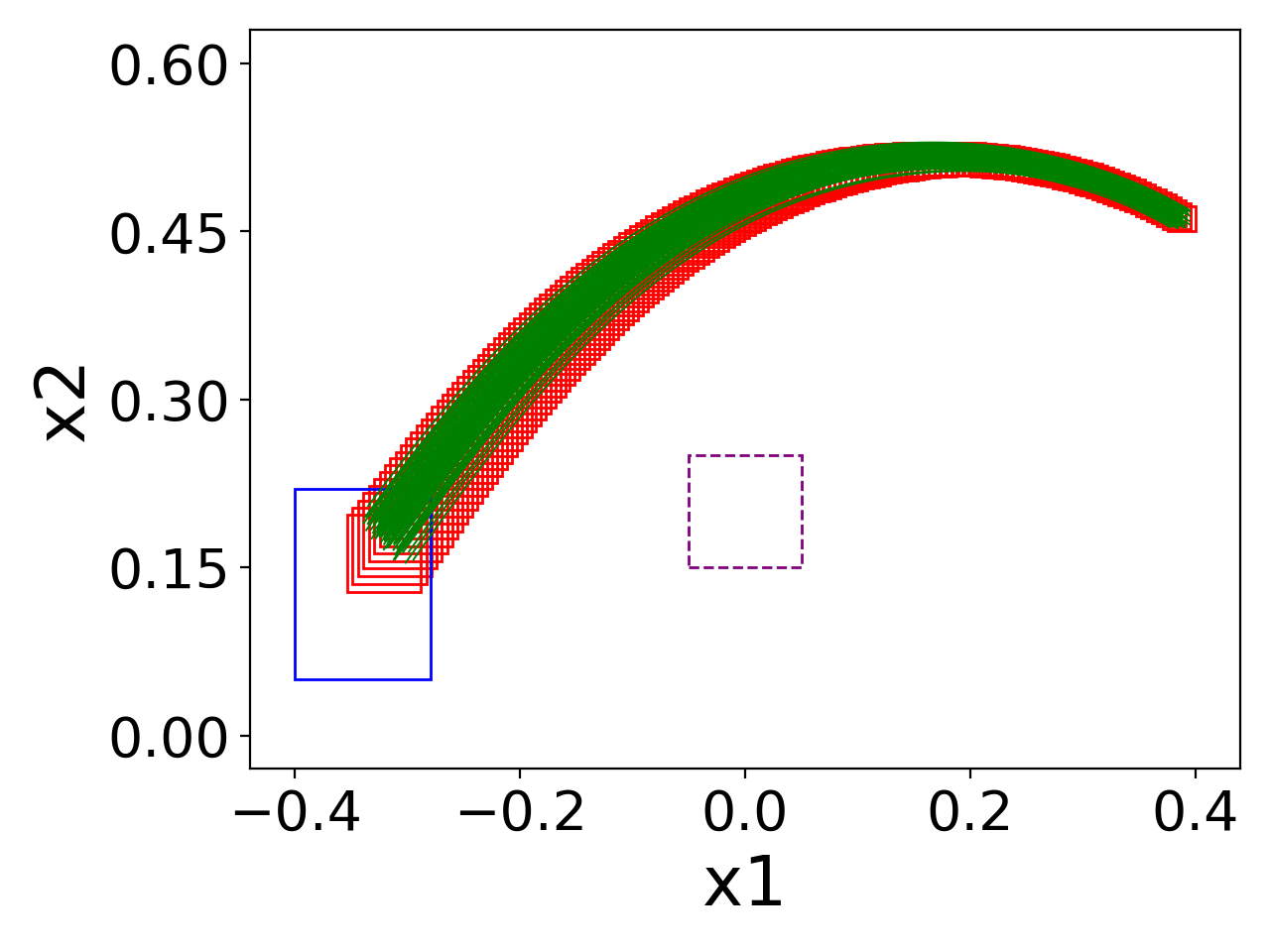}
	\end{minipage}
	\\
\vspace{1ex}
	\begin{minipage}[b]{0.06\linewidth}
\begin{tikzpicture}
\tikz \node [draw] at (0.6,-1)  {ACC};
\end{tikzpicture}
	\end{minipage}
	\begin{minipage}[b]{0.3\linewidth}
		\centering
		\includegraphics[scale=0.23]{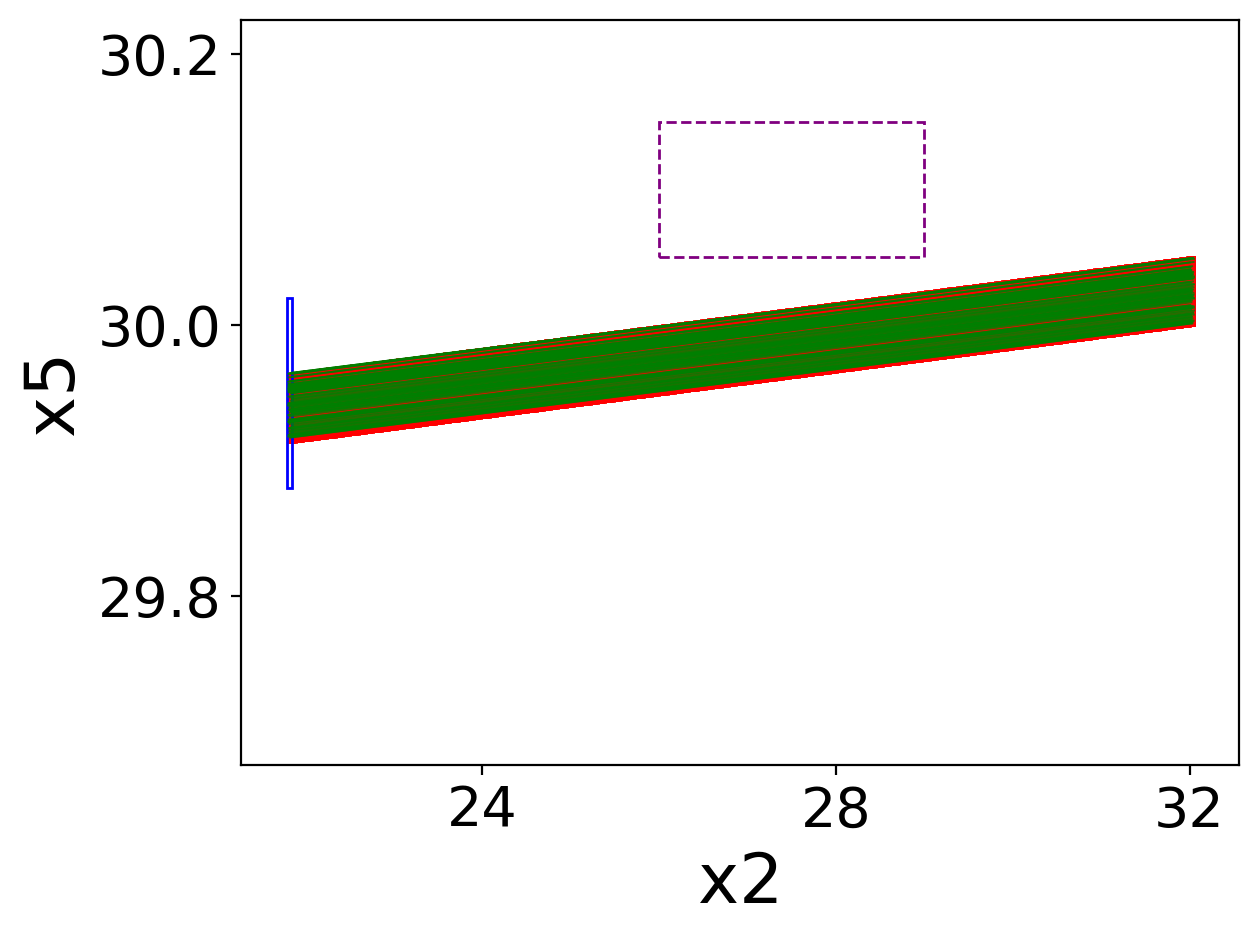}
	\end{minipage}
		\begin{minipage}[b]{0.3\linewidth}
		\centering
		\includegraphics[scale=0.23]{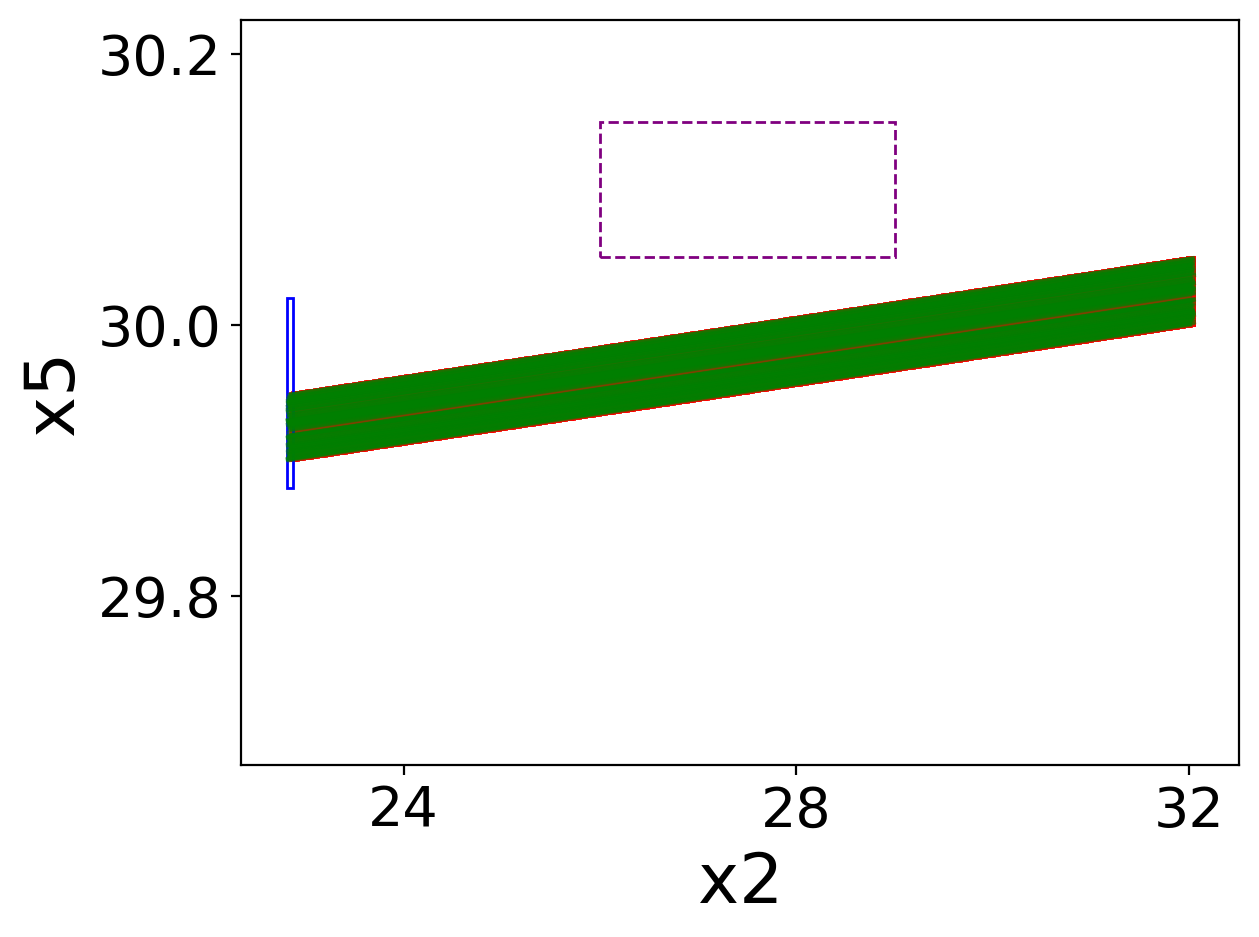}
	\end{minipage}
	\begin{minipage}[b]{0.3\linewidth}
		\centering
	\includegraphics[scale=0.23]{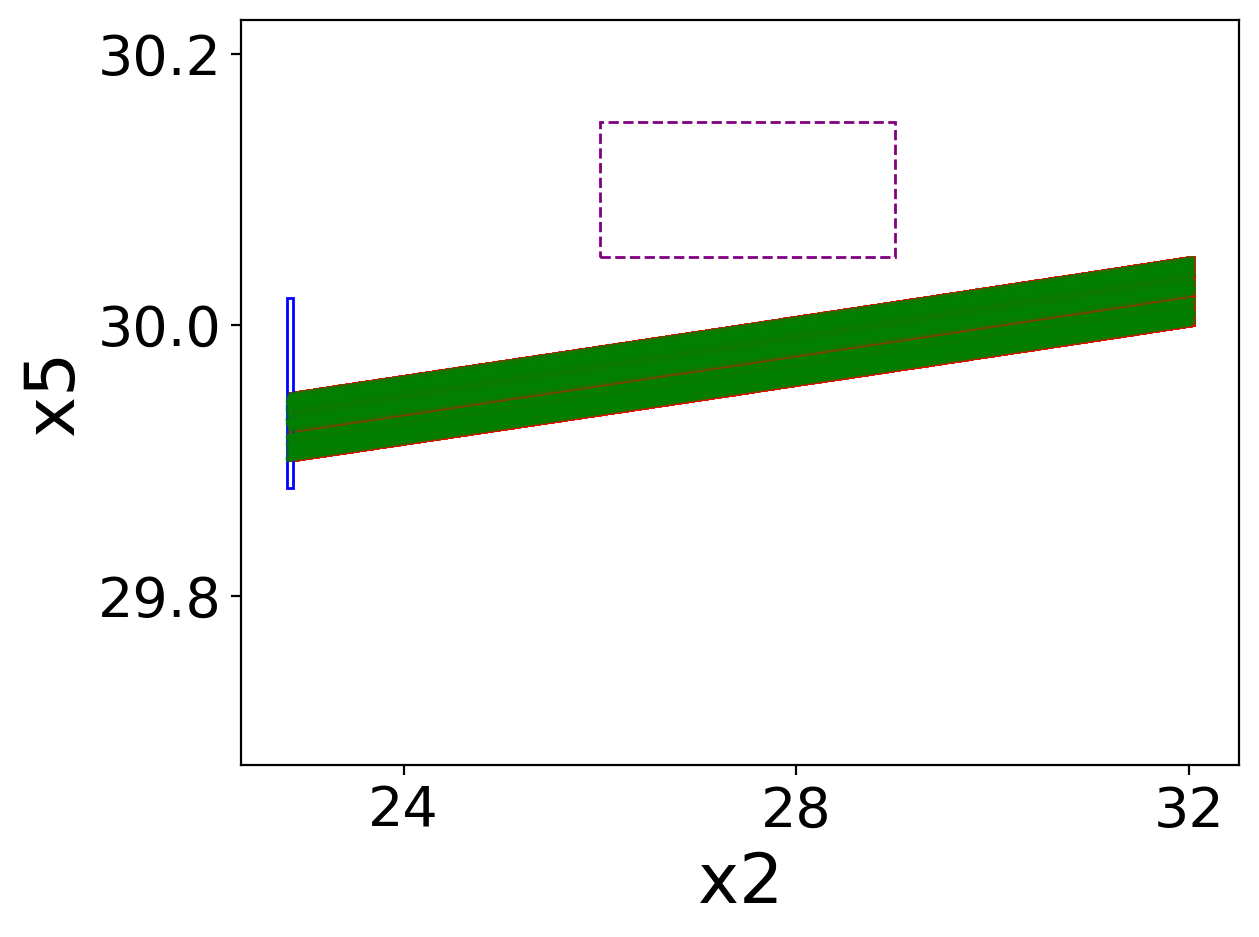}
	\end{minipage}
	\caption{Larger networks with the Tanh activation function.  green lines: simulation trajectories; red box: over-approximation set; blue box: goal region; 
	purple dashed box: unsafe region.} 
	\label{fig:reachable_sets_b345}
\end{figure}

 \clearpage 

\section{Differential and Decomposing Analysis Results}
\label{subsec:diff_dec_result}
In this section, we provide the complete differential and decomposing analysis results in Figures~\ref{fig:differential_result} and \ref{fig:decompose_result}.
Figure~\ref{fig:differential_result} shows the effect of adjacent interval aggregation, from which we can observe that with adjacent interval aggregation the number of interval boxes is small and stable compared to no aggregation, which makes our calculation more efficient while does not produce large over-approximation error as depicted in Figure~\ref{fig:reachable_sets} and Figure~\ref{fig:reachable_sets_b345}.
The abstraction granularity has a similar impact on the performance of \BBReach~as illustrated in Figure~\ref{fig:decompose_result}: the coarser-grained granularity leads to less time required.
 All of these results are consistent with the conclusion in Section~\ref{subsec:discussion}.

\begin{figure}[htbp]
\begin{center}
		\begin{tabular}{ccc}
		\begin{subfigure}[b]{0.32\textwidth}
			\includegraphics[width=\textwidth]{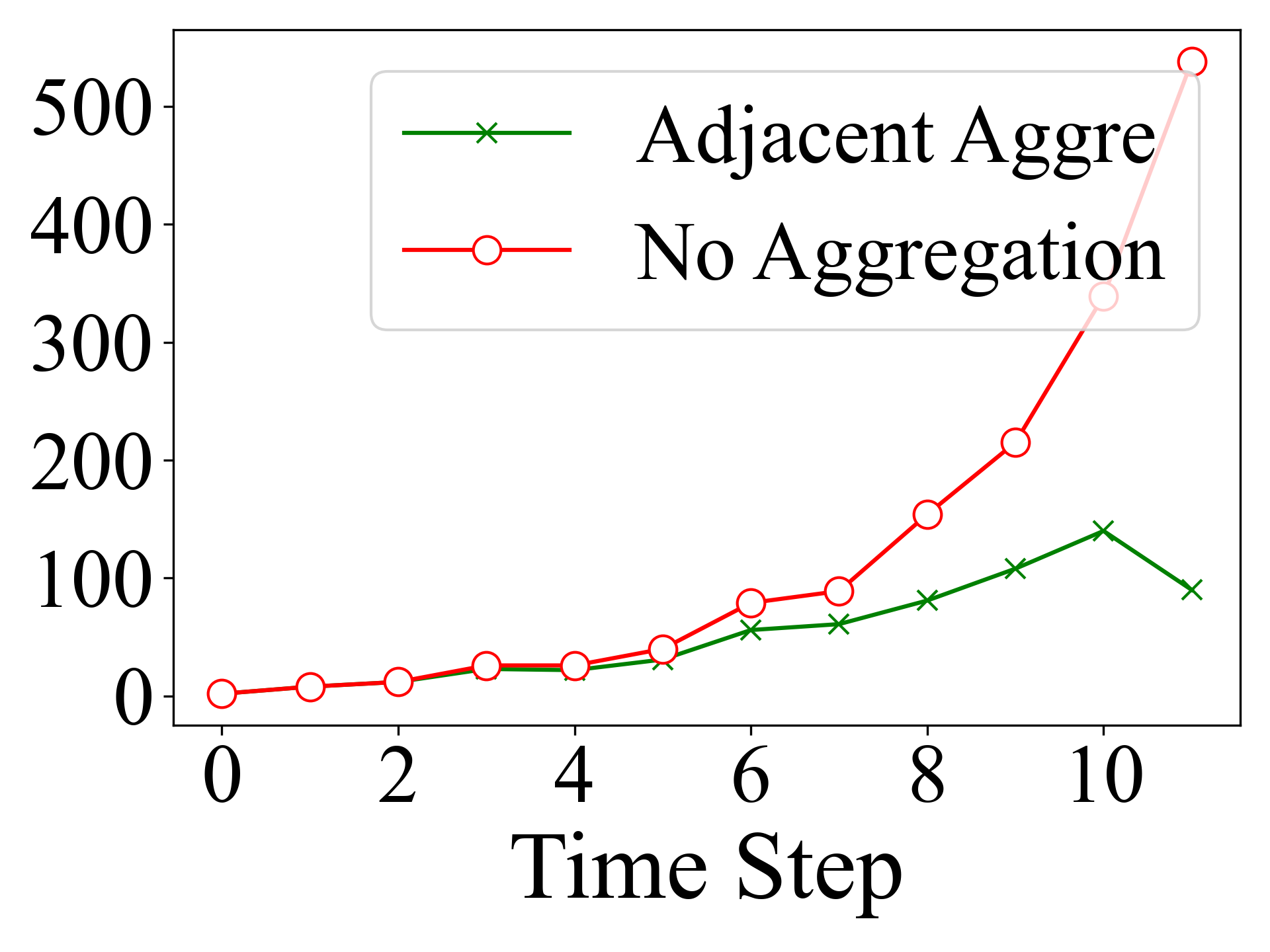}
			\caption{B5}
			\label{fig:b5_agg_com}
		\end{subfigure}&
		\begin{subfigure}[b]{0.32\textwidth}
			\includegraphics[width=\textwidth]{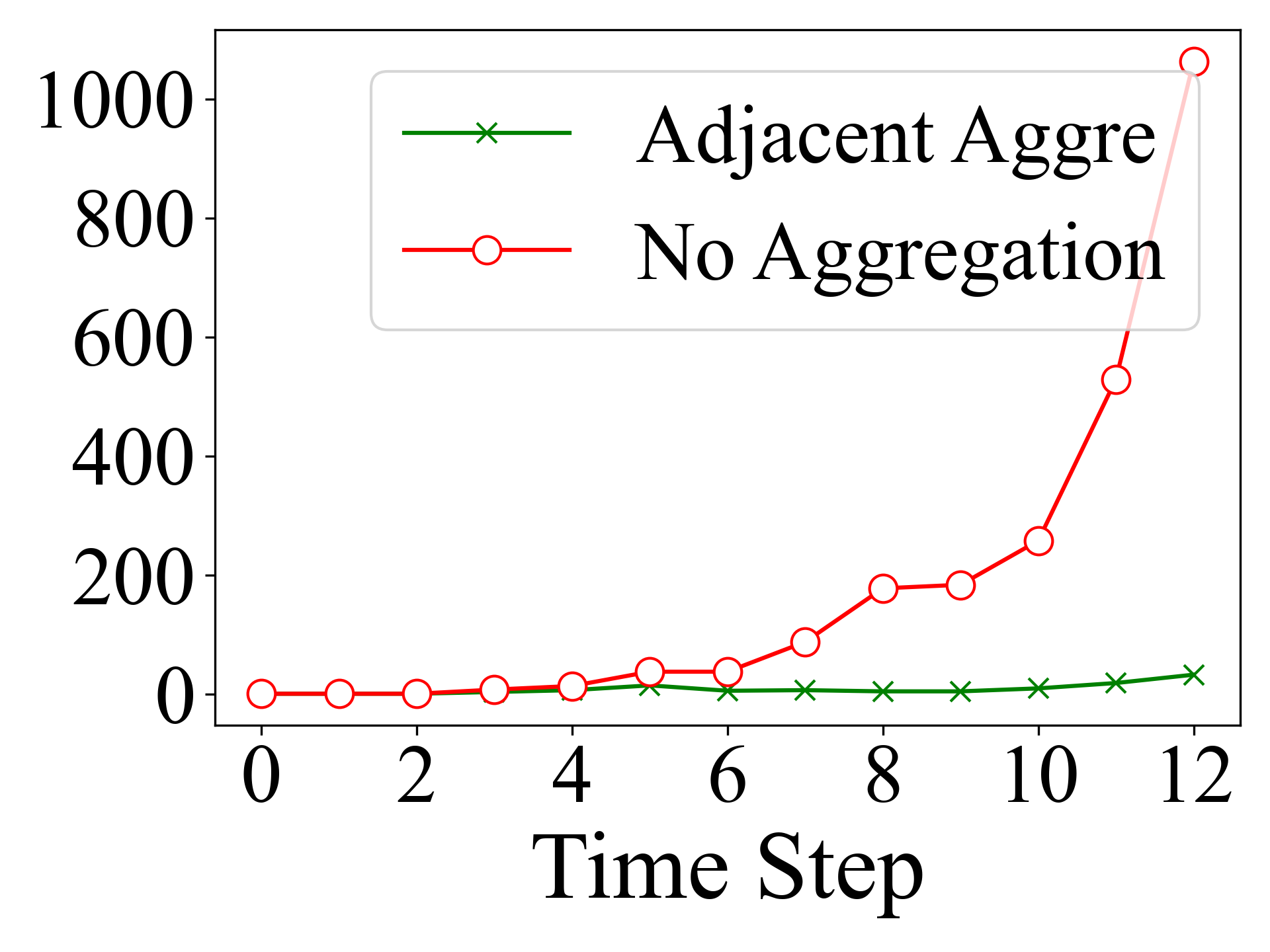}
			\caption{Tora}
			\label{fig:tora_agg_com}
		\end{subfigure}&
		\begin{subfigure}[b]{0.32\textwidth}
			\includegraphics[width=\textwidth]{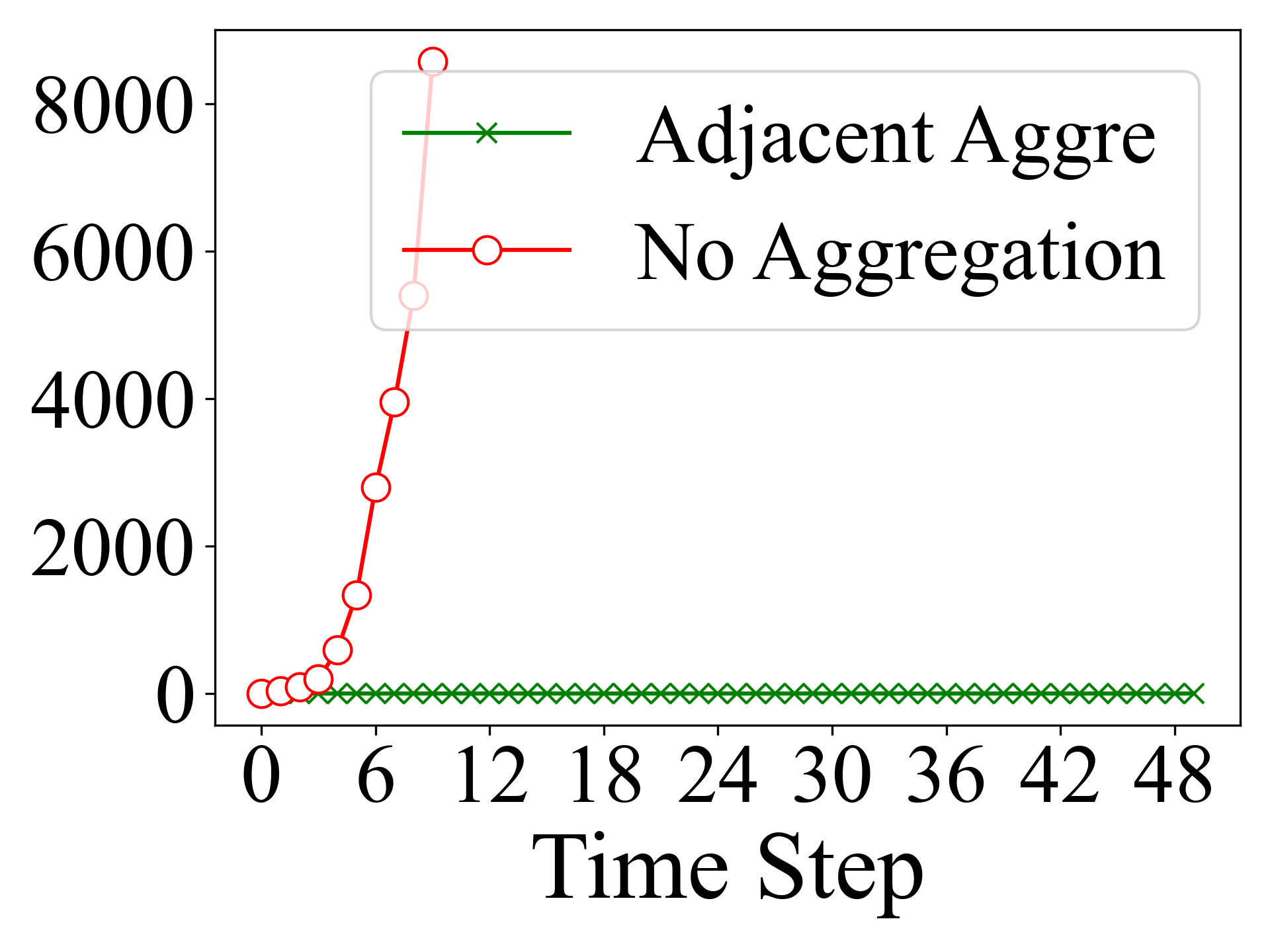}
			\caption{ACC}
			\label{fig:acc_agg_com}
		\end{subfigure}
		\\
	\end{tabular}
\end{center}
\vspace{-3ex}
\caption{Differential analysis results on the effect of adjacent aggregation. Y-axis indicates the number of interval boxes.}
\label{fig:differential_result}
\end{figure}

\begin{figure}[h!]
\begin{center}
		\begin{tabular}{cccc}
		\begin{subfigure}[b]{0.25\textwidth}
			\includegraphics[width=\textwidth]{imgs/b1_tanh_abstraction_granularity_com.png}
			\caption{B1 (tanh)}
		\end{subfigure}&
		\begin{subfigure}[b]{0.25\textwidth}
			\includegraphics[width=\textwidth]{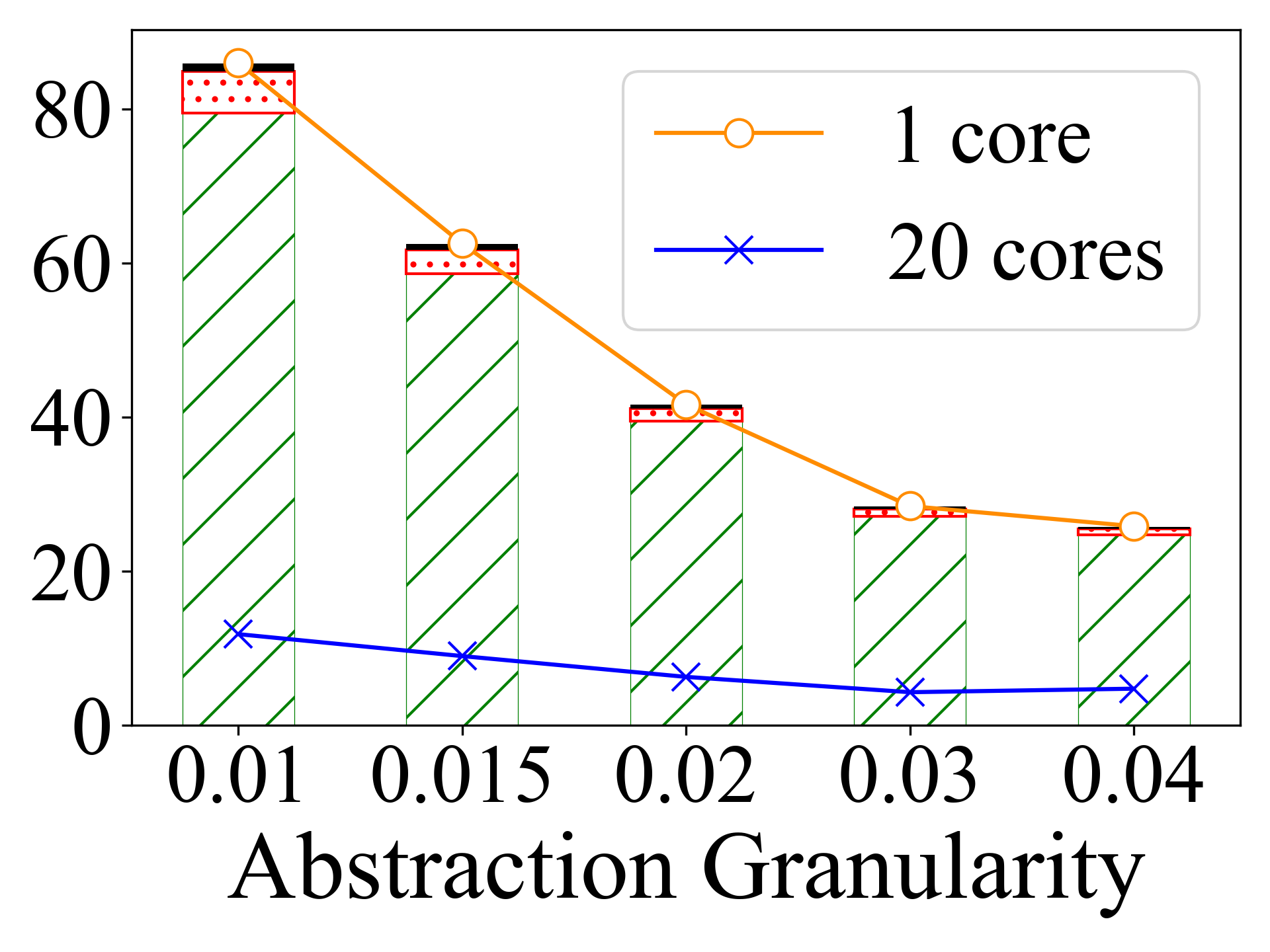}
			\caption{B1 (relu)}
			\label{fig:b1_dec_relu}
		\end{subfigure}&
		\begin{subfigure}[b]{0.25\textwidth}
			\includegraphics[width=\textwidth]{imgs/b2_tanh_abstraction_granularity_com.png}
			\caption{B2 (tanh)}
		\end{subfigure}&
				\begin{subfigure}[b]{0.25\textwidth}
			\includegraphics[width=\textwidth]{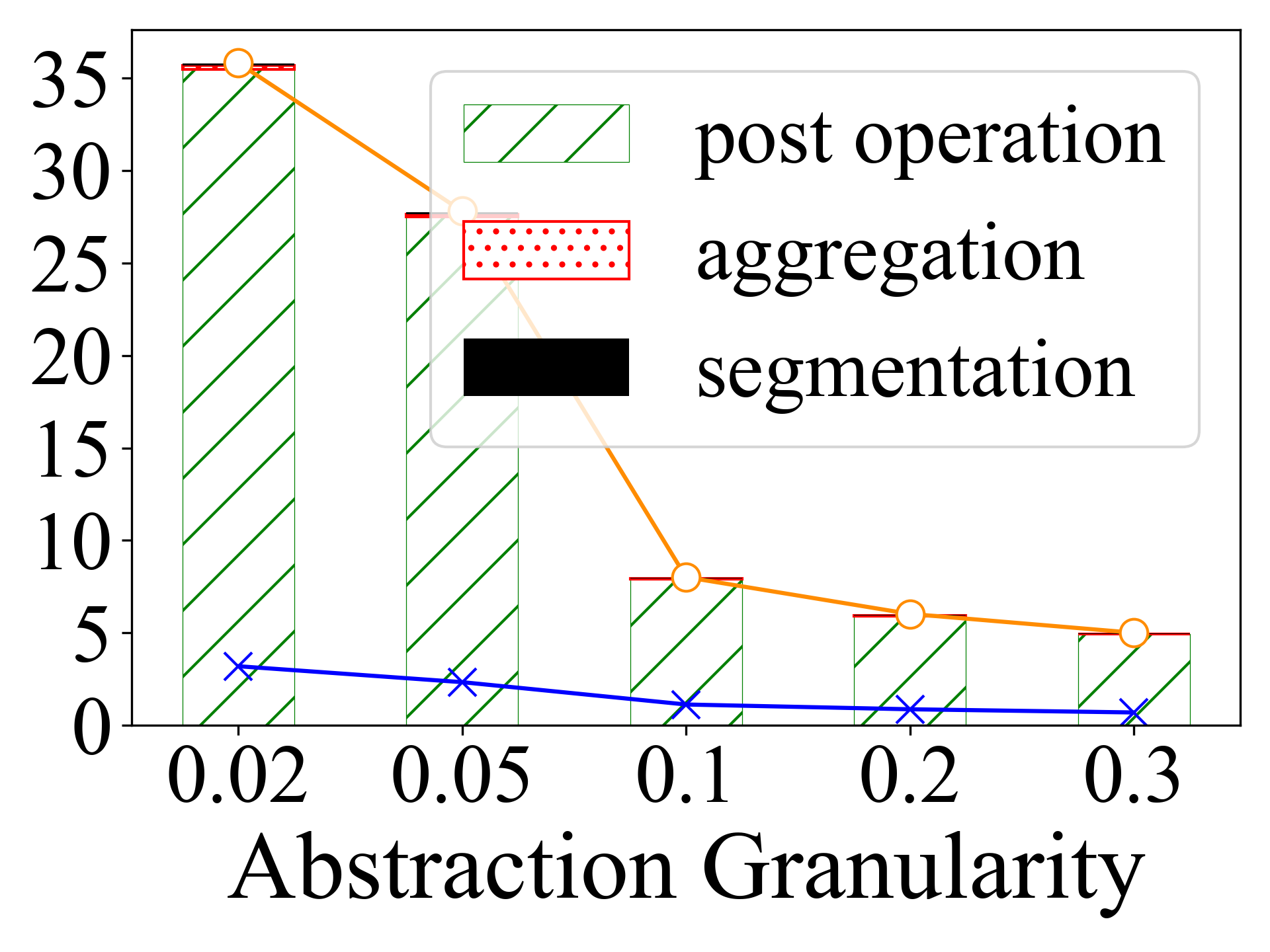}
			\caption{B2 (relu)}
			\label{fig:b2_dec_relu}
		\end{subfigure}
		\\
		\begin{subfigure}[b]{0.25\textwidth}
			\includegraphics[width=\textwidth]{imgs/b3_tanh_abstraction_granularity_com.png}
			\caption{B3 (tanh)}
		\end{subfigure}&
		\begin{subfigure}[b]{0.25\textwidth}
			\includegraphics[width=\textwidth]{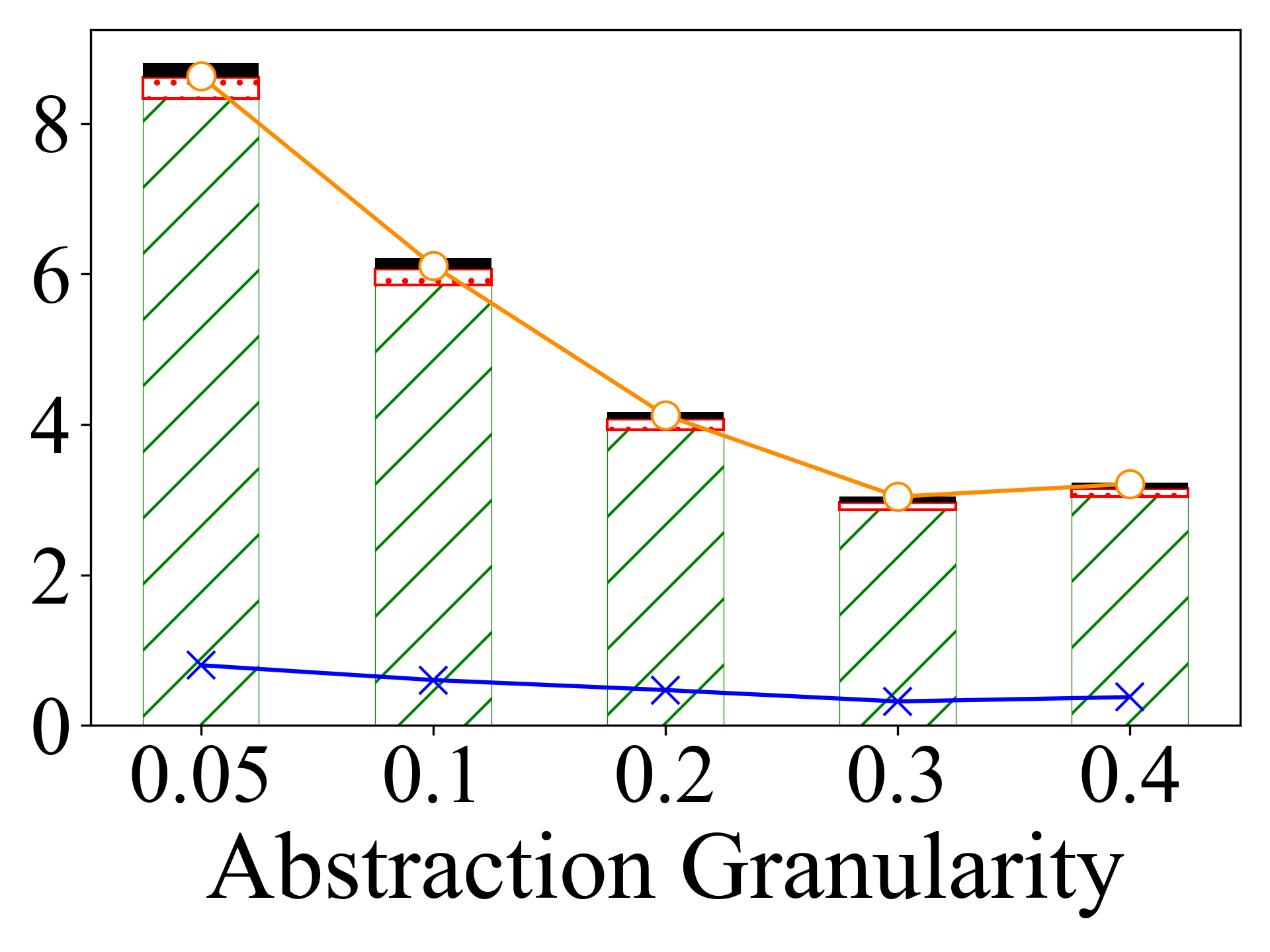}
			\caption{B3 (relu)}
			\label{fig:b3_dec_relu}
		\end{subfigure}&
		\begin{subfigure}[b]{0.25\textwidth}
			\includegraphics[width=\textwidth]{imgs/b4_tanh_abstraction_granularity_com.png}
			\caption{B4 (tanh)}
		\end{subfigure}&
				\begin{subfigure}[b]{0.25\textwidth}
			\includegraphics[width=\textwidth]{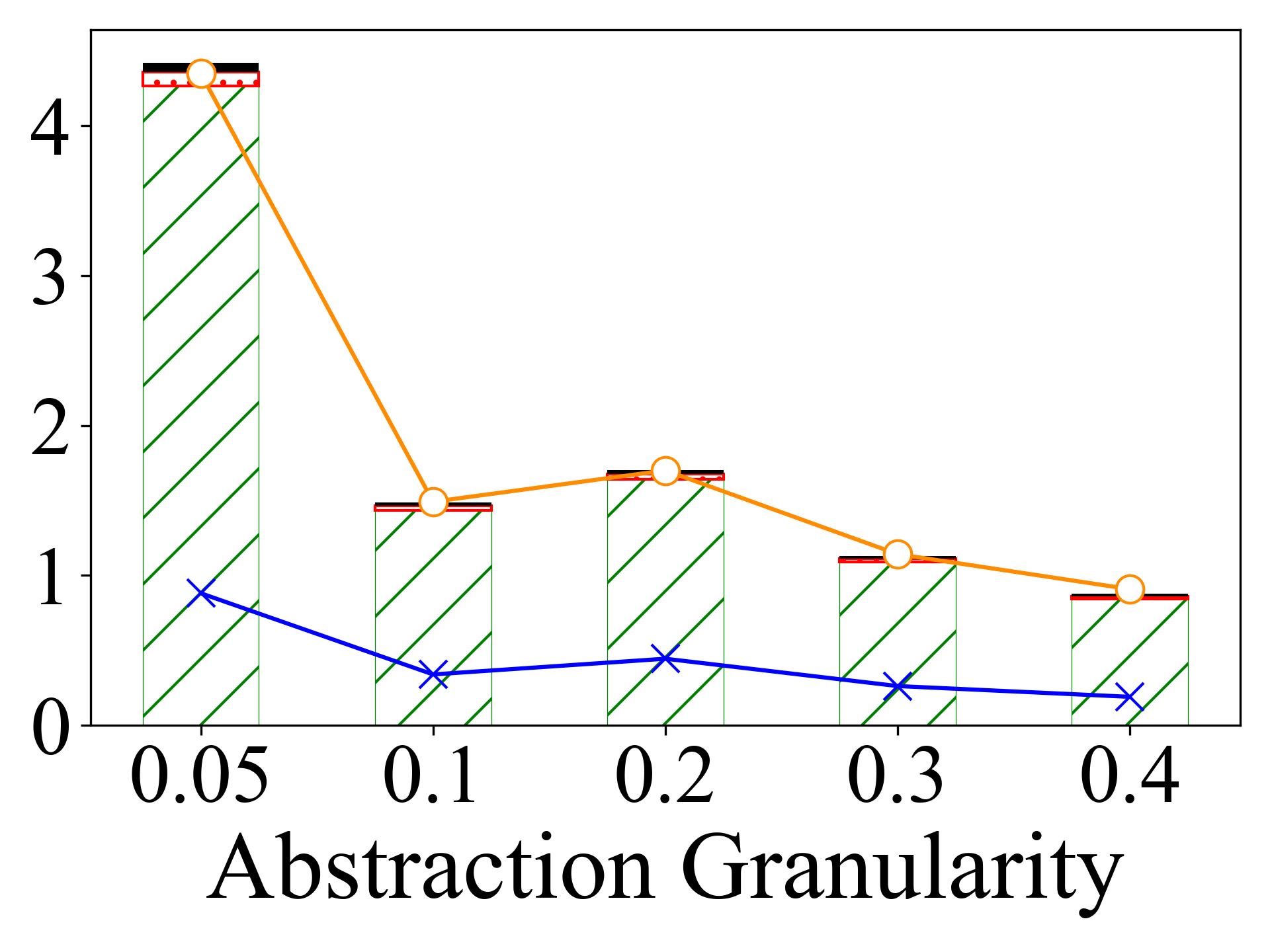}
			\caption{B4 (relu)}
			\label{fig:b4_dec_relu}
		\end{subfigure}
		\\
				\begin{subfigure}[b]{0.25\textwidth}
			\includegraphics[width=\textwidth]{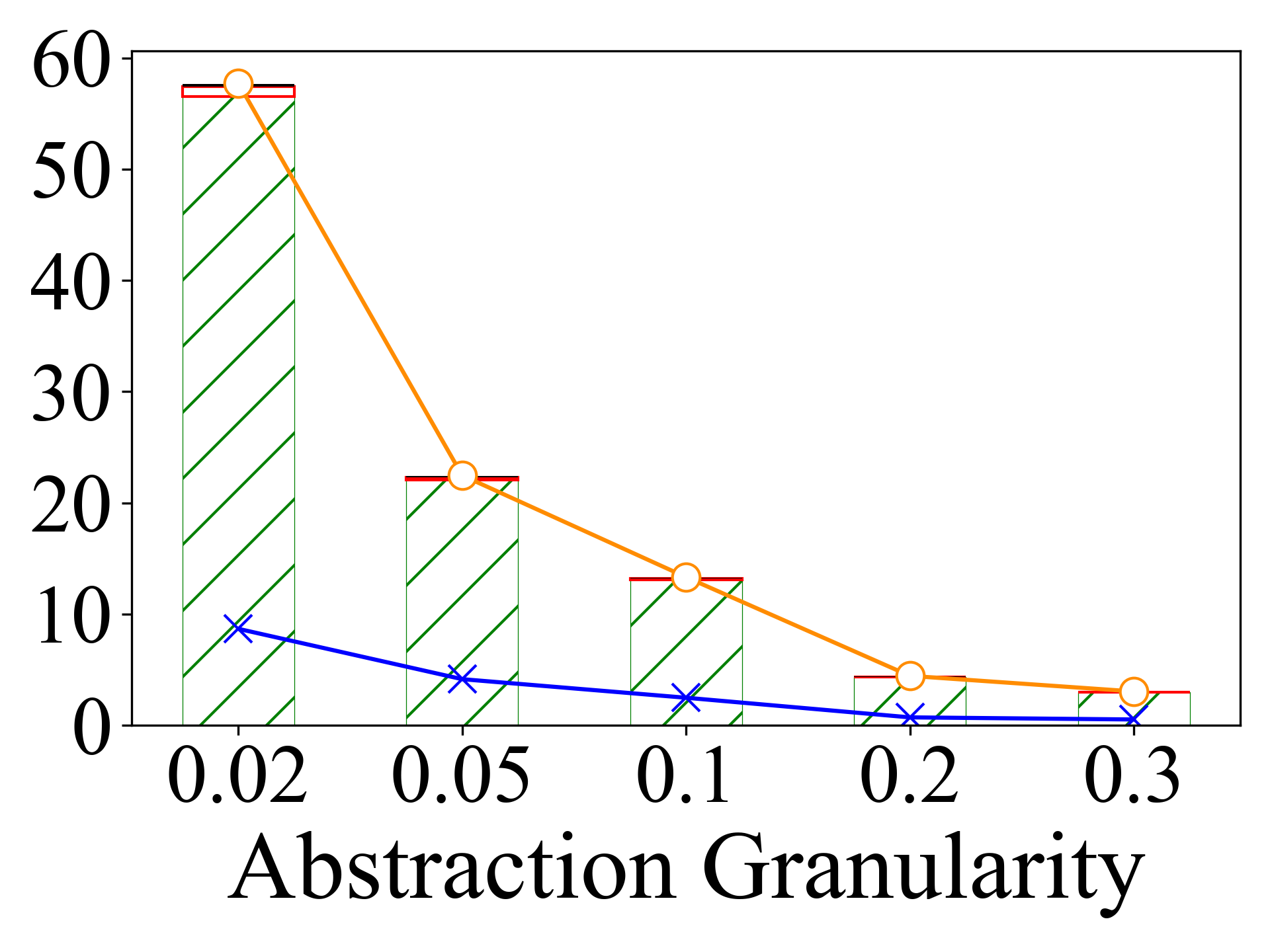}
			\caption{B5 (tanh)}
			\label{fig:b5_dec_tanh}
		\end{subfigure}&
		\begin{subfigure}[b]{0.25\textwidth}
			\includegraphics[width=\textwidth]{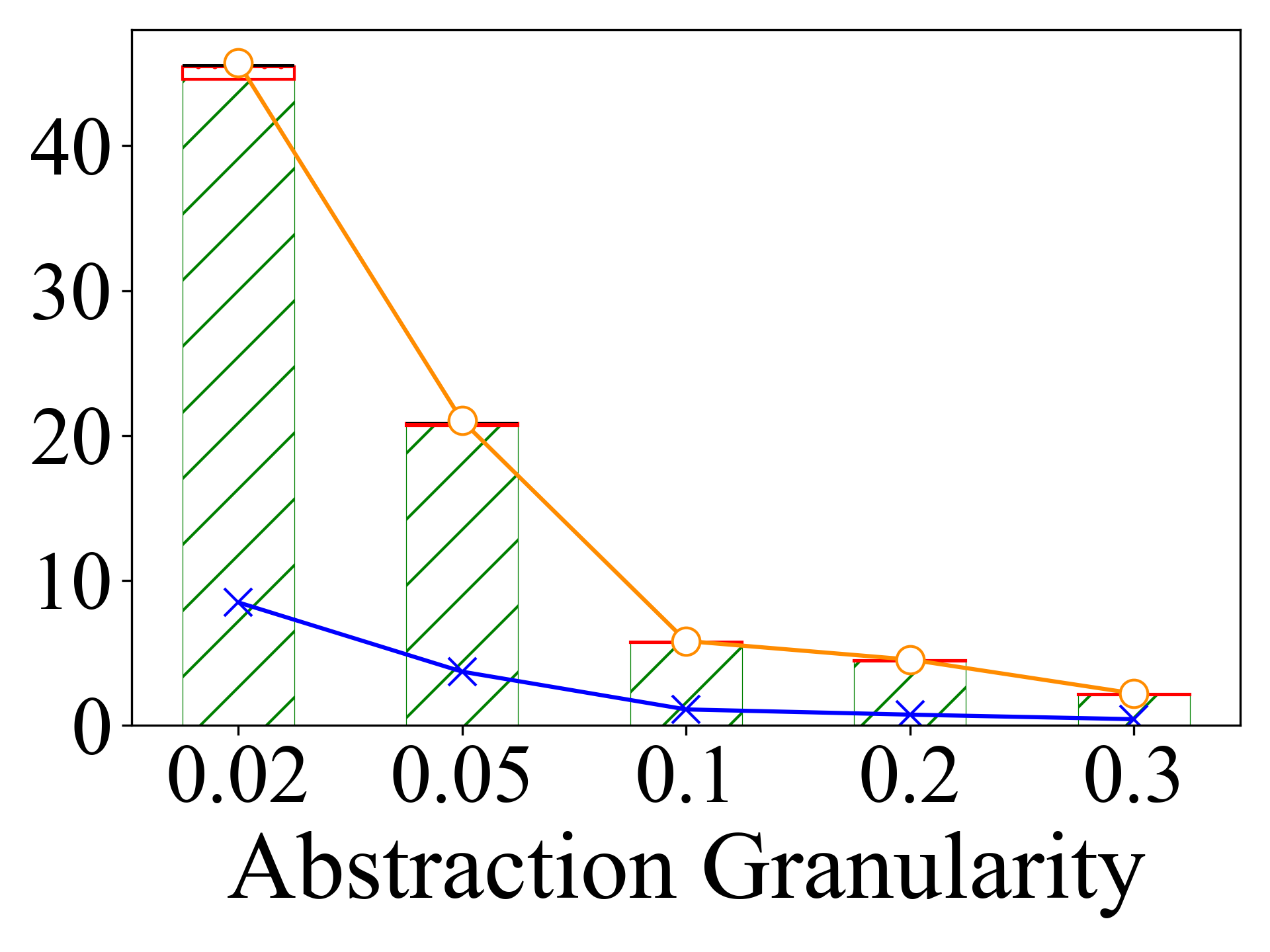}
			\caption{B5 (relu)}
			\label{fig:b5_dec_relu}
		\end{subfigure}&
		\begin{subfigure}[b]{0.25\textwidth}
			\includegraphics[width=\textwidth]{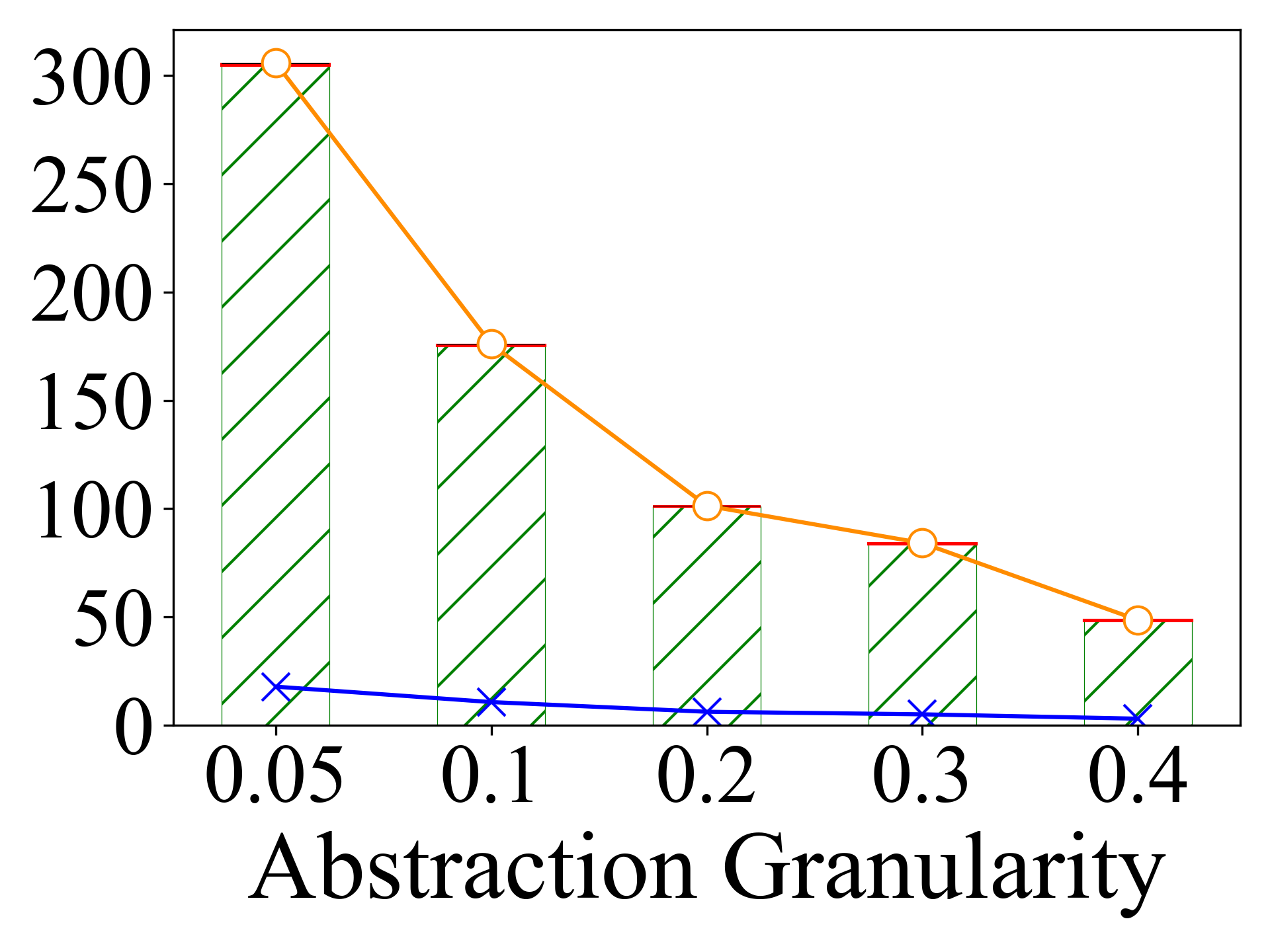}
			\caption{Tora (tanh)}
			\label{fig:tora_dec_tanh}
		\end{subfigure}&
				\begin{subfigure}[b]{0.25\textwidth}
			\includegraphics[width=\textwidth]{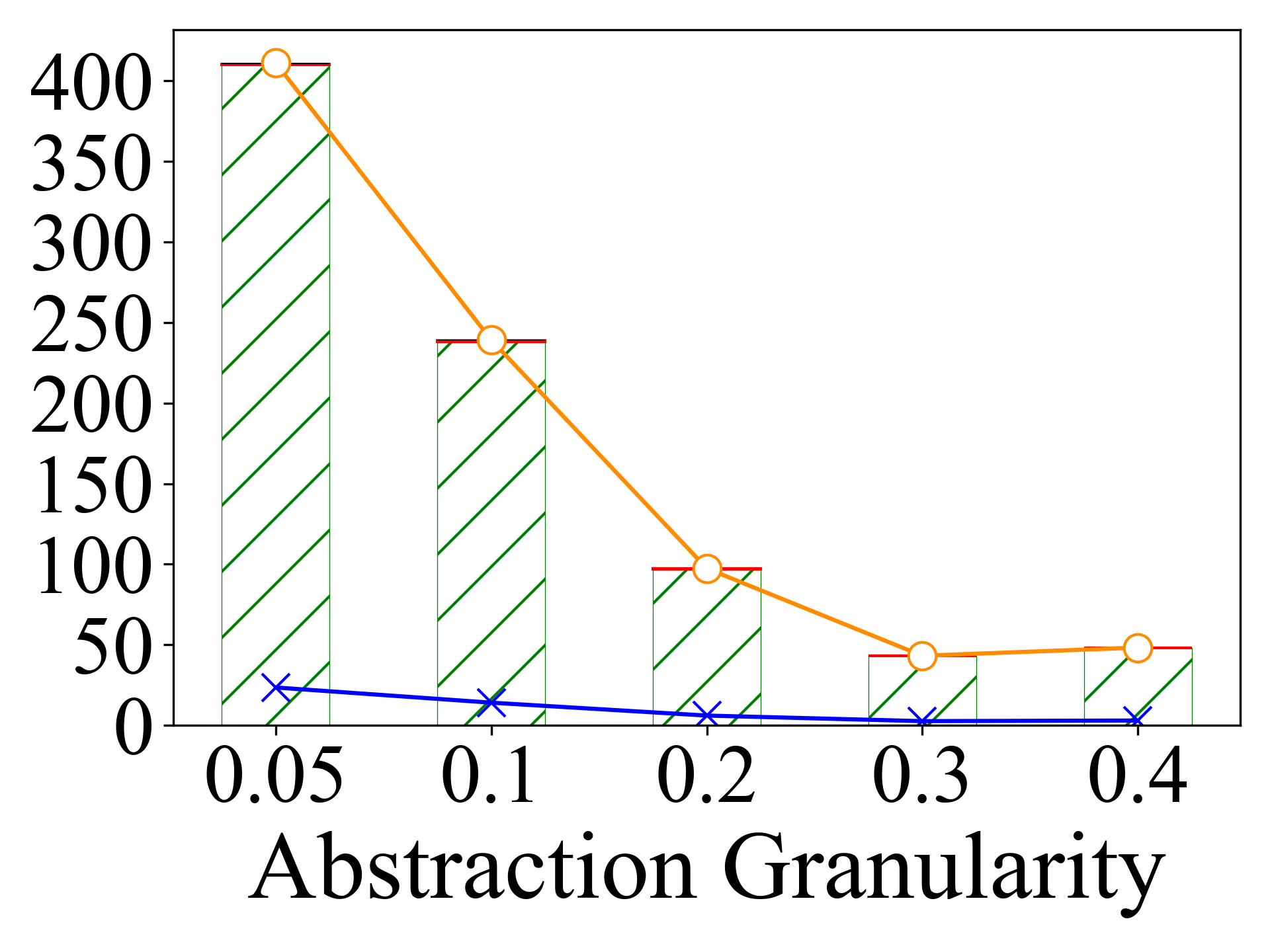}
			\caption{Tora (relu)}
			\label{fig:tora_dec_relu}
		\end{subfigure}
		\\
		\begin{subfigure}[b]{0.25\textwidth}
			\includegraphics[width=\textwidth]{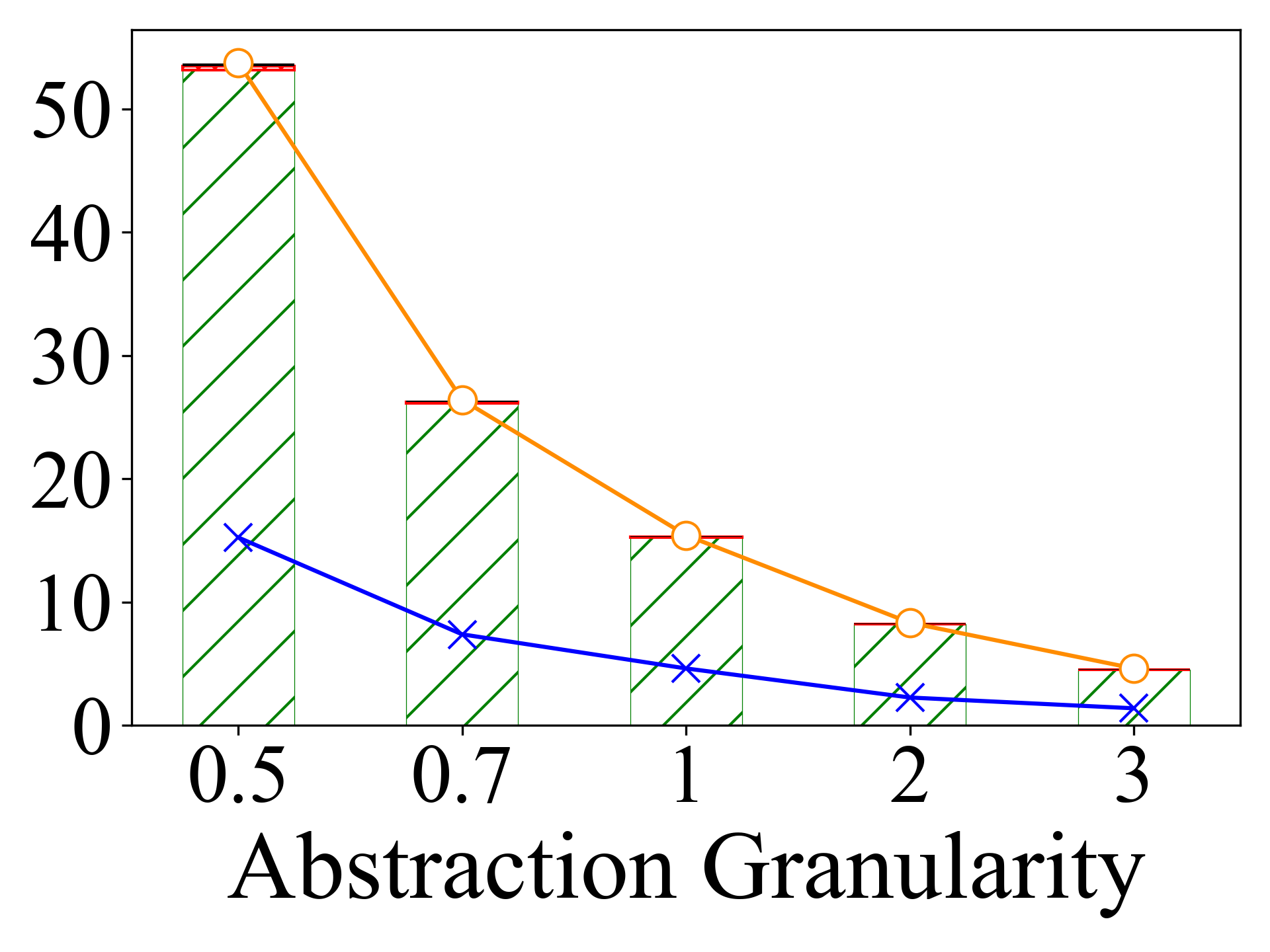}
			\caption{ACC (tanh)}
			\label{fig:acc_dec_atnh}
		\end{subfigure}&
		\begin{subfigure}[b]{0.25\textwidth}
			\includegraphics[width=\textwidth]{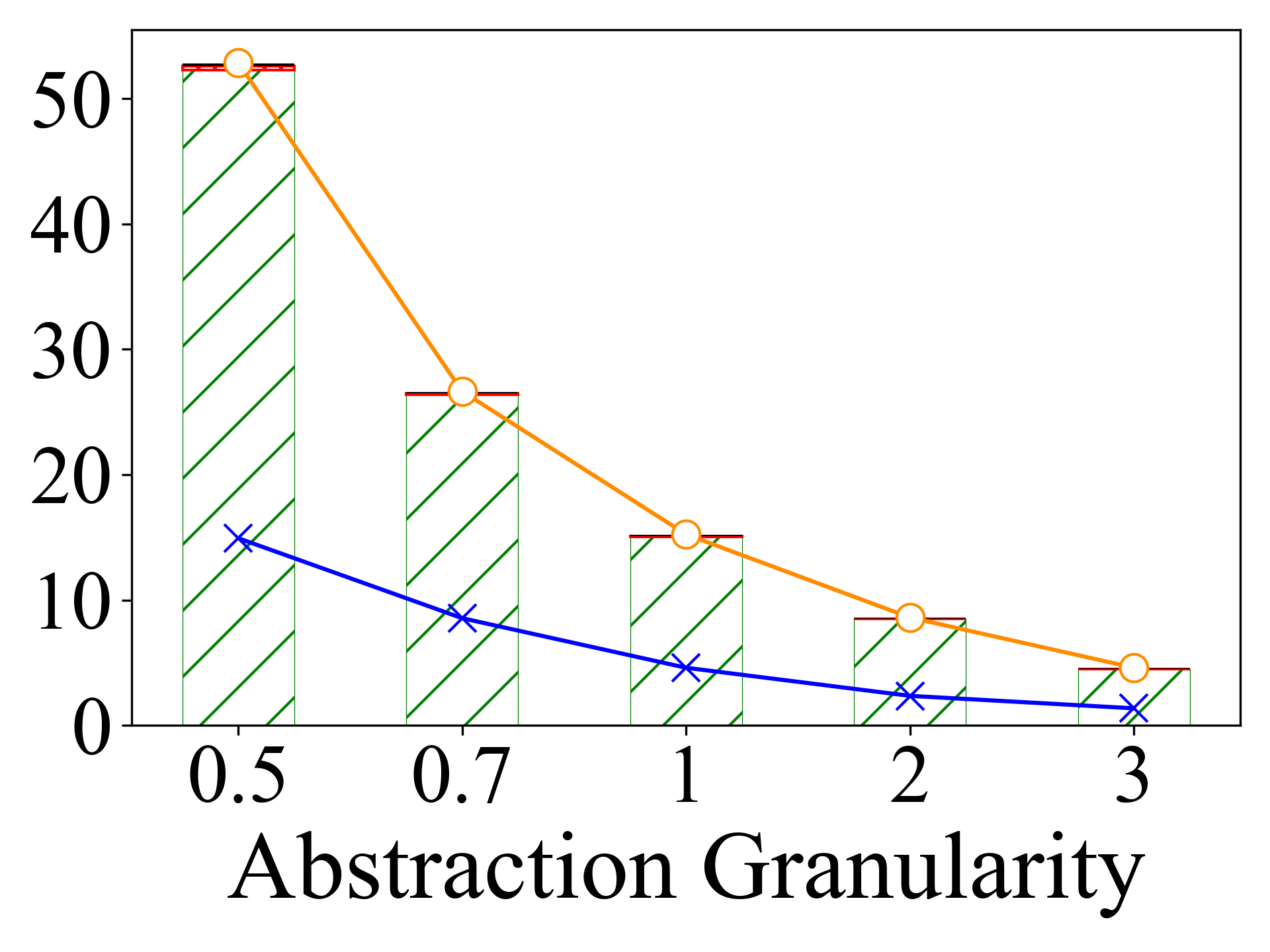}
			
			\caption{ACC (relu)}
			\label{fig:acc_dec_relu}
		\end{subfigure}
		&&\\
	\end{tabular}
\end{center}
\vspace{-3ex}
\caption{Decomposing analysis results. Due to the space reason, for B1-B5 and Tora, we use a scalar value $x_1$ to denote the $n$-dimensional abstraction granularity vector $\gamma = (x_1,...,x_1)$. For ACC, we use a 3-dimensional vector $(x_1, x_2, x_3)$ to denote the 6-dimensional abstraction granularity vector $\gamma = (x_1,x_2,x_3,x_1,x_2,x_3)$}
\label{fig:decompose_result}
\end{figure}

\clearpage 
\section{Comparison with ReachNN$^*$}
\label{app:vs-reachnn}

In the realm of advanced reachability analysis tools for systems controlled by deep neural networks (DNNs), ReachNN*~\cite{fan2020reachnn} stands out as a noteworthy contender. Distinguishing itself from methodologies like Verisig 2.0 and Polar, which adopt a layer-by-layer over-approximation strategy for DNNs, ReachNN* employs a distinctive \textit{sample-based} approach. This approach involves approximating entire neural networks through polynomial regression techniques.

Our investigation extends to a comparative analysis between \BBReach\ and ReachNN*. An added dimension to this comparison arises from both tools supporting parallelization, a feature that enhances efficiency. It is noteworthy, however, that our discussion concerning the performance of ReachNN* is reserved for this appendix. This decision is rooted in the fact that ReachNN*, despite its capabilities, has been eclipsed in evaluations by Polar~\cite{huang2022polar} and Verisig 2.0~\cite{ivanov2021verisig}.

Within this appendix, we present empirical evidence that underscores \BBReach's superiority over ReachNN*, with respect to both the precision of results and computational efficiency. Regarding the verification efficiency, despite ReachNN*'s acceleration with 32 cores, \BBReach\ achieves up to 2.5k times speedup with a single core; when the parallelization is enabled, \BBReach\ (with 20 cores) outperforms ReachNN$^*$ in all cases, with up to 22k times improvement.
 For the tightness of over-approximation results,  ReachNN* produces large over-approximation error in B2, Tora, and ACC (see Figure~\ref{fig:reachnn_tightness_result}) and fails to verify the reach-avoid properties in 15 out of 28 instances (see Table~\ref{tab:time_reachnn}).
 
\begin{table}[h!]
    \centering
    \footnotesize
    \renewcommand{\arraystretch}{0.95}
        \caption{Verification time (s) and result (verified or not) for the comparison between \BBReach~and  ReachNN$^*$. }
    \label{tab:time_reachnn}
    \begin{tabular}{|c|c|r|r|r|c|R{1.5cm}|R{1.5cm}|R{1.5cm}|C{1.5cm}|}
    	\hline
    	\multirow{2}{*}{\textbf{Task}} & \multirow{2}{*}{\textbf{Dim}} & \multicolumn{1}{c|}{\multirow{2}{*}{\textbf{Network}}} &  \multicolumn{3}{c|}{\textbf{\BBReach}}   &                                  \multicolumn{4}{c|}{\textbf{ReachNN$^{*}$}}                                  \\ \cline{4-10}
    	              ~                &               ~               &                                                      ~ & \textbf{1C} & \textbf{20Cs} & \textbf{VR} & \textbf{Default} &                      \textbf{Impr.} &                   \textbf{Impr.$^*$} &  \textbf{VR}  \\ \hline\hline
    	     \multirow{4}{*}{B1}       &      \multirow{4}{*}{2}       &                                   Tanh$_{2 \times 20}$ &        45.7 &          6.88 & \checkmark  &               60 &    \color{OliveGreen}{1.33$\times$} &     \color{OliveGreen}{8.72$\times$} & \ding{55}$^a$ \\ 
    	              ~                &               ~               &                                  Tanh$_{3 \times 100}$ &        42.8 &          5.53 & \checkmark  &              162 &    \color{OliveGreen}{3.79$\times$} &    \color{OliveGreen}{29.29$\times$} &  \checkmark   \\ \cline{3-10}
    	              ~                &               ~               &                                   ReLU$_{2 \times 20}$ &        42.9 &          6.44 & \checkmark  &               20 &      \color{Mahogany}{0.47$\times$} &     \color{OliveGreen}{3.11$\times$} &  \checkmark   \\ 
    	              ~                &               ~               &                                  ReLU$_{3 \times 100}$ &        52.5 &          8.65 & \checkmark  &              330 &    \color{OliveGreen}{6.29$\times$} &    \color{OliveGreen}{38.15$\times$} & \ding{55}$^a$ \\ \hline
    	     \multirow{4}{*}{B2}       &      \multirow{4}{*}{2}       &                                   Tanh$_{2 \times 20}$ &        10.0 &          1.19 & \checkmark  &               71 &    \color{OliveGreen}{7.10$\times$} &    \color{OliveGreen}{59.66$\times$} & \ding{55}$^a$ \\
    	              ~                &               ~               &                                  Tanh$_{3 \times 100}$ &        10.8 &          1.36 & \checkmark  &              172 &   \color{OliveGreen}{15.93$\times$} &   \color{OliveGreen}{126.47$\times$} & \ding{55}$^a$ \\ \cline{3-10}
    	              ~                &               ~               &                                   ReLU$_{2 \times 20}$ &         8.6 &          1.30 & \checkmark  &                4 &      \color{Mahogany}{0.47$\times$} &     \color{OliveGreen}{3.08$\times$} &  \checkmark   \\
    	              ~                &               ~               &                                  ReLU$_{3 \times 100}$ &        12.4 &          1.42 & \checkmark  &             6647 &  \color{OliveGreen}{536.05$\times$} &  \color{OliveGreen}{4680.99$\times$} & \ding{55}$^a$ \\ \hline
    	     \multirow{4}{*}{B3}       &      \multirow{4}{*}{2}       &                                   Tanh$_{2 \times 20}$ &         4.2 &          0.47 & \checkmark  &              115 &   \color{OliveGreen}{27.38$\times$} &   \color{OliveGreen}{244.68$\times$} &  \checkmark   \\ 
    	              ~                &               ~               &                                  Tanh$_{3 \times 100}$ &         4.3 &          0.50 & \checkmark  &               93 &   \color{OliveGreen}{21.63$\times$} &   \color{OliveGreen}{186.00$\times$} &  \checkmark   \\ \cline{3-10}
    	              ~                &               ~               &                                   ReLU$_{2 \times 20}$ &         4.1 &          0.47 & \checkmark  &               69 &   \color{OliveGreen}{16.83$\times$} &   \color{OliveGreen}{146.81$\times$} &  \checkmark   \\ 
    	              ~                &               ~               &                                  ReLU$_{3 \times 100}$ &         4.2 &          0.47 & \checkmark  &            10321 & \color{OliveGreen}{2457.38$\times$} & \color{OliveGreen}{21959.57$\times$} & \ding{55}$^a$ \\ \hline
    	     \multirow{4}{*}{B4}       &      \multirow{4}{*}{3}       &                                   Tanh$_{2 \times 20}$ &         1.3 &          0.32 & \checkmark  &               17 &   \color{OliveGreen}{13.08$\times$} &    \color{OliveGreen}{53.13$\times$} &  \checkmark   \\
    	              ~                &               ~               &                                  Tanh$_{3 \times 100}$ &         1.0 &          0.24 & \checkmark  &               20 &   \color{OliveGreen}{20.00$\times$} &    \color{OliveGreen}{83.33$\times$} &  \checkmark   \\ \cline{3-10}
    	              ~                &               ~               &                                   ReLU$_{2 \times 20}$ &         1.9 &          0.48 & \checkmark  &                7 &    \color{OliveGreen}{3.68$\times$} &    \color{OliveGreen}{14.58$\times$} &  \checkmark   \\
    	              ~                &               ~               &                                  ReLU$_{3 \times 100}$ &         1.8 &          0.43 & \checkmark  &               18 &   \color{OliveGreen}{10.00$\times$} &    \color{OliveGreen}{41.86$\times$} &  \checkmark   \\ \hline
    	     \multirow{4}{*}{B5}       &      \multirow{4}{*}{3}       &                                  Tanh$_{3 \times 100}$ &        13.3 &          2.48 & \checkmark  &               27 &    \color{OliveGreen}{2.03$\times$} &    \color{OliveGreen}{10.89$\times$} & \ding{55}$^a$ \\
    	              ~                &               ~               &                                  Tanh$_{4 \times 200}$ &         8.2 &          1.63 & \checkmark  &             2344 &  \color{OliveGreen}{285.85$\times$} &  \color{OliveGreen}{1438.04$\times$} & \ding{55}$^a$ \\ \cline{3-10}
    	              ~                &               ~               &                                  ReLU$_{3 \times 100}$ &         5.8 &          1.08 & \checkmark  &               90 &   \color{OliveGreen}{15.52$\times$} &    \color{OliveGreen}{83.33$\times$} &  \checkmark   \\
    	              ~                &               ~               &                                  ReLU$_{4 \times 200}$ &        13.5 &          2.50 & \checkmark  &             2845 &  \color{OliveGreen}{210.74$\times$} &  \color{OliveGreen}{1138.00$\times$} & \ding{55}$^a$ \\ \hline
    	    \multirow{4}{*}{Tora}      &      \multirow{4}{*}{4}       &                                   Tanh$_{3 \times 20}$ &       133.2 &          8.61 & \checkmark  &             1610 &   \color{OliveGreen}{12.09$\times$} &   \color{OliveGreen}{186.99$\times$} &  \checkmark   \\
    	              ~                &               ~               &                                  Tanh$_{4 \times 100}$ &       112.3 &          9.78 & \checkmark  &              --- &                                 --- &                                  --- & \ding{55}$^b$ \\ \cline{3-10}
    	              ~                &               ~               &                                   ReLU$_{3 \times 20}$ &       124.7 &          9.97 & \checkmark  &              778 &    \color{OliveGreen}{6.24$\times$} &    \color{OliveGreen}{78.03$\times$} &  \checkmark   \\
    	              ~                &               ~               &                                  ReLU$_{4 \times 100}$ &       128.1 &          7.54 & \checkmark  &              --- &                                 --- &                                  --- & \ding{55}$^b$ \\ \hline
    	     \multirow{4}{*}{ACC}      &      \multirow{4}{*}{6}       &                                   Tanh$_{3 \times 20}$ &        15.4 &          4.53 & \checkmark  &             5498 &  \color{OliveGreen}{357.01$\times$} &  \color{OliveGreen}{1213.69$\times$} & \ding{55}$^a$ \\
    	              ~                &               ~               &                                  Tanh$_{4 \times 100}$ &        15.2 &          4.51 & \checkmark  &              --- &                                 --- &                                  --- & \ding{55}$^b$ \\ \cline{3-10}
    	              ~                &               ~               &                                   ReLU$_{3 \times 20}$ &        15.2 &          4.45 & \checkmark  &             4633 &  \color{OliveGreen}{304.80$\times$} &  \color{OliveGreen}{1041.12$\times$} & \ding{55}$^a$ \\
    	              ~                &               ~               &                                  ReLU$_{4 \times 100}$ &        18.4 &          5.49 & \checkmark  &              --- &                                 --- &                                  --- & \ding{55}$^b$ \\ \hline
    \end{tabular}
    \begin{tablenotes}
			\small \item \textbf{Remarks.}
			 Impr.: time speedup of \BBReach~with one core compared to ReachNN$^*$ with 32 cores  (ReachNN$^*$/\BBReach). Impr.$^*$ denotes the comparison between BBReach with 20 cores and ReachNN$^*$ with 32 cores. 
		Tanh/ReLU($n \times k$): a neural network with the activation function Tanh/ReLU, $n$ hidden layers, and $k$ neurons per hidden layer.   
  VR:  verification result.
			$\checkmark$: the reach-avoid problem is successfully verified.
    \ding{55}$^{type}$: the reach-avoid problem cannot be verified due to $type$: (a) large over-approximation error, 
    (b) the calculation did not finish.  
    ---: no data available due to \ding{55}$^{b}$.
    \end{tablenotes}
\vspace{-3ex}
\end{table}
\clearpage 
\begin{figure}[htbp]
\begin{center}
		\begin{tabular}{ccc}
		\begin{subfigure}[b]{0.33\textwidth}
			\includegraphics[width=\textwidth]{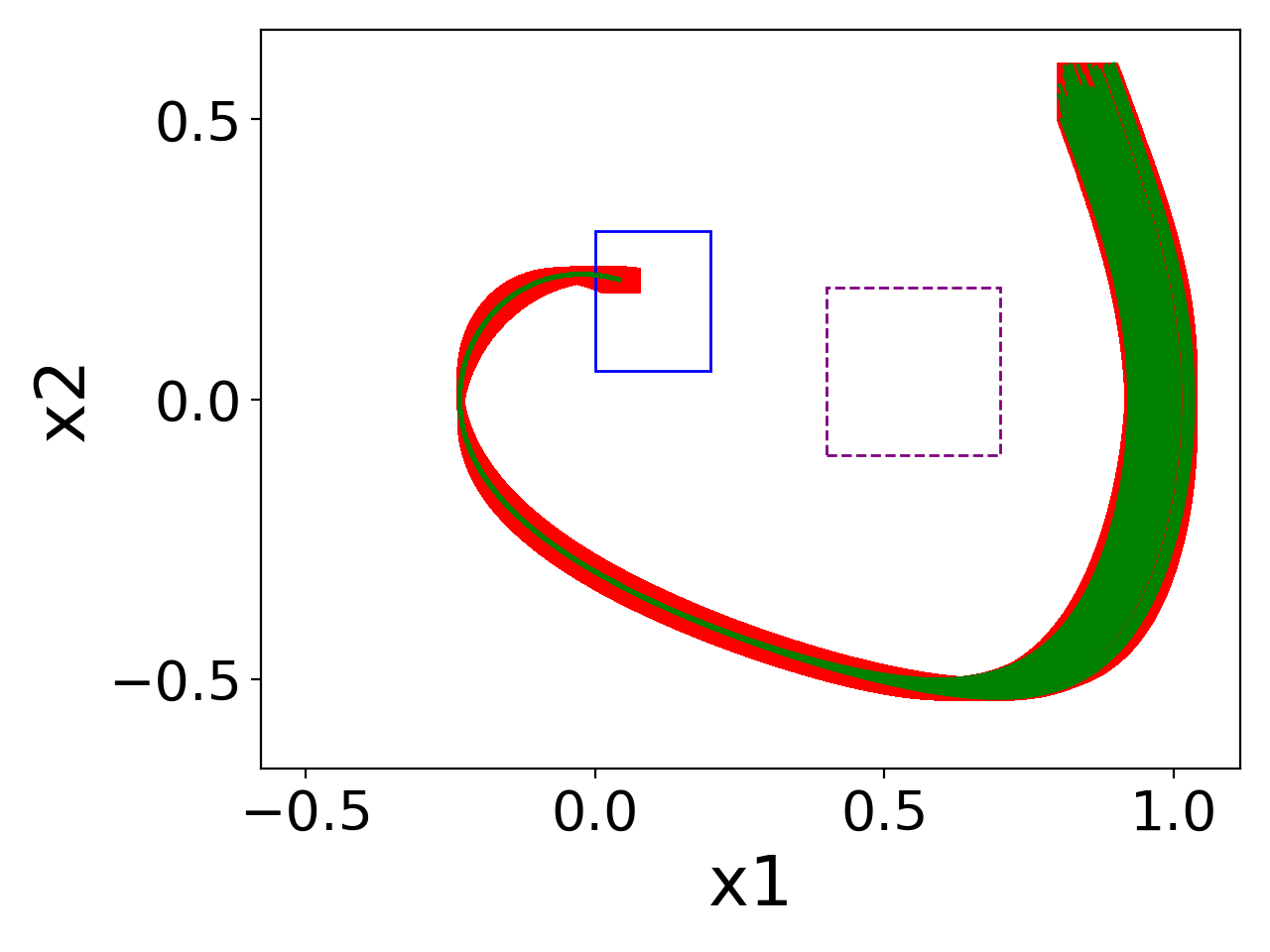}
			\caption{B1}
			\label{fig:reachnn_tightness_b1}
		\end{subfigure}&
		\begin{subfigure}[b]{0.33\textwidth}
			\includegraphics[width=\textwidth]{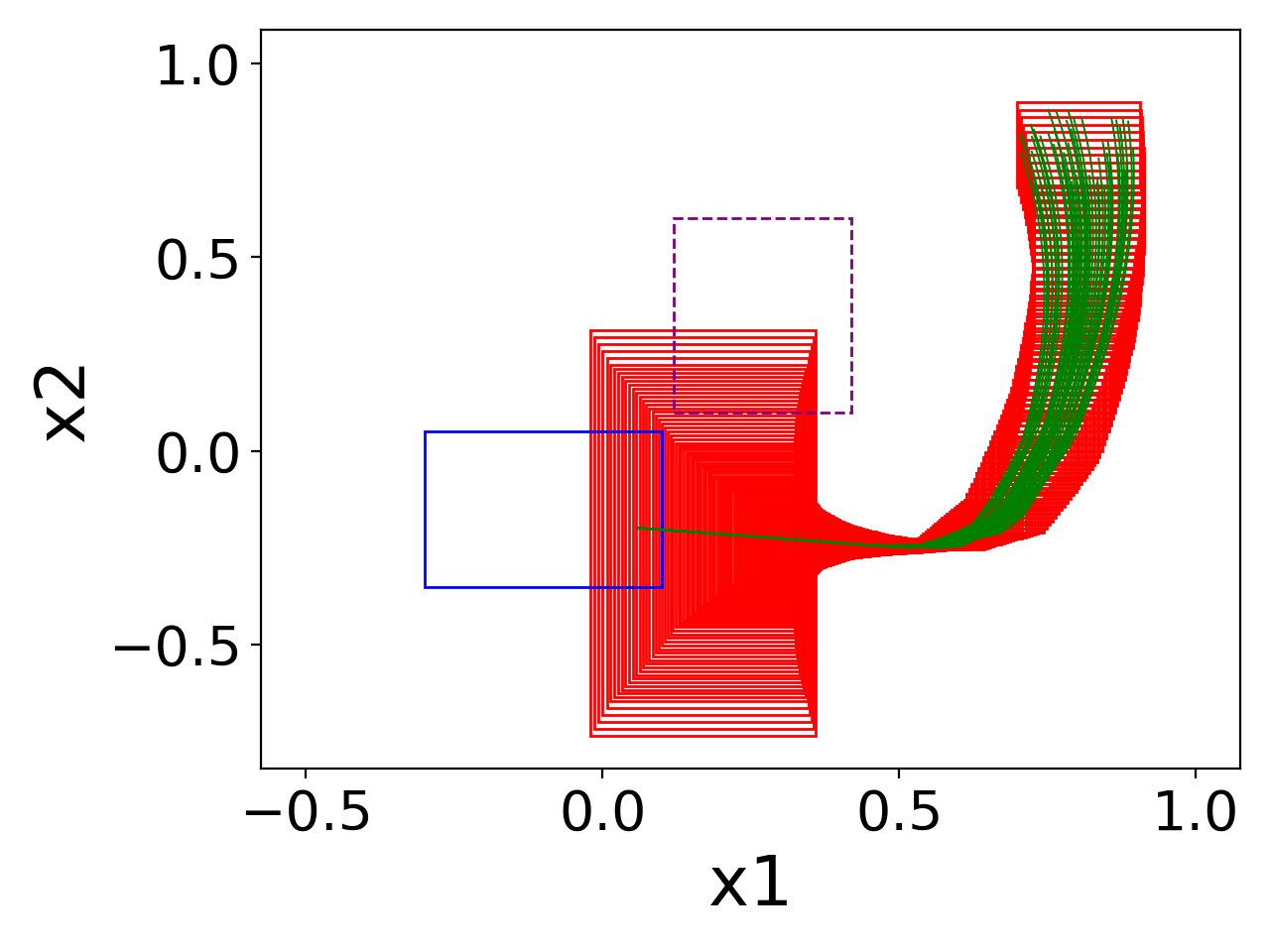}
			\caption{B2}
			\label{fig:reachnn_tightness_b2}
		\end{subfigure}&
		\begin{subfigure}[b]{0.33\textwidth}
			\includegraphics[width=\textwidth]{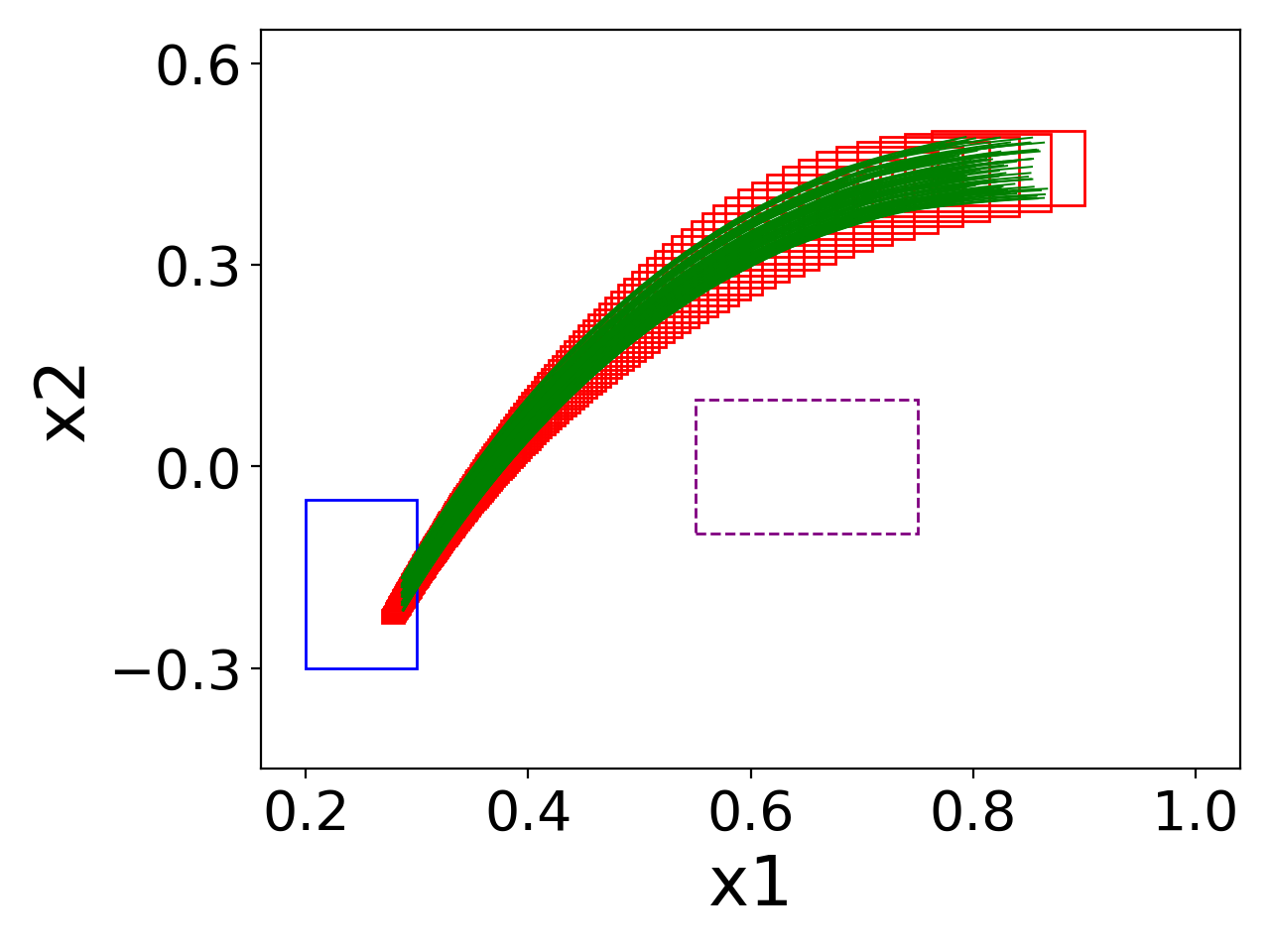}
			\caption{B3}
			\label{fig:reachnn_tightness_b3}
		\end{subfigure}
		\\
  \begin{subfigure}[b]{0.33\textwidth}
			\includegraphics[width=\textwidth]{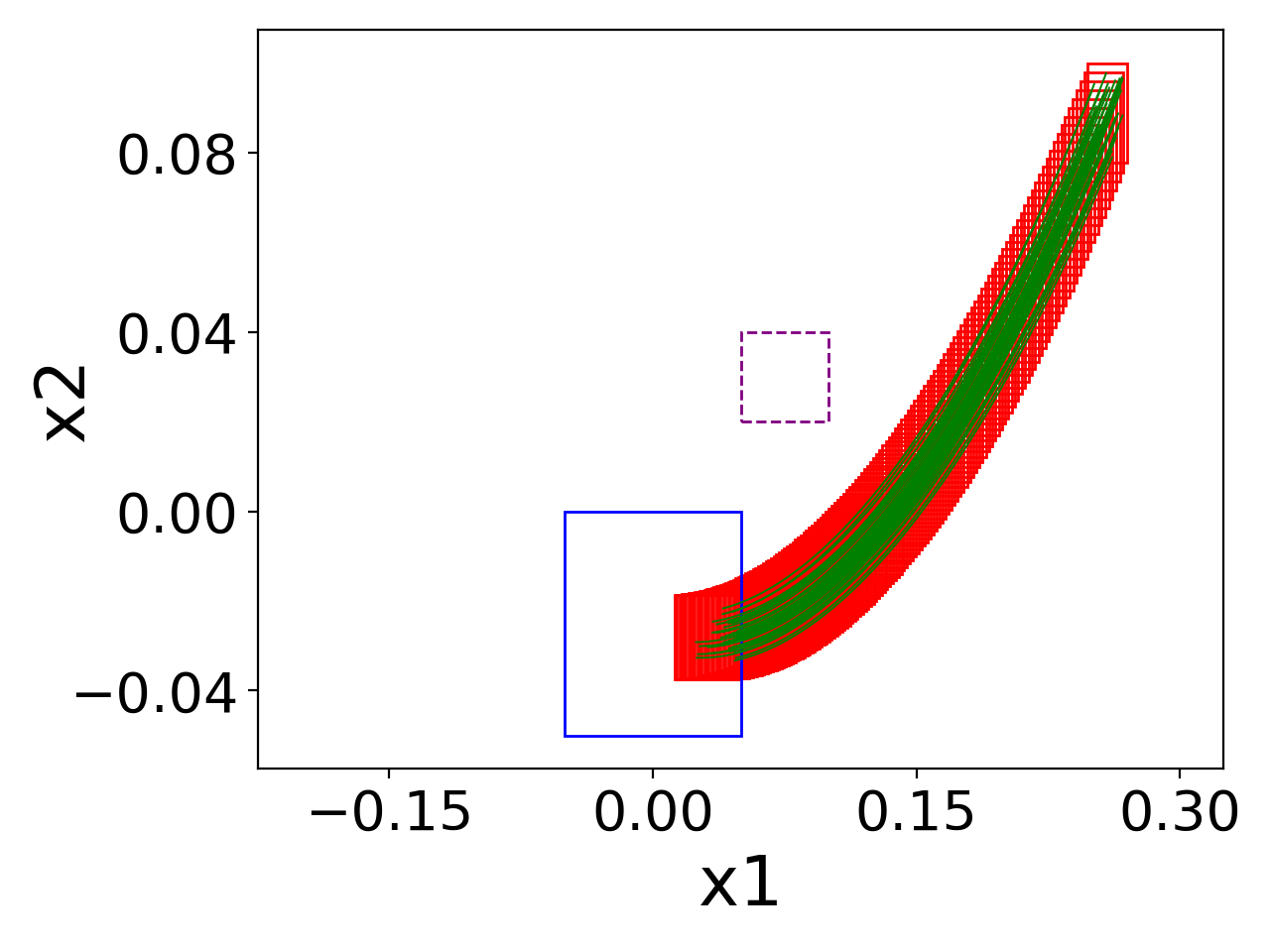}
			\caption{B4}
			\label{fig:reachnn_tightness_b4}
		\end{subfigure}&
		\begin{subfigure}[b]{0.33\textwidth}
			\includegraphics[width=\textwidth]{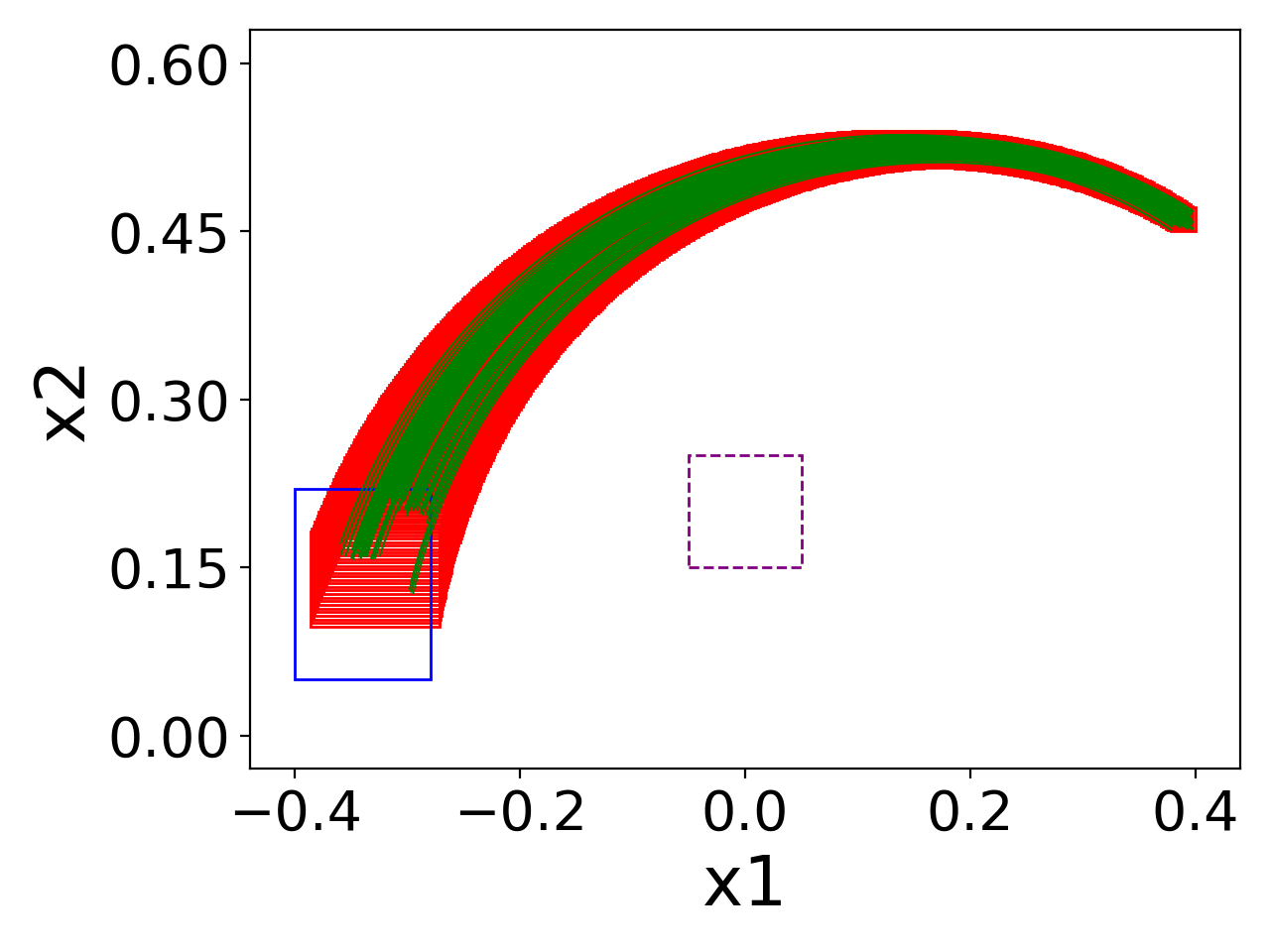}
			\caption{B5}
			\label{fig:reachnn_tightness_b5}
		\end{subfigure}&
		\begin{subfigure}[b]{0.33\textwidth}
			\includegraphics[width=\textwidth]{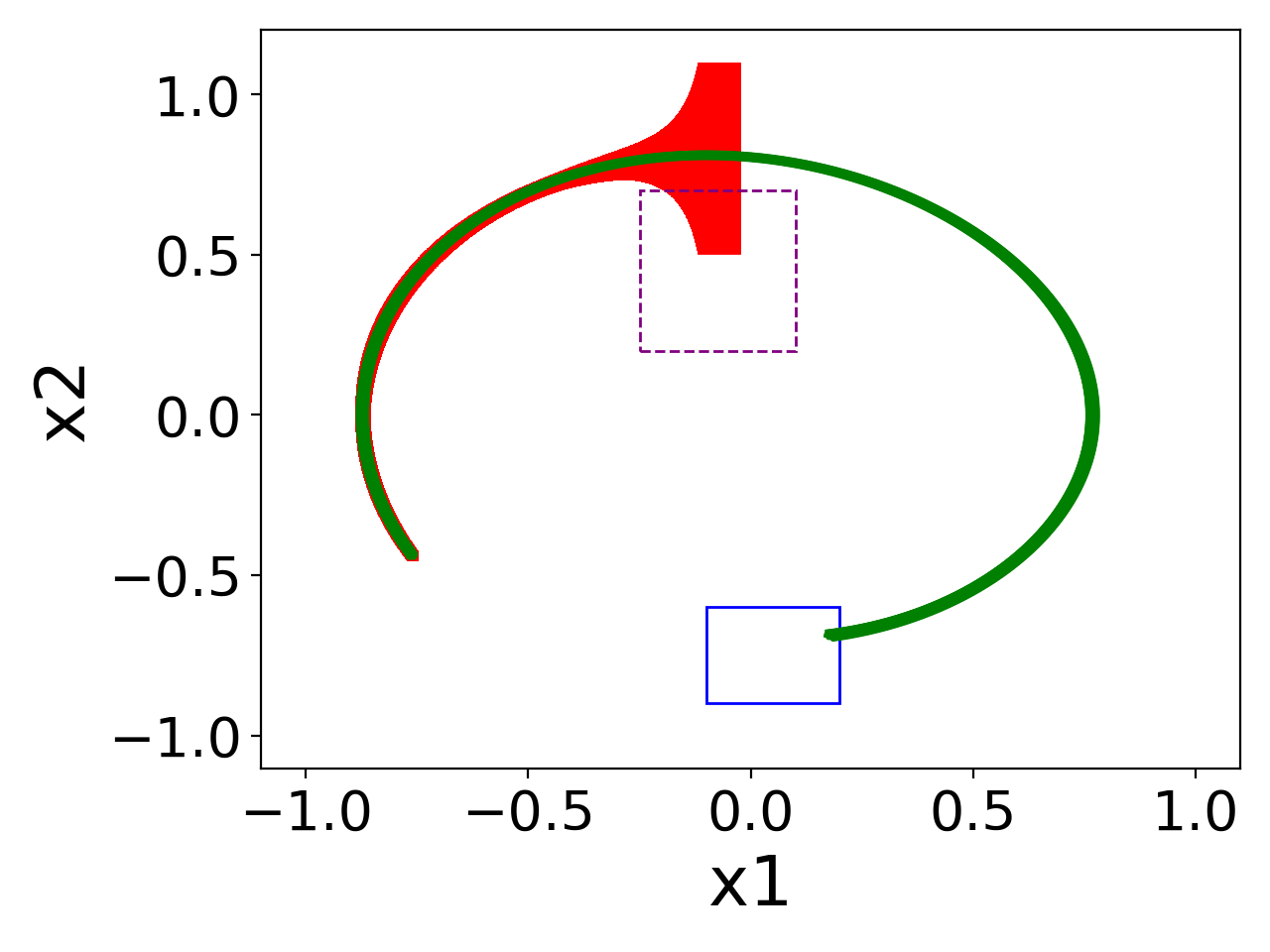}
			\caption{Tora}
			\label{fig:reachnn_tightness_tora}
		\end{subfigure}
		\\
		~~\begin{subfigure}[b]{0.33\textwidth}
			\includegraphics[width=0.95\textwidth]{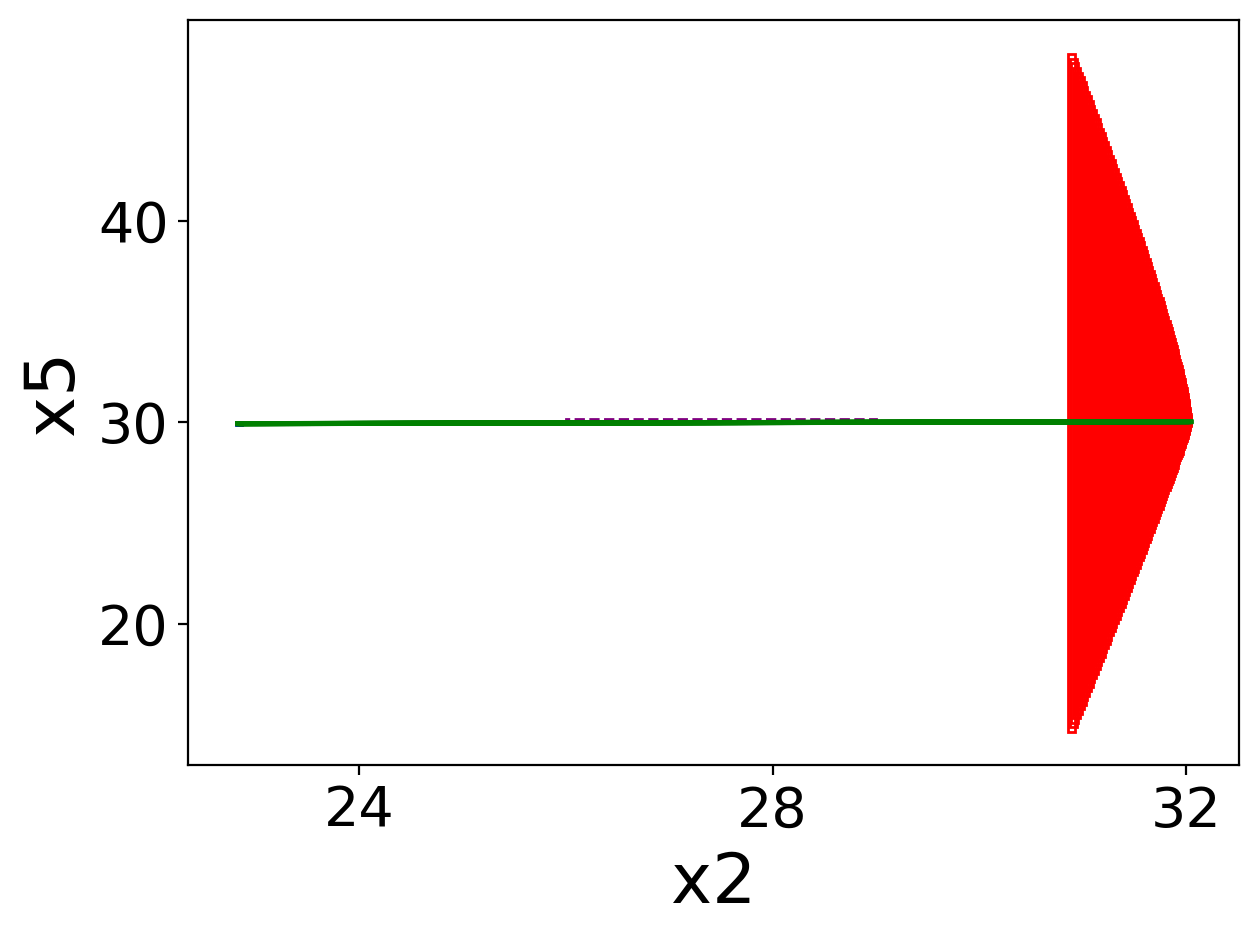}
			\caption{ACC}
			\label{fig:reachnn_tightness_acc}
		\end{subfigure}& &
		\\
	\end{tabular}
\end{center}
\vspace{-3ex}
\caption{Analysis results for ReachNN$^*$ on the larger networks with the Tanh activation function. red box: over-approximation set; green lines: simulation trajectories; blue box: goal region; 
purple dashed box: unsafe region.}
\label{fig:reachnn_tightness_result}
\end{figure}